\documentclass[english]{article}
\usepackage[T1]{fontenc}
\usepackage[latin9]{inputenc}
\usepackage{geometry}
\usepackage{color}
\usepackage{babel}
\usepackage{booktabs}
\usepackage{mathtools}
\usepackage{bm}
\usepackage{amsmath}
\usepackage{amsthm}
\usepackage{amssymb}
\usepackage{graphicx}
\usepackage[unicode=true,pdfusetitle,
 bookmarks=true,bookmarksnumbered=true,bookmarksopen=true,bookmarksopenlevel=2,
 breaklinks=false,pdfborder={0 0 0},pdfborderstyle={},backref=false,colorlinks=true]
 {hyperref}
\hypersetup{
 linkcolor=blue, urlcolor=blue, citecolor=blue, linktoc=all, linktocpage=false}

\makeatletter

\providecommand{\tabularnewline}{\\}

\theoremstyle{definition}
\newtheorem{defn}{\protect\definitionname}
\theoremstyle{plain}
\newtheorem{assumption}{\protect\assumptionname}
\theoremstyle{plain}
\newtheorem{thm}{\protect\theoremname}
\theoremstyle{plain}
\newtheorem{prop}{\protect\propositionname}
\theoremstyle{remark}
\newtheorem{rem}{\protect\remarkname}
\theoremstyle{plain}
\newtheorem{lem}{\protect\lemmaname}
\theoremstyle{plain}
\newtheorem{fact}{\protect\factname}

\@ifundefined{date}{}{\date{}}
\usepackage{babel}
\usepackage{bm}
\usepackage{bbm}

\usepackage[linesnumbered,ruled,vlined]{algorithm2e}
\usepackage{algorithmic}
\renewcommand{\Pr}{{\mathbb{P}}}
\newcommand{\E}{{\mathbb{E}}}
\newcommand{\R}{{\mathbb{R}}}
\newcommand{\Ncal}{{\mathcal{N}}}

\let\hat\widehat

\let\tilde\widetilde

\allowdisplaybreaks

\usepackage{color}

\definecolor{yxc}{RGB}{255,0,0}
\definecolor{yjc}{RGB}{125,0,0}
\definecolor{cm}{RGB}{0,0,200}
\definecolor{ytw}{RGB}{0,0,200}

\definecolor{redel}{RGB}{200,200,200}                   

\definecolor{readd}{RGB}{200,0,0}                    

\theoremstyle{plain} 

\newtheorem{itclaim}{\textbf{Claim}}

\newcommand{\bA}{\bm{A}}
\newcommand{\bB}{\bm{B}}
\newcommand{\bC}{\bm{C}}
\newcommand{\bD}{\bm{D}}
\newcommand{\bE}{\bm{E}}

\newcommand{\bI}{\bm{I}}

\newcommand{\bL}{\bm{L}}
\newcommand{\bM}{\bm{M}}

\newcommand{\bO}{\bm{O}}

\newcommand{\bR}{\bm{R}}
\newcommand{\bS}{\bm{S}}
\newcommand{\bT}{\bm{T}}
\newcommand{\bU}{\bm{U}}
\newcommand{\bV}{\bm{V}}
\newcommand{\bW}{\bm{W}}
\newcommand{\bX}{\bm{X}}
\newcommand{\bY}{\bm{Y}}
\newcommand{\bZ}{\bm{Z}}

\newcommand{\ba}{\bm{a}}

\newcommand{\be}{\bm{e}}

\newcommand{\bx}{\bm{x}}

\newcommand{\bz}{\bm{z}}

\makeatother

\providecommand{\assumptionname}{Assumption}
\providecommand{\definitionname}{Definition}
\providecommand{\factname}{Fact}
\providecommand{\lemmaname}{Lemma}
\providecommand{\propositionname}{Proposition}
\providecommand{\remarkname}{Remark}
\providecommand{\theoremname}{Theorem}

\begin{document}
\title{Learning Mixtures of Low-Rank Models\footnotetext{Corresponding author: Yuxin Chen (email: \texttt{yuxin.chen@princeton.edu}).}}
\author{Yanxi Chen\thanks{Department of Electrical Engineering, Princeton University, Princeton,
NJ 08544, USA; email: \texttt{\{yanxic,poor,yuxin.chen\}@princeton.edu}.} \and Cong~Ma\footnote{Department of Electrical Engineering and Computer Sciences, University of California Berkeley, Berkeley, CA 94720, USA; email: {\tt congm@berkeley.edu}}\and
H.~Vincent Poor\footnotemark[1] \and Yuxin Chen\footnotemark[1]}
\date{\today}

\maketitle
\global\long\def\poly{\mathsf{poly}}%
\global\long\def\plog{\mathsf{polylog}}%
\global\long\def\Frm{{\rm F}}%
\global\long\def\Tr{\mathsf{Tr}}%

\global\long\def\rank{\mathsf{rank}}%
\global\long\def\spn{\mathsf{span}}%
\global\long\def\col{\mathsf{col}}%
\global\long\def\row{\mathsf{row}}%
\global\long\def\vc{\mathsf{vec}}%
\global\long\def\mat{\mathsf{mat}}%
\global\long\def\SNR{\mathsf{SNR}}%

\global\long\def\Pr{\mathbb{P}}%
\global\long\def\E{\mathbb{E}}%
\global\long\def\R{\mathbb{R}}%
\global\long\def\Ncal{\mathcal{N}}%
\global\long\def\Dcal{\mathcal{D}}%

\global\long\def\ind{\mathbbm{1}}%
\global\long\def\indtilde{\chi}%

\global\long\def\quantalpha{Q_{\alpha}}%
\global\long\def\iid{\overset{\mathsf{i.i.d.}}{\sim}}%

\global\long\def\Monestar{{\boldsymbol{M}_{1}^{\star}}}%
\global\long\def\Mtwostar{{\boldsymbol{M}_{2}^{\star}}}%
\global\long\def\Mstar{\boldsymbol{M}^{\star}}%
\global\long\def\Mkstar{{\boldsymbol{M}_{k}^{\star}}}%
\global\long\def\MKstar{{\boldsymbol{M}_{K}^{\star}}}%

\global\long\def\Mistar{\bM_{i}^{\star}}%
\global\long\def\Mjstar{\bM_{j}^{\star}}%
\global\long\def\Mlstar{\bM_{l}^{\star}}%

\global\long\def\Uonestar{{\boldsymbol{U}_{1}^{\star}}}%
\global\long\def\Utwostar{{\boldsymbol{U}_{2}^{\star}}}%
\global\long\def\Uhat{{\bU}}%
\global\long\def\Ustar{{\boldsymbol{U}^{\star}}}%
\global\long\def\Ukstar{{\boldsymbol{U}_{k}^{\star}}}%
\global\long\def\UKstar{{\boldsymbol{U}_{K}^{\star}}}%

\global\long\def\Sigonestar{{\boldsymbol{\Sigma}_{1}^{\star}}}%
\global\long\def\Sigtwostar{{\boldsymbol{\Sigma}_{2}^{\star}}}%
\global\long\def\Sigkstar{{\boldsymbol{\Sigma}_{k}^{\star}}}%
\global\long\def\SigKstar{{\boldsymbol{\Sigma}_{K}^{\star}}}%

\global\long\def\Vonestar{{\boldsymbol{V}_{1}^{\star}}}%
\global\long\def\Vtwostar{{\boldsymbol{V}_{2}^{\star}}}%
\global\long\def\Vhat{{\bV}}%
\global\long\def\Vstar{{\boldsymbol{V}^{\star}}}%
\global\long\def\Vkstar{{\boldsymbol{V}_{k}^{\star}}}%
\global\long\def\VKstar{{\boldsymbol{V}_{K}^{\star}}}%

\global\long\def\Uistar{\bU_{i}^{\star}}%
\global\long\def\Ujstar{\bU_{j}^{\star}}%
\global\long\def\Ulstar{\bU_{l}^{\star}}%
\global\long\def\Vistar{\bV_{i}^{\star}}%
\global\long\def\Vjstar{\bV_{j}^{\star}}%
\global\long\def\Vlstar{\bV_{l}^{\star}}%

\global\long\def\UL{\bU_{L}}%
\global\long\def\SigL{\boldsymbol{\Sigma}_{L}}%
\global\long\def\VL{\bV_{L}}%
\global\long\def\UR{\bU_{R}}%
\global\long\def\SigR{\boldsymbol{\Sigma}_{R}}%
\global\long\def\VR{\bV_{R}}%

\global\long\def\Sigistar{\boldsymbol{\Sigma}_{i}^{\star}}%
\global\long\def\Sigjstar{\boldsymbol{\Sigma}_{j}^{\star}}%
\global\long\def\Siglstar{\boldsymbol{\Sigma}_{l}^{\star}}%

\global\long\def\Sonestar{{\boldsymbol{S}_{1}^{\star}}}%
\global\long\def\Stwostar{{\boldsymbol{S}_{2}^{\star}}}%
\global\long\def\Sonehat{{\widehat{\bS}_{1}}}%
\global\long\def\Stwohat{{\widehat{\bS}_{2}}}%

\global\long\def\Skstar{{\boldsymbol{S}_{k}^{\star}}}%
\global\long\def\SKstar{{\boldsymbol{S}_{K}^{\star}}}%
\global\long\def\Skhat{{\widehat{\bS}_{k}}}%
\global\long\def\SKhat{{\widehat{\bS}_{K}}}%

\global\long\def\Sihat{{\widehat{\bS}_{i}}}%
\global\long\def\Sjhat{{\widehat{\bS}_{j}}}%
\global\long\def\Slhat{{\widehat{\bS}_{l}}}%
\global\long\def\Sistar{{\boldsymbol{S}_{i}^{\star}}}%
\global\long\def\Sjstar{{\boldsymbol{S}_{j}^{\star}}}%
\global\long\def\Slstar{{\boldsymbol{S}_{l}^{\star}}}%

\global\long\def\Sp{{\bS^{+}}}%
\global\long\def\Stildep{{{\widetilde{\bS}}^{+}}}%
\global\long\def\Stildepone{{{\widetilde{S}}_{1}^{+}}}%
\global\long\def\Stildepperp{{{\widetilde{\bS}}_{\perp}^{+}}}%

\global\long\def\Sperp{{\bS_{\perp}}}%
\global\long\def\Sbarest{\overline{\boldsymbol{S}}_{{\rm est}}}%

\global\long\def\betaonehat{\widehat{\bbeta}_{1}}%
\global\long\def\betaKhat{\widehat{\bbeta}_{K}}%
\global\long\def\betakhat{\widehat{\bbeta}_{k}}%
\global\long\def\betatwohat{\widehat{\bbeta}_{2}}%

\global\long\def\betalhat{\widehat{\bbeta}_{l}}%
\global\long\def\betaihat{\widehat{\bbeta}_{i}}%
\global\long\def\betajhat{\widehat{\bbeta}_{j}}%

\global\long\def\betaonestar{\boldsymbol{\beta}_{1}^{\star}}%
\global\long\def\betaKstar{\boldsymbol{\beta}_{K}^{\star}}%
\global\long\def\betakstar{\boldsymbol{\beta}_{k}^{\star}}%
\global\long\def\betatwostar{\boldsymbol{\beta}_{2}^{\star}}%
\global\long\def\betabarest{\overline{\bbeta}_{{\rm est}}}%
\global\long\def\bbeta{\boldsymbol{\beta}}%

\global\long\def\Lonestar{{\boldsymbol{L}_{1}^{\star}}}%
\global\long\def\Ltwostar{{\boldsymbol{L}_{2}^{\star}}}%
\global\long\def\Lkstar{{\boldsymbol{L}_{k}^{\star}}}%
\global\long\def\LKstar{{\boldsymbol{L}_{K}^{\star}}}%

\global\long\def\Ronestar{{\boldsymbol{R}_{1}^{\star}}}%
\global\long\def\Rtwostar{{\boldsymbol{R}_{2}^{\star}}}%
\global\long\def\Rkstar{{\boldsymbol{R}_{k}^{\star}}}%
\global\long\def\RKstar{{\boldsymbol{R}_{K}^{\star}}}%

\global\long\def\Lp{{\bL^{+}}}%
\global\long\def\Ltildep{{{\bL}_{\mathsf{pop}}^{+}}}%
\global\long\def\tLtildep{(\Ltildep)^{\top}}%

\global\long\def\Rp{{\bR^{+}}}%
\global\long\def\Rtildep{{{\bR}_{\mathsf{pop}}^{+}}}%
\global\long\def\tRtildep{(\Rtildep)^{\top}}%

\global\long\def\Rinv{\bR(\bR^{\top}\bR)^{-1}}%
\global\long\def\Linv{\bL(\bL^{\top}\bL)^{-1}}%
\global\long\def\EL{\bE_{\bL}}%
\global\long\def\ER{\bE_{\bR}}%

\global\long\def\Omegaonestar{{\Omega_{1}^{\star}}}%
\global\long\def\Omegatwostar{{\Omega_{2}^{\star}}}%
\global\long\def\Omegakstar{{\Omega_{k}^{\star}}}%
\global\long\def\OmegaKstar{{\Omega_{K}^{\star}}}%

\global\long\def\Omegaistar{\Omega_{i}^{\star}}%
\global\long\def\Omegajstar{\Omega_{j}^{\star}}%
\global\long\def\Omegalstar{\Omega_{l}^{\star}}%

\global\long\def\distp{\mathsf{dist}}%
\global\long\def\distU{\big\|\Uhat\Uhat^{\top}-\Ustar\Ustar^{\top}\big\|}%
\global\long\def\distV{\big\|\Vhat\Vhat^{\top}-\Vstar\Vstar^{\top}\big\|}%
\global\long\def\dist{\mathsf{dist}}%
\global\long\def\disthat{\mathsf{dist}_{\bU,\bV}}%

\global\long\def\sc{\mathsf{(small\,constant)}\cdot}%
\global\long\def\const{\mathsf{(constant)}\cdot}%

\global\long\def\kstar{{\kappa}}%

\global\long\def\sumone{\sum_{i\in\Omegaonestar}}%
\global\long\def\sumtwo{\sum_{i\in\Omegatwostar}}%
\global\long\def\sumall{\sum_{i=1}^{KN}}%
\global\long\def\sumomgk{\sum_{i\in\Omegakstar}}%
\global\long\def\sumomgK{\sum_{i\in\OmegaKstar}}%

\global\long\def\sumK{\sum_{k=1}^{K}}%
\global\long\def\sumall{\sum_{i=1}^{\Nall}}%

\global\long\def\nabonehat{{\widehat{\boldsymbol{\nabla}}_{1}}}%
\global\long\def\nabtwohat{{\widehat{\boldsymbol{\nabla}}_{2}}}%
\global\long\def\nabkhat{{\widehat{\boldsymbol{\nabla}}_{k}}}%
\global\long\def\nabKhat{{\widehat{\boldsymbol{\nabla}}_{K}}}%

\global\long\def\nabone{{\boldsymbol{\nabla}_{1}}}%
\global\long\def\nabtwo{{\boldsymbol{\nabla}_{2}}}%
\global\long\def\nabk{{\boldsymbol{\nabla}_{k}}}%
\global\long\def\nabK{{\boldsymbol{\nabla}_{K}}}%

\global\long\def\nabLonehat{{\widehat{\boldsymbol{\nabla}}_{1}^{(L)}}}%
\global\long\def\nabLtwohat{{\widehat{\boldsymbol{\nabla}}_{2}^{(L)}}}%
\global\long\def\nabLkhat{{\widehat{\boldsymbol{\nabla}}_{k}^{(L)}}}%
\global\long\def\nabLKhat{{\widehat{\boldsymbol{\nabla}}_{K}^{(L)}}}%

\global\long\def\nabLone{{\boldsymbol{\nabla}_{1}^{(L)}}}%
\global\long\def\nabLtwo{{\boldsymbol{\nabla}_{2}^{(L)}}}%
\global\long\def\nabLk{{\boldsymbol{\nabla}_{k}^{(L)}}}%
\global\long\def\nabLK{{\boldsymbol{\nabla}_{K}^{(L)}}}%

\global\long\def\nabRonehat{{\widehat{\boldsymbol{\nabla}}_{1}^{(R)}}}%
\global\long\def\nabRtwohat{{\widehat{\boldsymbol{\nabla}}_{2}^{(R)}}}%
\global\long\def\nabRkhat{{\widehat{\boldsymbol{\nabla}}_{k}^{(R)}}}%
\global\long\def\nabRKhat{{\widehat{\boldsymbol{\nabla}}_{K}^{(R)}}}%

\global\long\def\nabRone{{\boldsymbol{\nabla}_{1}^{(R)}}}%
\global\long\def\nabRtwo{{\boldsymbol{\nabla}_{2}^{(R)}}}%
\global\long\def\nabRk{{\boldsymbol{\nabla}_{k}^{(R)}}}%
\global\long\def\nabRK{{\boldsymbol{\nabla}_{K}^{(R)}}}%

\global\long\def\Delone{\boldsymbol{\Delta}_{1}}%
\global\long\def\Deltwo{\boldsymbol{\Delta}_{2}}%
\global\long\def\Delk{\boldsymbol{\Delta}_{k}}%
\global\long\def\DelK{\boldsymbol{\Delta}_{K}}%
\global\long\def\Delj{\boldsymbol{\Delta}_{j}}%
\global\long\def\Delmis{\boldsymbol{\Delta}_{\mathsf{mis}}}%
\global\long\def\Deltrip{\boldsymbol{\Delta}_{\mathsf{fs}}}%

\global\long\def\chat{\widehat{c}}%
\global\long\def\sigzk{\sigma_{z,k}}%

\global\long\def\mzero{m_{0}}%
\global\long\def\mone{\boldsymbol{m}_{1}}%
\global\long\def\Mtwo{\boldsymbol{M}_{2}}%
\global\long\def\Mthree{\boldsymbol{M}_{3}}%
\global\long\def\Mhat{\hat{\bM}}%
\global\long\def\Rstar{\bR^{\star}}%

\global\long\def\mtilzero{\widetilde{m}_{0}}%
\global\long\def\mtilone{\widetilde{\boldsymbol{m}}_{1}}%
\global\long\def\Mtiltwo{\widetilde{\boldsymbol{M}}_{2}}%
\global\long\def\Mtilthree{\widetilde{\boldsymbol{M}}_{3}}%

\global\long\def\Tcal{\mathcal{T}}%

\global\long\def\ai{\boldsymbol{a}_{i}}%
\global\long\def\aitil{{\ba_{i}^{\mathsf{aug}}}}%
\global\long\def\Ai{\boldsymbol{A}_{i}}%
\global\long\def\sigek{\sigma_{e,k}}%
\global\long\def\dtil{\widetilde{d}}%
\global\long\def\betaktil{{\bbeta_{k}^{\mathsf{aug}}}}%
\global\long\def\xtil{\widetilde{\bx}}%
\global\long\def\ztil{\widetilde{\bz}}%

\global\long\def\xihat{\widehat{\boldsymbol{\xi}}}%
\global\long\def\ei{\zeta_{i}}%
\global\long\def\yi{y_{i}}%
\global\long\def\zi{z_{i}}%
\global\long\def\sigt{\sigma_{t,k}}%
\global\long\def\ti{t_{i}}%

\global\long\def\pkhat{p_{k}}%
\global\long\def\ponehat{p_{1}}%
\global\long\def\pKhat{p_{K}}%
\global\long\def\ptwohat{p_{2}}%
\global\long\def\Nall{N}%

\global\long\def\sige{\sigma}%
\global\long\def\Nbar{\overline{N}}%

\global\long\def\Ttil{\widetilde{T}}%
\global\long\def\Nk{N_{k}}%
\global\long\def\NMLR{N_{\mathsf{MLR}}}%

\global\long\def\Mbar{\overline{S}}%
\global\long\def\munder{\underline{m}}%
\global\long\def\taumin{\tau_{\mathsf{min}}}%

\global\long\def\Aiaug{\Ai^{\mathsf{aug}}}%
\global\long\def\Xaug{\bX^{\mathsf{aug}}}%
\global\long\def\Zaug{\bZ^{\mathsf{aug}}}%

\global\long\def\cchi{c_{\chi}}%
\global\long\def\epstau{\epsilon_{\tau}}%
\global\long\def\TTRIP{\Tcal_{\mathsf{TRIP}}}%
\global\long\def\MTRIP{\mathcal{M}_{\mathsf{TRIP}}}%
\global\long\def\Tquant{\Tcal_{1}}%
\global\long\def\Tquantnorm{\Tcal_{1}^{\mathsf{norm}}}%
\global\long\def\Mquantnorm{\mathcal{M}_{1}^{\mathsf{norm}}}%
\global\long\def\setlowrank{\mathcal{R}}%

\global\long\def\xztau{(\bX,\bZ,\tau)}%
\global\long\def\xztauzero{(\bX_{0},\bZ_{0},\tau_{0})}%
\global\long\def\Echi{E^{\chi}}%
\global\long\def\Eind{E}%
\global\long\def\Emchi{E_{m}^{\chi}}%
\global\long\def\Emind{E_{m}}%

\global\long\def\neighbor{\mathcal{B}}%

\begin{abstract}
We study the problem of learning mixtures of low-rank models, i.e.~reconstructing
multiple low-rank matrices from unlabelled linear measurements of
each. This problem enriches two widely studied settings --- low-rank
matrix sensing and mixed linear regression --- by bringing latent
variables (i.e.~unknown labels) and structural priors (i.e.~low-rank
structures) into consideration. To cope with the non-convexity issues
arising from unlabelled heterogeneous data and low-complexity structure,
we develop a three-stage meta-algorithm that is guaranteed to recover
the unknown matrices with near-optimal sample and computational complexities
under Gaussian designs. In addition, the proposed algorithm is provably
stable against random noise. We complement the theoretical studies
with empirical evidence that confirms the efficacy of our algorithm.
\end{abstract}


\tableofcontents{}

\section{Introduction\label{sec:intro}}

This paper explores a mixture of low-rank models with latent variables,
which seeks to reconstruct a couple of low-rank matrices $\Mkstar\in\mathbb{R}^{n_{1}\times n_{2}}$
$(1\leq k\leq K)$ from \emph{unlabeled} linear measurements of each.
More specifically, what we have available is a collection of $N$
linear measurements $\{y_{i}\}_{1\leq i\leq N}$ taking the following
form:

\begin{equation}
y_{i}=\begin{cases}
\langle\boldsymbol{A}_{i},\Monestar\rangle,\quad & \text{if }i\in\Omegaonestar,\\
\dots & \dots\\
\langle\boldsymbol{A}_{i},\MKstar\rangle,\quad & \text{if }i\in\OmegaKstar,
\end{cases}\label{eq:mix_measurements}
\end{equation}
where $\{\Ai\}_{1\le i\le N}$ are the sampling/design matrices, $\langle\cdot,\cdot\rangle$
denotes the matrix inner product, and $\{\Omegakstar\}_{1\le k\le K}$
represents an unknown partition of the index set $\{1,\ldots,N\}$\emph{.
}The aim is to design an algorithm that is guaranteed to recover $\{\Mkstar\}$
efficiently and faithfully, despite the absence of knowledge of $\{\Omegakstar\}_{1\le k\le K}$. 

This problem of learning mixtures of low-rank models enriches two
widely studied settings: (1) it generalizes classical low-rank matrix
recovery \cite{recht2010guaranteed,chi2019nonconvex} by incorporating
heterogeneous data and latent variables (i.e.~the labels indicating
which low-rank matrices are being measured), and (2) it expands the
studies of mixed linear regression \cite{quandt1978estimating,yi2014alternating}
by integrating low-complexity structural priors (i.e.~low-rank structures).
In addition to the prior work \cite{yi2015regularized} that has studied
this setting, we single out two broader scenarios that bear relevance
to and motivate the investigation of mixtures of low-rank models.
\begin{itemize}
\item \emph{Mixed matrix completion}. If each measurement $y_{i}$ only
reveals a single entry of one of the unknown matrices $\{\bm{M}_{k}^{\star}\}$,
then the problem is commonly referred to as mixed matrix completion
(namely, completing several low-rank matrices from a mixture of unlabeled
observations of their entries) \cite{pimentel2018mixture}. One motivating
application arises from computer vision, where several problems like
joint shape matching can be posed as structured matrix completion
\cite{pmlr-v32-chend14,chen2018projected}. When the objects to be
matched exhibit certain geometric symmetry, there might exist multiple
plausible maps (and hence multiple ground-truth matrices), and the
provided observations might become intrinsically unlabeled due to
symmetric ambiguities \cite{sun2018joint}. Other applications include
network topology inference and metagenomics given mixed DNA samples;
see \cite{pimentel2018mixture} for details. 
\item \emph{Multi-task learning and meta-learning}. The model (\ref{eq:mix_measurements})
can be viewed as an instance of multi-task learning or meta-learning
\cite{baxter2000model,maurer2016benefit,kong2020meta}, where the
tasks follow a discrete prior distribution supported on a set of $K$
meta parameters, and each training data point $(\Ai,\yi)$ is a realization
of one task that comes with a single sample. While it is typically
assumed in meta-learning that even light tasks have more than one
samples, understanding this single-sample model is essential towards
tackling more general settings. Additionally, in comparison to meta-learning
for mixed linear regression \cite{kong2020meta,kong2020robust}, the
model (\ref{eq:mix_measurements}) imposes further structural prior
on the unknown meta parameters, thereby allowing for potential reduction
of sample complexities. 
\end{itemize}
The challenge for learning mixtures of low-rank models primarily stems
from the non-convexity issues. While the low-rank structure alone
already leads to non-convex optimization landscapes, the presence
of heterogeneous data and discrete hidden variables further complicates
matters significantly.

\subsection{Main contributions \label{subsec:meta_intuition}}

This paper takes a step towards learning mixtures of low-rank models,
focusing on the tractable Gaussian design where the $\bm{A}_{i}$'s
have i.i.d.~Gaussian entries; in light of this, we shall also call
the problem \emph{mixed matrix sensing}, to be consistent with the
terminology used in recent literature \cite{bhojanapalli2016global,chi2019nonconvex}.
In particular, we propose a meta-algorithm comprising the following
three stages: 
\begin{enumerate}
\item Estimate the joint column and row spaces of $\{\Mkstar\}_{1\le k\le K}$;
\item Transform mixed matrix sensing into low-dimensional mixed linear regression
using the above subspace estimates, and invoke a mixed linear regression
solver to obtain initial estimates of $\{\Mkstar\}_{1\le k\le K}$;
\item Successively refine the estimates via a non-convex low-rank matrix
factorization algorithm (more specifically, an algorithm called \emph{scaled
truncated gradient descent} to be described in Algorithm~\ref{alg:tsgd}). 
\end{enumerate}
The details of each stage will be spelled out and elucidated in Section
\ref{sec:alg}. 

Encouragingly, the proposed algorithm is guaranteed to succeed under
mild conditions (to be specified in Section \ref{subsec:models_assumptions}).
Informally, our contributions are three-fold. 
\begin{itemize}
\item \emph{Exact recovery in the noiseless case}. In the absence of noise,
our algorithm enables exact recovery of $\{\Mkstar\}$ modulo global
permutation. The sample complexity required to achieve this scales
linearly (up to some log factor) in the dimension $\max\{n_{1},n_{2}\}$
and polynomially in other salient parameters. 
\item \emph{Stability vis-\`a-vis random noise}. The proposed algorithm
is provably stable against Gaussian noise, in the sense that the estimation
accuracy degrades gracefully as the signal-to-noise-ratio decreases. 
\item \emph{Computational efficiency}. When the number $K$ of components
and the maximum rank of the unknown matrices are both constants,\textcolor{blue}{{}
}the computational cost of our algorithm scales nearly linearly in
$Nn_{1}n_{2}$ with $N$ the number of samples --- this is proportional
to the time taken to read all design matrices.
\end{itemize}
The precise theorem statements are postponed to Section \ref{sec:main_results}.
Empirical evidence will also be provided in Section \ref{sec:main_results}
to corroborate the efficacy of our algorithm.

\subsection{Notation\label{subsec:Notation}}

Before we proceed, let us collect some notation that will be frequently
used. Throughout this paper, we reserve boldfaced symbols for vectors
(lower case) and matrices (upper case). For a vector $\bx$, $\|\bx\|_{2}$
denotes its $\ell_{2}$ norm. For a matrix $\bX$, $\|\bX\|$ (resp.~$\|\bX\|_{{\rm F}}$)
denotes its spectral (resp.~Frobenius) norm, $\sigma_{k}(\bX)$ denotes
its $k$-th largest singular value, and $\col\{\bX\}$ (resp.~$\row\{\bX\}$)
denotes its column (resp.~row) space. If $\bU$ is a matrix with
orthonormal columns, we also use the same notation $\bU$ to represent
its column space, and vice versa. For any matrices $\bA,\bB\in\R^{n_{1}\times n_{2}}$,
let $\langle\bA,\bB\rangle\coloneqq\sum_{i=1}^{n_{1}}\sum_{j=1}^{n_{2}}A_{ij}B_{ij}$
stand for the matrix inner product. $\bI_{n}$ represents the $n\times n$
identity matrix. $\vc(\cdot)$ denotes vectorization of a matrix,
and $\mat(\cdot)$ denotes the inverse operation (the corresponding
matrix dimensions should often be clear from the context).

We use both $a_{n}\lesssim b_{n}$ and $a_{n}=O(b_{n})$ to indicate
that $a_{n}\le C_{0}b_{n}$ for some universal constant $C_{0}>0$;
in addition, $a_{n}\gtrsim b_{n}$ is equivalent to $b_{n}\lesssim a_{n}$,
and $a_{n}\asymp b_{n}$ means both $a_{n}\lesssim b_{n}$ and $b_{n}\lesssim a_{n}$
hold true. Finally, $a_{n}=o(b_{n})$ means that $a_{n}/b_{n}\rightarrow0$
as $n\rightarrow\infty$.

For a finite set $\Omega$, we denote by $|\Omega|$ its cardinality.
For a number $\alpha\in[0,1]$ and a random variable $X$ following
some distribution on $\R$, we let $\quantalpha(X)$ denote the $\alpha$-quantile
function, namely
\begin{equation}
Q_{\alpha}(X)\coloneqq\inf\big\{ t\in\R:\Pr(X\le t)\ge\alpha\big\}.\label{eq:def_quantile_distr}
\end{equation}
For a finite set $\Dcal$ of real numbers, with slight abuse of notation,
we let $Q_{\alpha}(\Dcal)$ be the $\alpha$-quantile of $\Dcal$;
more precisely, we define $Q_{\alpha}(\Dcal)\coloneqq Q_{\alpha}(X_{\Dcal})$,
where $X_{\Dcal}$ denotes a random variable uniformly drawn from
$\Dcal$.

\section{Algorithm \label{sec:alg}}

This section formalizes our algorithm design by specifying each stage
of our meta-algorithm with a concrete procedure (namely, Algorithms
\ref{alg:spectral}, \ref{alg:stage2}, \ref{alg:tsgd} for Stages 1,
2, 3, respectively). It is worth noting that these are definitely
not the only choices; in fact, an advantage of our meta-algorithm
is its flexibility and modularity, in the sense that one can plug
in different sub-routines to address various models and assumptions.

Before continuing, we introduce more notation that will be used throughout.
For any $1\le k\le K$, define

\begin{equation}
p_{k}\coloneqq\frac{|\Omegakstar|}{N}\qquad\text{and}\qquad r_{k}\coloneqq\rank(\Mkstar),\label{eq:defn-pk-rk}
\end{equation}
which represent the fraction of samples associated with the $k$-th
component and the rank of the $k$-th ground-truth matrix $\bm{M}_{k}^{\star}$,
respectively. In addition, let the compact singular value decomposition
(SVD) of $\left\{ \bm{M}_{k}^{\star}\right\} $~be 
\begin{equation}
\Mkstar=\Ukstar\Sigkstar\Vkstar^{\top},\quad\quad1\le k\le K,\label{eq:defn-Mk-svd}
\end{equation}
where $\Ukstar\in\mathbb{R}^{n_{1}\times r_{k}}$ and $\Vkstar\in\mathbb{R}^{n_{2}\times r_{k}}$
consist of orthonormal columns, and $\bm{\Sigma}_{k}^{\star}$ is
a diagonal matrix. 

\subsection{\label{subsec:alg_stage_1}Stage 1: subspace estimation via a spectral
method}

\paragraph{Procedure.}

We propose to estimate the following joint column and row spaces:
\begin{equation}
\Ustar\coloneqq\col\big\{[\Uonestar,\dots,\UKstar]\big\}\quad\text{and}\quad\Vstar\coloneqq\col\big\{[\Vonestar,\dots,\VKstar]\big\}\label{eq:def_Ustar_Vstar}
\end{equation}
by means of a spectral method. More specifically, we start by forming
a data matrix
\begin{equation}
\bY\coloneqq\frac{1}{N}\sumall y_{i}\Ai,\label{eq:data-matrix-spectral}
\end{equation}
and set $\Uhat\in\mathbb{R}^{n_{1}\times R}$ (resp.~$\Vhat\in\mathbb{R}^{n_{2}\times R}$)
to be a matrix whose columns consist of the top-$R$ left (resp.~right)
singular vectors of $\bY$, where
\begin{equation}
R\coloneqq\rank\big(\E[\bY]\big).\label{eq:defn-R}
\end{equation}
This method is summarized in Algorithm \ref{alg:spectral}.

\paragraph{Rationale. }

To see why this might work, note that if $\{\Ai\}$ consist of i.i.d.~standard
Gaussian entries, then
\begin{align}
\E[\bY]= & \sum_{k=1}^{K}p_{k}\E\big[\langle\Ai,\Mkstar\rangle\Ai\big]=\sumK\pkhat\Mkstar=\sumK\pkhat\Ukstar\Sigkstar\Vkstar^{\top}\nonumber \\
= & \big[\Uonestar,\Utwostar,\dots,\UKstar\big]\begin{bmatrix}\ponehat\Sigonestar & \boldsymbol{0} & \dots & \boldsymbol{0}\\
\boldsymbol{0} & \ptwohat\Sigtwostar & \boldsymbol{0} & \vdots\\
\vdots & \boldsymbol{0} & \ddots & \boldsymbol{0}\\
\boldsymbol{0} & \dots & \boldsymbol{0} & \pKhat\SigKstar
\end{bmatrix}\begin{bmatrix}\Vonestar^{\top}\\
\Vtwostar^{\top}\\
\vdots\\
\VKstar^{\top}
\end{bmatrix}.\label{eq:EY}
\end{align}
 Recalling the definitions of $\Ustar$ and $\Vstar$ in (\ref{eq:def_Ustar_Vstar}),
we have
\[
\col\big\{\E[\boldsymbol{Y}]\big\}=\Ustar,\quad\row\big\{\E[\boldsymbol{Y}]\big\}=\Vstar,\quad\rank(\Ustar)=\rank(\Vstar)=R
\]
under some mild conditions (detailed in Section \ref{sec:main_results}).
This motivates the development of Algorithm~\ref{alg:spectral}. 

\begin{algorithm}[tbp]
\DontPrintSemicolon
\caption{Subspace estimation via a spectral method} \label{alg:spectral}
{\bf Input:} {samples $\{\bA_i, y_i\}_{1 \le i \le N}$, rank $R$}. \\
Compute $\bY \gets \frac{1}{\Nall} \sum_{i=1}^{\Nall} y_i \bA_i$. \\
Let $\Uhat\in \mathbb{R}^{n_1\times R}$ (resp.~$\Vhat\in \mathbb{R}^{n_2\times R}$) be the matrix consisting of the top-$R$ left (resp.~right) singular vectors of $\bY$. \\
{\bf Output:} {$\Uhat, \Vhat$}.
\end{algorithm}

\subsection{Stage 2: initialization via low-dimensional mixed linear regression\label{subsec:alg_stage_2}}

\begin{algorithm}[tbp]
\DontPrintSemicolon
\caption{Initialization via low-dimensional mixed linear regression} \label{alg:stage2}
{\bf Input:} {samples $\{\bA_i, y_i\}_{1 \le i \le N}$,  subspaces $\Uhat, \Vhat$, ranks $\{r_k\}_{1 \le k \le K}$}. \\
Transform $\ai\gets\vc(\Uhat^{\top}\Ai\Vhat),  1 \le i \le N$. \\
Obtain $\{\betakhat\}_{1 \le k \le K} \gets$  the output of a black-box mixed linear regression solver (i.e.~Algorithm \ref{alg:tensor}) on  $\{\ai,\yi\}_{1 \le i \le N}$. \label{line:alg2_mlr} \\
\For{$k=1,\dots,K$}{ \label{line:alg2_forloop}
$\bU_k\boldsymbol{\Sigma}_k\bV_k^{\top} \gets $ rank-$r_{k}$ SVD of $\Uhat\Skhat\Vhat^{\top}$, where $\Skhat\coloneqq\mat(\betakhat)$. \label{line:alg2_svd} \\
$\bL_k\gets\bU_k\boldsymbol{\Sigma}_k^{1/2},\bR_k\gets\bV_k\boldsymbol{\Sigma}_k^{1/2}$. \label{line:alg2_LkRk}
}
{\bf Output:} {$\{\bL_{k},\bR_{k}\}_{1 \le k \le K}$}.
\end{algorithm}

\paragraph{Key observations.}

Suppose that there is an oracle informing us of the subspaces $\Ustar$
and $\Vstar$ defined in~(\ref{eq:def_Ustar_Vstar}). Recognizing
the basic relation $\Mkstar=\Ustar\Ustar^{\top}\Mkstar\Vstar\Vstar^{\top}$
and defining
\begin{equation}
\bS_{k}^{\star}\coloneqq\Ustar^{\top}\Mkstar\Vhat^{\star}\in\R^{R\times R},\qquad1\leq k\leq K,\label{eq:defn-Sk-star}
\end{equation}
we can rewrite the measurements in hand as follows:
\begin{equation}
y_{i}=\begin{cases}
\langle\boldsymbol{A}_{i},\Monestar\rangle=\langle\Ai,\Ustar\Ustar^{\top}\Monestar\Vstar\Vstar^{\top}\rangle=\langle\Ustar^{\top}\Ai\Vstar,\bS_{1}^{\star}\rangle,\quad & \text{if }i\in\Omegaonestar,\\
\dots & \dots\\
\langle\boldsymbol{A}_{i},\MKstar\rangle=\langle\Ustar^{\top}\Ai\Vstar,\bS_{K}^{\star}\rangle,\quad & \text{if }i\in\OmegaKstar.
\end{cases}\label{eq:low_dim_mlr-1}
\end{equation}
In other words, the presence of the oracle effectively reduces the
original problem into a mixed linear regression problem in lower dimensions
--- that is, the problem of recovering $\{\bS_{k}^{\star}\}$ from
mixed linear measurements. If $\{\bS_{k}^{\star}\}$ can be reliably
estimated, then one can hope to recover $\{\bM_{k}^{\star}\}$ via
the following relation:
\begin{equation}
\bm{M}_{k}^{\star}=\bm{U}^{\star}\bm{U}^{\star\top}\bm{M}_{k}^{\star}\bm{V}^{\star}\bm{V}^{\star\top}=\bm{U}^{\star}\bm{S}_{k}^{\star}\bm{V}^{\star\top}.\label{eq:M-k-recovery-through-S}
\end{equation}

\paragraph{Procedure.}

While we certainly have no access to the aforementioned oracle in
reality, Stage 1 described above provides us with subspace estimates
$\Uhat$ and $\Vhat$ satisfying $\Uhat\Uhat^{\top}\approx\Ustar\Ustar^{\top}$
and $\Vhat\Vhat^{\top}\approx\Vstar\Vstar^{\top}$. Treating these
as surrogates of $(\bm{U}^{\star},\bm{V}^{\star})$ (so that $\bm{M}_{k}^{\star}\approx\bm{U}\bm{U}^{\top}\bm{M}_{k}^{\star}\bm{V}\bm{V}^{\top}$),
we can view the measurements as 
\begin{equation}
y_{i}=\begin{cases}
\langle\boldsymbol{A}_{i},\Monestar\rangle\approx\langle\Ai,\bm{U}\bm{U}^{\top}\Monestar\bm{V}\bm{V}^{\top}\rangle=\langle\bm{U}^{\top}\Ai\bm{V},\bm{S}_{1}\rangle=\langle\bm{a}_{i},\bm{\beta}_{1}\rangle,\quad & \text{if }i\in\Omegaonestar,\\
\dots & \dots\\
\langle\boldsymbol{A}_{i},\MKstar\rangle\approx\langle\bm{a}_{i},\bm{\beta}_{K}\rangle,\quad & \text{if }i\in\OmegaKstar,
\end{cases}\label{eq:low_dim_mlr}
\end{equation}
which are mixed linear measurements about the following vectors/matrices:
\begin{equation}
\quad\quad\bbeta_{k}\coloneqq\vc\big(\bS_{k}\big)\in\R^{R^{2}},\qquad\bS_{k}\coloneqq\Uhat^{\top}\Mkstar\Vhat\in\R^{R\times R},\qquad1\le k\le K.\label{eq:def_Skhat}
\end{equation}
Here, the equivalent sensing vectors are defined to be $\ai\coloneqq\vc\big(\Uhat^{\top}\Ai\Vhat\big)\in\R^{R^{2}}$
for any $1\leq i\leq N$. All this motivates us to resort to mixed
linear regression algorithms for recovering $\{\bbeta_{k}\}$. The
proposed algorithm thus entails the following steps, with the precise
procedure summarized in Algorithm \ref{alg:stage2}. 
\begin{itemize}
\item Invoke any mixed linear regression algorithm to obtain estimates $\{\widehat{\bbeta}_{k}\}_{1\leq k\leq K}$
for $\{\bm{\beta}_{k}\}_{1\leq k\leq K}$ (up to global permutation).
For concreteness, the current paper applies the tensor method (Algorithm
\ref{alg:tensor}) originally proposed in \cite{yi2016solving}; this
is a polynomial-time algorithm, with details deferred to Appendix~\ref{sec:alg_MLR}.
To simplify presentation, let us assume here that the global permutation
happens to be an identity map, so that $\widehat{\bm{\beta}}_{k}$
is indeed a faithful estimate of $\bm{\beta}_{k}$ $(1\leq k\leq K)$.
By simple matricization, $\widehat{\bm{\beta}}_{k}$ leads to a reliable
estimate $\Skhat$ of $\bm{S}_{k}$.
\item Given the observation that
\begin{equation}
\Uhat\Skhat\Vhat^{\top}\approx\Uhat\bS_{k}\Vhat^{\top}=\Uhat\Uhat^{\top}\Mkstar\Vhat\Vhat^{\top}\approx\Ustar\Ustar^{\top}\Mkstar\Vstar\Vstar^{\top}=\Mkstar,\label{eq:Sk_Mpik}
\end{equation}
we propose to compute the rank-$r_{k}$ SVD --- denoted by $\bU_{k}\boldsymbol{\Sigma}_{k}\bV_{k}^{\top}$
--- of the matrix $\Uhat\Skhat\Vhat^{\top}$ for each $1\leq k\leq K$.
This in turn leads to our initial estimate for the low-rank factors
\begin{equation}
\bL_{k}\coloneqq\bU_{k}\boldsymbol{\Sigma}_{k}^{1/2}\in\R^{n_{1}\times r_{k}},\quad\text{and}\quad\bR_{k}\coloneqq\bV_{k}\boldsymbol{\Sigma}_{k}^{1/2}\in\R^{n_{2}\times r_{k}}.\label{eq:tsgd_init}
\end{equation}
\end{itemize}

\subsection{Stage 3: local refinement via scaled truncated gradient descent (ScaledTGD)\label{subsec:alg_stage_3}}

Suppose that an initial point $\bm{L}^{0}(\bm{R}^{0})^{\top}$ lies
within a reasonably small neighborhood of $\bm{M}_{k}^{\star}$ for
some $1\leq k\leq K$. Stage 3 serves to locally refine this initial
estimate, moving it closer to our target $\bm{M}_{k}^{\star}$. Towards
this end, we propose to deploy the following update rule termed \emph{scaled
truncated gradient descent} (ScaledTGD): 
\begin{subequations}
\label{eq:tsgd_updates}
\begin{align}
\bm{L}^{t+1} & =\bL^{t}-\frac{\eta}{\Nall}\sum_{i\in\Omega^{t}}\big(\langle\Ai,\bL^{t}(\bR^{t})^{\top}\rangle-\yi\big)\Ai\bR^{t}\big((\bR^{t})^{\top}\bR^{t}\big)^{-1},\label{eq:tsgd_L}\\
\bm{R}^{t+1} & =\bR^{t}-\frac{\eta}{\Nall}\sum_{i\in\Omega^{t}}\big(\langle\Ai,\bL^{t}(\bR^{t})^{\top}\rangle-\yi\big)\Ai^{\top}\bL^{t}\big((\bL^{t})^{\top}\bL^{t}\big)^{-1},\label{eq:tsgd_R}
\end{align}
\end{subequations}
where $\eta>0$ denotes the step size. Here, $\Omega^{t}\subseteq\{1,2,\cdots,N\}$
is an adaptive and iteration-varying index set designed to mimic the
index set $\Omega_{k}^{\star}$. Indeed, if $\Omega^{t}=\Omega_{k}^{\star}$,
the aforementioned update rule reduces to the ScaledGD method developed
for vanilla low-rank matrix sensing (see~\cite{tong2020accelerating}),
which is guaranteed to converge to $\bm{M}_{k}^{\star}$ in the presence
of a suitable initialization. Here, the rescaling matrix $\big((\bR^{t})^{\top}\bR^{t}\big)^{-1}$
(resp.~$\big((\bL^{t})^{\top}\bL^{t}\big)^{-1}$) acts as a pre-conditioner
of the conventional gradient $\sum_{i\in\Omega^{t}}\big(\langle\Ai,\bL^{t}(\bR^{t})^{\top}\rangle-\yi\big)\Ai\bR^{t}$
(resp.~$\sum_{i\in\Omega^{t}}\big(\langle\Ai,\bL^{t}(\bR^{t})^{\top}\rangle-\yi\big)\Ai^{\top}\bL^{t}$),
which effectively accelerates convergence when $\bm{M}_{k}^{\star}$
is ill-conditioned. See~\cite{tong2020accelerating,tong2020low} for more intuitions
and justifications of this rescaling strategy. 

Viewed in this light, the key to ensuring effectiveness of ScaledTGD
lies in the design of the index set $\Omega^{t}$. If we know \emph{a
priori} that $\bm{L}^{t}(\bR^{t})^{\top}\approx\bm{M}_{k}^{\star}$,
then it is intuitively clear that $|\langle\bA_{i},\bL^{t}(\bR^{t})^{\top}\rangle-y_{i}|$
typically has a smaller scale for a sample $i\in\Omega_{k}^{\star}$
when compared with those $i\notin\Omega_{k}^{\star}$. This motivates
us to include in $\Omega^{t}$ a certain fraction (denoted by $0<\alpha<1$)
of samples enjoying the smallest empirical loss $|\langle\bA_{i},\bL^{t}(\bR^{t})^{\top}\rangle-y_{i}|$.
Intuitively, the fraction $\alpha$ should not be too large in which
case $\Omega^{t}$ is likely to contain samples outside $\Omega_{k}^{\star}$;
on the other hand, $\alpha$ should not be chosen too small in order
not to waste information. As it turns out, choosing $0.6\pkhat\le\alpha\le0.8\pkhat$
strikes a suitable balance and works well for our purpose. See Algorithm~\ref{alg:tsgd}
for a precise description.  

\begin{algorithm}[tbp]
\DontPrintSemicolon
\caption{Scaled Truncated Gradient Descent (ScaledTGD) for recovering $\Mkstar$} \label{alg:tsgd}
{\bf Input:}  samples $\{\Ai, y_i\}_{1 \le i \le N}$, initialization $\bL^{0}\in\R^{n_1 \times r_k},\bR^{0}\in\R^{n_2 \times r_k}$,     step size $\eta$, truncating fraction $\alpha$. \\
\For{$t = 0, 1, 2, \dots, T_0 - 1$}{
$$
\bL^{t+1}\gets\bL^{t}-\frac{\eta}{\Nall}\sum_{i\in\Omega^t}\big(\langle\Ai,\bL^{t}(\bR^{t})^{\top}\rangle-\yi\big)\Ai\bR^{t}\big((\bR^{t})^{\top}\bR^{t}\big)^{-1},
$$
$$
\bR^{t+1}\gets\bR^{t}-\frac{\eta}{\Nall}\sum_{i\in\Omega^t}\big(\langle\Ai,\bL^{t}(\bR^{t})^{\top}\rangle-\yi\big)\Ai^{\top}\bL^{t}\big((\bL^{t})^{\top}\bL^{t}\big)^{-1},
$$
where $\Omega^{t}\coloneqq\{1 \le i \le N:|\langle\Ai,\bL^{t}(\bR^{t})^{\top}\rangle-\yi|\le\tau_{t}\}$, $\tau_t\coloneqq Q_{\alpha}(\{|\langle\Ai,\bL^{t}(\bR^{t})^{\top}\rangle-\yi|\}_{1 \le i \le N})$.
}
{\bf Output:} $\bL^{T_0}, \bR^{T_0}$.
\end{algorithm}

\subsection{The full algorithm\label{subsec:full_algorithm}}

\begin{algorithm}[tbp]
\DontPrintSemicolon
\caption{A fully specified three-stage algorithm for mixed matrix sensing} \label{alg:full}
{\bf Input:} independent samples $\{\Ai, y_i\}_{1 \le i \le N}, \{\Ai', \yi'\}_{1 \le i \le \NMLR}$, parameters $R,\{r_{k},\eta_{k},\alpha_{k}\}_{1 \le k \le K}$ (see Table \ref{tab:alg_params}). \\
Run Algorithm \ref{alg:spectral} with $\{\Ai, y_i\}_{1 \le i \le N}$ and $R$ to obtain $\Uhat, \Vhat$. \label{line:stage1} \\
Run Algorithm \ref{alg:stage2} with $\{\Ai', \yi'\}_{1 \le i \le \NMLR}, \Uhat, \Vhat$ and $\{r_k\}_{1 \le k \le K}$ to obtain $\{\bL_{k},\bR_{k}\}_{1 \le k \le K}$. \label{line:stage2} \\
\For{$k = 1, 2, \dots, K$}{
Run Algorithm \ref{alg:tsgd} on $\{\Ai, y_i\}_{1 \le i \le N}$ with $(\bL^{0},\bR^{0})\gets(\bL_{k},\bR_{k}),\eta_{k},\alpha_{k}$ to obtain $\bL^{T_{0}},\bR^{T_{0}}$. \label{line:stage3} \\ 
Set $\bM_{k}\gets\bL^{T_{0}}(\bR^{T_{0}})^{\top}$.
}
{\bf Output:} $\{\bM_{k}\}_{1 \le k \le K}$.
\end{algorithm}

\begin{table}
\begin{centering}
\caption{\label{tab:alg_params}Our choices of the algorithmic parameters in
Algorithm \ref{alg:full}.}
\par\end{centering}
\centering{}%
\begin{tabular}{ll}
\toprule 
Algorithm \ref{alg:spectral} & Rank $R=\rank(\sum_{k}\pkhat\Mkstar)$.\tabularnewline
\midrule 
Algorithm \ref{alg:stage2} & Ranks $r_{k}=\rank(\Mkstar),1\le k\le K$.\tabularnewline
\midrule 
Algorithm \ref{alg:tsgd} (for $\Mkstar$) & Step size $0<\eta_{k}\le1.3/p_{k}$, truncating fraction $0.6p_{k}\le\alpha_{k}\le0.8p_{k}$.\tabularnewline
\bottomrule
\end{tabular}
\end{table}

With the three stages fully described, we can specify the whole algorithm
in Algorithm~\ref{alg:full}, with the choices of algorithmic parameters
listed in Table \ref{tab:alg_params}. Note that the discussion in
Section \ref{subsec:alg_stage_3} focuses on estimating a single component;
in order to recover all $K$ components $\{\Mkstar\}_{1\le k\le K}$,
we simply need to run Algorithm~\ref{alg:tsgd} for $K$ times (which
can be executed in parallel). In addition, Algorithm \ref{alg:full}
is built upon sample splitting: while Stages 1 and 3 employ the same
set of samples $\{\Ai,\yi\}_{1\le i\le N}$, Stage 2 (i.e.~Line~\ref{line:stage2}
of Algorithm \ref{alg:full}) operates upon an \emph{independent}
set of samples $\{\Ai',\yi'\}_{1\le i\le\NMLR}$ (where ``$\mathsf{MLR}$''
stands for ``mixed linear regression''), thus resulting in a total
sample complexity of $N+\NMLR$. The main purpose of sample splitting
is to decouple statistical dependency across stages and facilitate
analysis. Finally, the interested reader is referred to Appendix \ref{sec:estimate_parameters}
for a discussion regarding how to estimate certain parameters in Algorithm
\ref{alg:full} if they are not known \emph{a priori}.

\section{Main results\label{sec:main_results}}

\subsection{Models and assumptions\label{subsec:models_assumptions}}

For notational convenience, let us define the following parameters:
\begin{equation}
n\coloneqq\max\{n_{1},n_{2}\},\quad\,\,r\coloneqq\max_{1\le k\le K}r_{k},\,\,\quad\kappa\coloneqq\max_{1\le k\le K}\kappa(\Mkstar),\,\,\quad\text{and}\quad\Gamma\coloneqq\frac{\max_{1\le k\le K}\|\Mkstar\|_{\Frm}}{\min_{1\le k\le K}\|\Mkstar\|_{\Frm}},\label{eq:def_params}
\end{equation}
where $\kappa(\Mkstar)\coloneqq\sigma_{1}(\Mkstar)/\sigma_{r_{k}}(\Mkstar)$
stands for the condition number of $\Mkstar$. This paper focuses
on the \emph{Gaussian design}, where the entries of each design matrix
$\Ai$ are independently drawn from the standard Gaussian distribution.
In addition, we assume that the samples drawn from the $K$ components
are reasonably \emph{well-balanced} in the sense that for all $1\le k\le K$,
\begin{equation}
p_{k}=\frac{|\Omegakstar|}{N}\asymp\frac{1}{K},\label{eq:balanced_pk}
\end{equation}
where $\Omegakstar$ is the index set for the $k$-th component (see
(\ref{eq:mix_measurements})). We assume that this well-balancedness
assumption holds for both sets of samples $\{\Ai,\yi\}_{1\le i\le N}$
and $\{\Ai',\yi'\}_{1\le i\le\NMLR}$.  

Next, we introduce an incoherence parameter that plays a crucial role
in our theoretical development. 
\begin{defn}
\emph{The incoherence parameter $\mu\ge0$} is the smallest quantity
that satisfies
\begin{equation}
\big\|{\Uistar}^{\top}\Ujstar\big\|_{\Frm}\le\frac{\mu r}{\sqrt{n_{1}}},\qquad\text{and}\qquad\big\|{\Vistar}^{\top}\Vjstar\big\|_{\Frm}\le\frac{\mu r}{\sqrt{n_{2}}}\qquad\text{for all }1\leq i<j\leq K.\label{eq:incoherence-defn}
\end{equation}
\end{defn}
The incoherence parameter $\mu$ takes value on $[0,\sqrt{n/r}]$.
As an example, if $\{\Ukstar\}_{1\le k\le K}$ (resp.~$\{\Vkstar\}_{1\le k\le K}$)
are random low-dimensional subspaces in $\R^{n_{1}}$ (resp.~$\R^{n_{2}}$),
then for any $i\neq j$, $\|{\Uistar}^{\top}\Ujstar\|_{\Frm}$ (resp.~$\|{\Vistar}^{\top}\Vjstar\|_{\Frm}$)
is on the order of $\sqrt{r_{i}r_{j}/n_{1}}$ (resp.~$\sqrt{r_{i}r_{j}/n_{2}}$),
which is further upper bounded by $r/\sqrt{n_{1}}$ (resp.~$r/\sqrt{n_{2}}$).
This observation motivates our definition of the incoherence parameter.
One of our main technical assumptions is that the column (resp.~row)
spaces of the ground-truth matrices are mutually weakly correlated
--- defined through the parameter $\mu$ --- which covers a broad
range of settings.
\begin{assumption}
\label{asp:weak_correlation} The incoherence parameter $\mu$ is
upper bounded by
\begin{equation}
\mu\le\frac{\sqrt{\min\{n_{1},n_{2}\}}}{2r\max\{K,\sqrt{K}\Gamma\}}.\label{eq:defn-weak-correlation}
\end{equation}
\end{assumption}

\subsection{Theoretical guarantees \label{subsec:theoretical-guarantees}}

\paragraph{Exact recovery in the absence of noise. }

Our first main result uncovers that, in the noiseless case, Algorithm~\ref{alg:full}
achieves exact recovery efficiently, in terms of both sample and computational
complexities.
\begin{thm}
[Exact recovery] \label{thm:noiseless} Consider the noiseless case
\eqref{eq:mix_measurements} under the assumptions in Section~\ref{subsec:models_assumptions}.
Suppose 
\begin{equation}
N\ge C_{1}K^{3}r^{2}\kappa^{2}\Gamma^{2}\max\{K^{2}\Gamma^{4},r\kappa^{2}\}\cdot n\log N\,\,\,\,\text{and}\,\,\,\,\NMLR\ge C_{2}K^{8}r^{2}\Gamma^{12}\max\{K^{2},r\kappa^{2}\}\cdot\log n\cdot\log^{3}\NMLR\label{eq:sample_complexity_thm}
\end{equation}
for some sufficiently large constants $C_{1},C_{2}>0$. Then with
probability at least $1-o(1)$, there exists some permutation $\pi:\{1,\dots,K\}\mapsto\{1,\dots,K\}$
such that the outputs of Algorithm \ref{alg:full} obey for all $1\leq k\leq K$
\begin{equation}
\big\|\bm{M}_{\pi(k)}-\bM_{k}^{\star}\big\|_{\Frm}\le\big(1-c_{0}\eta_{k}p_{k}\big)^{T_{0}}\big\|\bM_{k}^{\star}\big\|_{\Frm}\label{eq:tsgd_linear_conv_noiseless}
\end{equation}
for some universal constant $0<c_{0}<1/4$, where $T_{0}$ is the
number of iterations used in Algorithm \ref{alg:tsgd}. 
\end{thm}
The proof can be found in Section \ref{sec:analysis}. Two implications
are in order.
\begin{itemize}
\item Suppose that the parameters $K,r,\kappa,\Gamma=O(1)$. In order to
achieve exact recovery, the sample size $N$ in~(\ref{eq:sample_complexity_thm})
only needs to scale as $O(n\log n)$, while $\NMLR$ only needs to
exceed the order of {$\log n \cdot \log\log^3 n$}. 
\item By setting the step size $\eta_{k}=c_{1}/\pkhat$ for some constant
$0<c_{1}\le1.3$, we see that the third stage (i.e.~ScaledTGD) achieves
linear convergence with a \emph{constant} contraction rate, which
is independent of the condition number $\kappa(\bm{M}_{k}^{\star})$
of the matrix $\bm{M}_{k}^{\star}$. 
\end{itemize}

\paragraph{Stability vis-\`a-vis noise. }

Moving on to the more realistic case with noise, we consider the following
set of samples $\{\Ai,\yi\}_{1\le i\le N}$:
\begin{equation}
\ei\iid\Ncal\left(0,\sigma^{2}\right),\quad\quad y_{i}=\begin{cases}
\langle\boldsymbol{A}_{i},\Monestar\rangle+\ei, & \text{if }i\in\Omegaonestar,\\
\dots\\
\langle\boldsymbol{A}_{i},\MKstar\rangle+\ei, & \text{if }i\in\OmegaKstar.
\end{cases}\label{eq:mix_measurements_noisy}
\end{equation}
The set $\{\Ai',\yi'\}_{1\le i\le\NMLR}$ is independently generated
in a similar manner. Our next result reveals that the proposed algorithm
is stable against Gaussian noise. The proof is postponed to Section
\ref{sec:analysis}. 
\begin{thm}
[Stable recovery] \label{thm:noisy} Consider the noisy model (\ref{eq:mix_measurements_noisy})
under the assumptions of Section~\ref{subsec:models_assumptions}.
Suppose that the sample sizes satisfy \eqref{eq:sample_complexity_thm},
and that the noise level satisfies
\begin{equation}
\sigma\le c\min_{1\le k\le K}\|\Mkstar\|_{\Frm}\cdot\min\left\{ \frac{1}{K},\frac{1}{\sqrt{r}\kappa}\right\} \label{eq:sigma_bound}
\end{equation}
for some sufficiently small constant $c>0$. Then with probability
at least $1-o(1)$, there exists some permutation $\pi:\{1,\dots,K\}\mapsto\{1,\dots,K\}$
such that the outputs of Algorithm \ref{alg:full} obey for all $1\leq k\leq K$
\begin{align}
\big\|\bm{M}_{\pi(k)}-\bM_{k}^{\star}\big\|_{\Frm} & \le\big(1-c_{0}\eta_{k}p_{k}\big)^{T_{0}}\big\|\bM_{k}^{\star}\big\|_{\Frm}+C_{0}\max\left\{ \sige\sqrt{\frac{nrK^{3}\log N}{N}},\,\frac{K\sige^{2}}{\min_{j:j\neq k}\|\Mjstar-\bM_{k}^{\star}\|_{\Frm}}\right\} ,\label{eq:tsgd_linear_conv_noisy}
\end{align}
where $0<c_{0}<1/4$ and $C_{0}>0$ are some universal constants,
and $T_{0}$ is the number of iterations used in Algorithm \ref{alg:tsgd}.\textcolor{blue}{}
\end{thm}
Theorem \ref{thm:noisy} asserts that, when initialized using the
proposed schemes, the ScaledTGD algorithm converges linearly until
an error floor is hit. To interpret the statistical guarantees (\ref{eq:tsgd_linear_conv_noisy}),
we find it helpful to define the signal-to-noise-ratio (SNR) w.r.t.~$\Mkstar$
as follows:
\begin{equation}
\SNR_{k}\coloneqq\frac{\mathbb{E}\left[\big|\langle\bm{A}_{i},\bm{M}_{k}^{\star}\rangle\big|^{2}\right]}{\mathbb{E}[\zeta_{i}^{2}]}=\frac{\|\Mkstar\|_{\Frm}^{2}}{\sigma^{2}}.\label{eq:def_snr}
\end{equation}
This together with the simple consequence $\min_{j:j\neq k}\|\Mjstar-\Mkstar\|_{\Frm}\gtrsim\|\Mkstar\|_{\Frm}$
of Assumption~\ref{asp:weak_correlation} implies that
\begin{equation}
\frac{\|\bm{M}_{\pi(k)}-\Mkstar\|_{\Frm}}{\|\Mkstar\|_{\Frm}}\lesssim\max\left\{ \frac{1}{\sqrt{\SNR_{k}}}\sqrt{\frac{nrK^{3}\log N}{N}},\,\frac{K}{\SNR_{k}}\right\} \label{eq:normalized_error}
\end{equation}
as long as the iteration number $T_{0}$ is sufficiently large. Here,
the first term on the right-hand side of (\ref{eq:normalized_error})
 matches the \emph{minimax lower bound} for low-rank matrix sensing
\cite[Theorem 2.5]{candes2011tight} (the case with $K=1)$ up to
a factor of $K\sqrt{\log N}$. In contrast, the second term on the
right-hand side of (\ref{eq:normalized_error}) --- which becomes
very small as $\SNR_{k}$ grows --- is not a function of the sample
size $N$ and does not vanish as $N\rightarrow\infty$. This term
arises since, even at the population level, the point $(\bL,\bR)$
satisfying $\bL\bR^{\top}=\Mkstar$ is not a fixed point of the ScaledTGD
update rule, due to the presence of mislabeled samples. 

\subsection{Numerical experiments }

\begin{figure}
\begin{centering}
\begin{minipage}[t]{0.33\textwidth}%
\begin{center}
\includegraphics[width=1\textwidth]{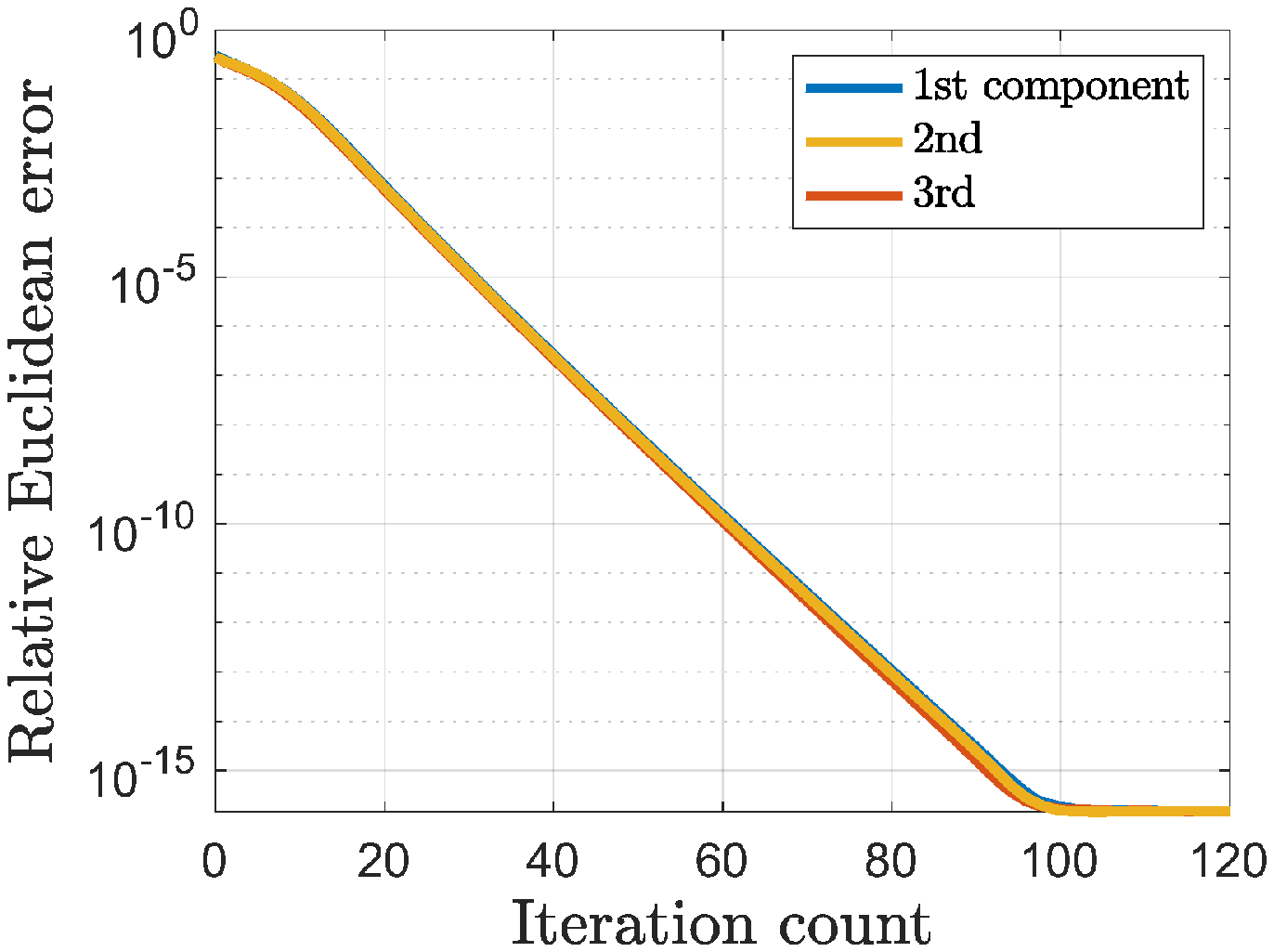}
\par\end{center}
\begin{center}
(a) 
\par\end{center}%
\end{minipage}%
\begin{minipage}[t]{0.33\columnwidth}%
\begin{center}
\includegraphics[width=1\textwidth]{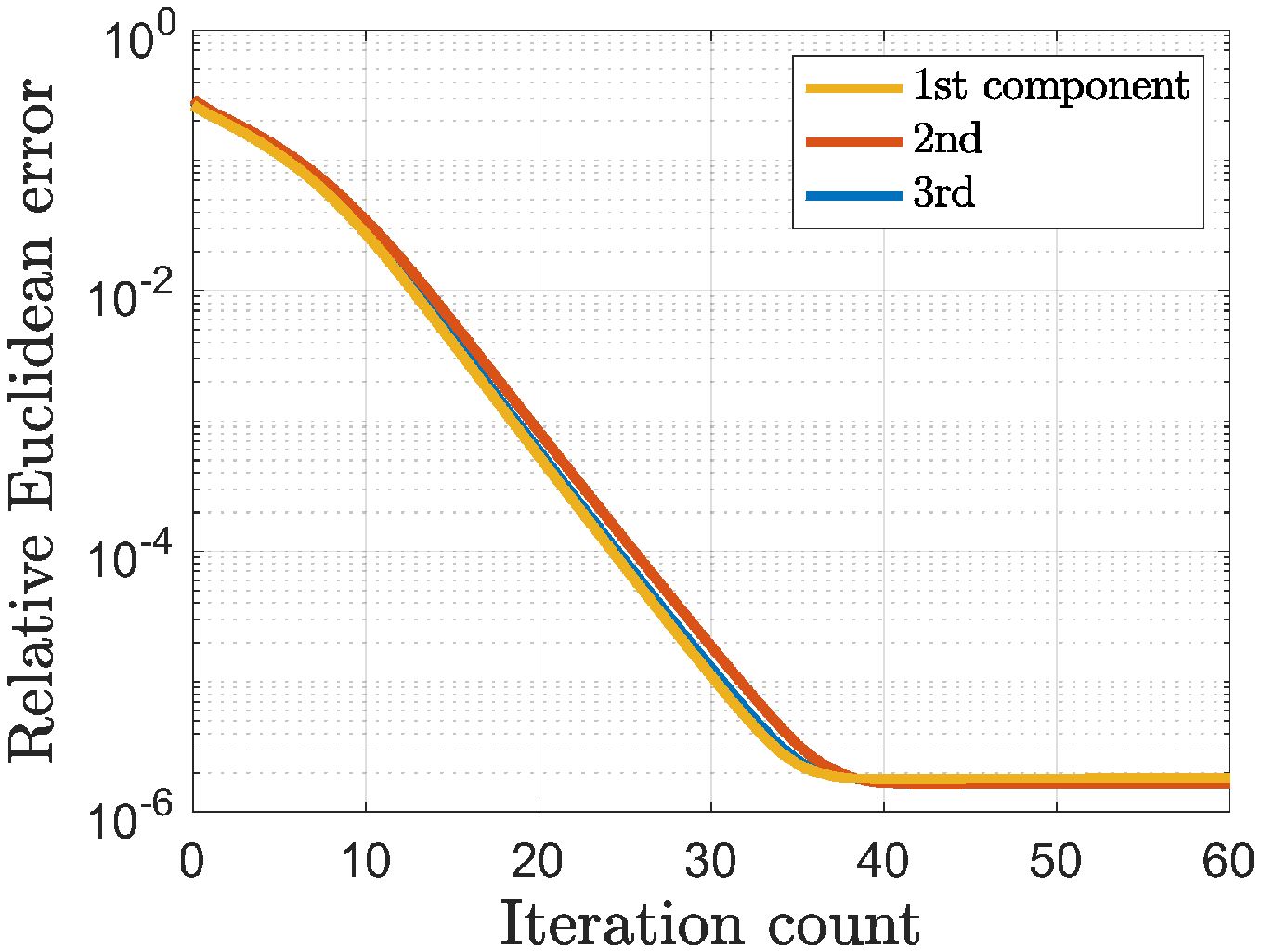}
\par\end{center}
\begin{center}
(b) 
\par\end{center}%
\end{minipage}%
\begin{minipage}[t]{0.33\columnwidth}%
\begin{center}
\includegraphics[width=1\textwidth]{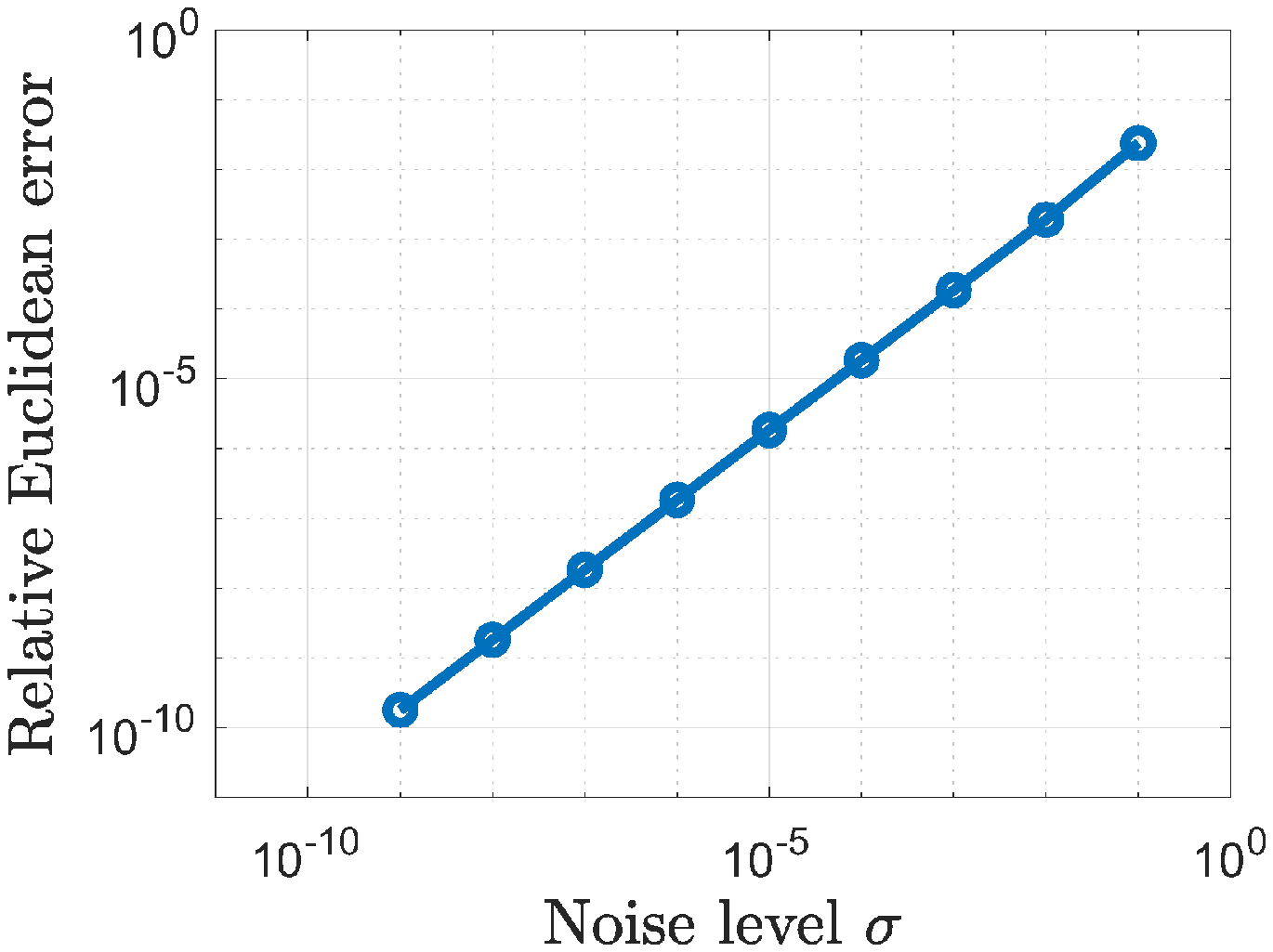}
\par\end{center}
\begin{center}
(c) 
\par\end{center}%
\end{minipage}
\par\end{centering}
\caption{\label{fig:convergence_tensor} (a) The relative Euclidean error vs.~the
iteration count of ScaledTGD in Stage 3 of Algorithm~\ref{alg:full}
for each of the three components, in the noiseless case. (b) Convergence
of ScaledTGD in the noisy case $\sigma=10^{-5}$. (c) The largest
relative Euclidean error (at convergence) of ScaledTGD in Algorithm
\ref{alg:full}, vs.~the noise level $\sigma$. Each data point is
an average over 10 independent trials. }
\end{figure}

To validate our theoretical findings, we conduct a series of numerical
experiments. To match practice, we do not deploy sample splitting
(given that it is merely introduced to simplify analysis), and reuse
the same dataset of size $N$ for all three stages. Throughout the
experiments, we set $n_{1}=n_{2}=n=120$, $r=2$, and $K=3$. For
each $k$, we let $\pkhat=1/K$ and $\Sigkstar=\boldsymbol{I}_{r}$,
and generate $\Ukstar$ and $\Vkstar$ as random $r$-dimensional
subspaces in $\R^{n}$. We fix the sample size to be $N=90\,nrK$.
The algorithmic parameters are chosen according to our recommendations
in Table~\ref{tab:alg_params}. For instance, for each run of ScaledTGD,
we set the step size as $\eta=1.3K$ and the truncation fraction as
$\alpha=0.8/K$. 

\paragraph{Linear convergence of ScaledTGD. }

Our first series of experiments aims at verifying the linear convergence
of ScaledTGD towards the ground-truth matrices $\{\bm{M}_{k}^{\star}\}$
when initialized using the outputs of Stage 2. We consider both the
noiseless case (i.e.~$\sigma=0$) and the noisy case $\sigma=10^{-5}$.
Figures~\ref{fig:convergence_tensor}(a) and \ref{fig:convergence_tensor}(b)
plot the relative Euclidean error $\|\bL^{t}(\bR^{t})^{\top}-\Mkstar\|_{\Frm}/\|\Mkstar\|_{\Frm}$
versus the iteration count $t$ for each component $1\leq k\leq3$.
It is easily seen from Figures~\ref{fig:convergence_tensor}(a) and
\ref{fig:convergence_tensor}(b) that ScaledTGD, when seeded with
the outputs from Stage 2, converges linearly to the ground-truth matrices
$\{\bm{M}_{k}^{\star}\}$ in the absence of noise, and to within a
small neighborhood of $\{\bm{M}_{k}^{\star}\}$ in the noisy setting.

\paragraph{Estimation error in the presence of random noise.}

The second series of experiments investigates the stability of the
three-stage algorithm in the presence of random noise. We vary the
noise level within $[10^{-9},10^{-1}]$. Figure~\ref{fig:convergence_tensor}(c)
plots the largest relative Euclidean error $\max_{1\leq k\leq K}\|\bM_{k}-\Mkstar\|_{\Frm}/\|\Mkstar\|_{\Frm}$
(where $\{\bM_{k}\}$ are the outputs of Algorithm~\ref{alg:full})
versus the noise level $\sigma$, showing that the recovering error
is indeed linear in $\sigma$, as predicted by our theory.

\section{Prior work\label{sec:prior_work}}

\paragraph{Low-rank matrix recovery.}

There exists a vast literature on low-rank matrix recovery (e.g.~\cite{candes2009exact,keshavan2010matrix,bhojanapalli2016global,chen2017solving,ma2017implicit,chen2019gradient,sun2016guaranteed,candes2015phase,jain2013low,chen2019noisy,chen2019inference,sun2018geometric,chen2020nonconvex,ding2020leave,netrapalli2014non,chen2015exact,charisopoulos2019low,arora2019implicit,zhang2020symmetry,zhang2018primal,li2018algorithmic,chen2020bridging,park2017non});
we refer the readers to \cite{chen2018harnessing,chi2019nonconvex,chen2020spectral}
for an overview of this extensively studied topic. Most related to
our work is the problem of matrix sensing (or low-rank matrix recovery
from linear measurements). While convex relaxation \cite{candes2009exact,recht2010guaranteed,candes2011tight}
enjoys optimal statistical performance, two-stage non-convex approaches
\cite{zheng2015convergent,tu2016low,tong2020accelerating} have received
growing attention in recent years, due to their ability to achieve
statistical and computational efficiency at once. Our three-stage
algorithm is partially inspired by the two-stage approach along this
line. It is worth mentioning that the non-convex loss function associated
with low-rank matrix sensing enjoys benign landscape, which in turn
enables tractable global convergence of simple first-order methods
\cite{bhojanapalli2016global,ge2017no,zhu2018global,li2018algorithmic,li2019non}.

\paragraph{Mixed linear regression.}

Being a classical problem in statistics \cite{quandt1978estimating},
mixed linear regression has attracted much attention due to its broad
applications in music perception \cite{de1989mixtures,viele2002modeling},
health care \cite{deb2000estimates}, trajectory clustering \cite{gaffney1999trajectory},
plant science \cite{turner2000estimating}, neuroscience \cite{yin2018learning},
to name a few. While computationally intractable in the worst case
\cite{yi2014alternating}, mixed linear regression can be solved efficiently
under certain statistical models on the design matrix. Take the two-component
case for instance: efficient methods include alternating minimization
with initialization via grid search \cite{yi2014alternating}, EM
with random initialization \cite{klusowski2019estimating,kwon2019global},
and convex reformulations \cite{chen2014convex,hand2018convex}, where
EM further achieves minimax estimation guarantees \cite{chen2014convex}
in the presence of Gaussian noise \cite{kwon2020minimax}. Mixed linear
regression becomes substantially more challenging when the number
$K$ of components is allowed to grow with $n$. 
{Recently, \cite{diakonikolas2020small} achieves quasi-polynomial sample and computational complexities w.r.t.~$K$; other existing methods either suffer from (sub-)exponential dependence on $K$ (e.g.~the Fourier moment method \cite{chen2020learning}, the method of moments \cite{li2018learning}, and grid search over $K$-dimensional subspaces \cite{shen2019iterative}), or only have local convergence guarantees (e.g.~expectation-maximization \cite{kwon2020converges}).} 
It turns out that by restricting
the ground-truth vectors to be in ``general position'' (e.g.~linearly
independent), tensor methods \cite{yi2016solving,chaganty2013spectral,sedghi2016provable,zhong2016mixed}
solve mixed linear regression with polynomial sample and computational
complexities in $K$. It is worth noting that most of the prior work
focused on the Gaussian design for theoretical analysis, with a few
exceptions \cite{chen2014convex,hand2018convex,shen2019iterative}.
Another line of work \cite{khalili2007variable,stadler2010L,yin2018learning,krishnamurthy2019sample,mazumdar2020recovery}
considered mixed linear regression with sparsity, which is beyond
the scope of the current paper. 

\paragraph{Mixed low-rank matrix estimation. }

Moving beyond mixed linear regressions, there are a few papers that
tackle mixtures of low-rank models. For example, \cite{yi2015regularized}
proposed a regularized EM algorithm and applied it to mixed matrix
sensing with two symmetric components; however, only local convergence
was investigated therein. Additionally, \cite{pimentel2018mixture}
was the first to systematically study mixed matrix completion, investigating
the identifiability conditions and sample complexities of this problem;
however, the heuristic algorithm proposed therein comes without provable
guarantees.

\paragraph{Iterative truncated loss minimization. }

\noindent Least trimmed square \cite{rousseeuw1984least} is a classical
method for robust linear regression. Combining the idea of trimming
(i.e.~selecting a subset of ``good'' samples) with iterative optimization
algorithms (e.g.~gradient descent and its variants) leads to a general
paradigm of iterative truncated loss minimization --- a principled
method for improving robustness w.r.t.~heavy-tailed data, adversarial
outliers, etc.~\cite{shen2019learning,shah2020choosing}. Successful
applications of this kind include linear regression \cite{bhatia2015robust},
mixed linear regression \cite{shen2019iterative}, phase retrieval
\cite{chen2017solving,zhang2018median}, matrix sensing \cite{li2020non},
and learning entangled single-sample distributions \cite{yuan2020learning},
among others. 

\paragraph{Multi-task learning and meta-learning. }

The aim of multi-task learning \cite{caruana1997multitask,baxter2000model,ben2003exploiting,ando2005framework,evgeniou2005learning,argyriou2007multi,jalali2010dirty,peng2020cfit,pentina2014pac,maurer2016benefit}
is to simultaneously learn a model that connects multiple \emph{related
}tasks. Exploiting the similarity across tasks enables improved performance
for learning each individual task, and leads to enhance generalization
capabilities for unseen but related tasks with limited samples. This
paradigm (or its variants) is also referred to in the literature as
meta-learning \cite{finn2017model,tripuraneni2020provable} (i.e.~learning-to-learn),
transfer learning \cite{pan2009survey}, and few-shot learning \cite{snell2017prototypical,du2020few},
depending on the specific scenarios of interest. Our study on learning
mixture of models is related to the probabilistic approach taken in
multi-task learning and meta-learning, in which all the tasks (both
the training and the testing ones) are independently sampled from
a common environment, i.e.~a prior distribution of tasks \cite{baxter2000model}.
See \cite{kong2020meta,kong2020robust} for recent efforts that make
explicit the connection between mixed linear regression and meta-learning.

\section{Analysis\label{sec:analysis}}

In this section, we present the proofs of Theorems \ref{thm:noiseless}
and \ref{thm:noisy}. Our analysis is modular in the sense that we
deliver the performance guarantees for the three stages separately
that are independent of each other. For instance, one can replace
the tensor method in Stage 2 by any other mixed linear regression
solver with provable guarantees, without affecting Stages 1 and 3. 

\paragraph{Stage 1. }

The first result confirms that given enough samples, Algorithm \ref{alg:spectral}
outputs reasonable estimates of the subspaces $(\Ustar,\bm{V}^{\star})$
(cf.~(\ref{eq:def_Ustar_Vstar})). The proof is deferred to Appendix
\ref{subsec:proof_stage1}.
\begin{thm}
\label{thm:stage1} Consider the model (\ref{eq:mix_measurements_noisy})
under the assumptions in Section~\ref{subsec:models_assumptions}.
Recall the definitions of $\kappa$ and $\Gamma$ in \eqref{eq:def_params}.
For any $0<\delta<1$, the estimates $\Uhat$ and $\Vhat$ returned
by Algorithm \ref{alg:spectral} satisfy
\begin{equation}
\max\Big\{\distU,\distV\Big\}\lesssim\delta K\sqrt{r}\kappa\left(\Gamma+\frac{1}{\sqrt{Kr}}\frac{\sigma}{\min_{k}\|\Mkstar\|_{\Frm}}\right)\label{eq:disthat}
\end{equation}
with probability at least $1-Ce^{-cn}$ for some universal constants
$C,c>0$, provided that the sample size obeys
\begin{equation}
N\ge C_{0}\frac{nrK}{\delta^{2}}\log\frac{1}{\delta}\label{eq:N_stage1}
\end{equation}
for some sufficiently large constant $C_{0}>0$ . 
\end{thm}

\paragraph{Stage 2. }

Next, we demonstrate that the tensor method employed in Algorithm
\ref{alg:stage2} reliably solves the intermediate mixed linear regression
problem defined in~(\ref{eq:low_dim_mlr}). The proof is postponed
to Appendix \ref{subsec:proof_stage2}.
\begin{thm}
\label{thm:stage2} Consider the model (\ref{eq:mix_measurements_noisy})
under the assumptions in Section~\ref{subsec:models_assumptions}.
Suppose that the subspace estimates $\Uhat$ and $\Vhat$ are independent
of $\{\Ai,\yi\}_{1\le i\le N}$ and obey $\max\{\|\Uhat\Uhat^{\top}-\Ustar\Ustar^{\top}\|,\|\Vhat\Vhat^{\top}-\Vstar\Vstar^{\top}\|\}\le c_{1}/(K\Gamma^{2})$
for some sufficiently small constant $c_{1}>0$. Let $\{\betakhat\}_{1\le k\le K}$
be the estimates returned by Line~\ref{line:alg2_mlr} of Algorithm
\ref{alg:stage2}. Given any $0<\epsilon\le c_{2}/K$, there exists
a permutation $\pi(\cdot):\{1,\dots,K\}\mapsto\{1,\dots,K\}$ such
that
\begin{equation}
\big\|\widehat{\bbeta}_{\pi(k)}-\bbeta_{k}\big\|_{2}\le\epsilon\cdot\max_{1\le j\le K}\big\|\Mjstar\big\|_{{\rm F}}\qquad\text{for all }1\leq k\leq K\label{eq:beta_error}
\end{equation}
with probability at least $1-O(1/\log n)$, provided that the sample
size obeys 
\begin{equation}
N\ge C\frac{K^{8}r^{2}}{\epsilon^{2}}\left(\Gamma^{10}+\frac{\sigma^{10}}{\min_{k}\|\Mkstar\|_{\Frm}^{10}}\right)\log n\cdot\log^{3}N.\label{eq:NMLR}
\end{equation}
Here, $c_{2}>0$ (resp.~$C>0$) is some sufficiently small (resp.~large)
constant.
\end{thm}
 From now on, we shall assume without loss of generality that $\pi(\cdot)$
is an identity map (i.e.~$\pi(k)=k$) to simplify the presentation.
Our next result transfers the estimation error bounds for $\bm{U},\bm{V}$
and $\{\widehat{\bbeta}_{k}\}$ to that for $\{\bm{L}_{k}\bm{R}_{k}^{\top}\}$,
thus concluding the analysis of Stage 2; see Appendix \ref{subsec:proof_prop_stage2}
for a proof.
\begin{prop}
\label{prop:init} The estimates $\{\bL_{k},\bR_{k}\}_{k=1}^{K}$
computed in Lines~\ref{line:alg2_forloop}-\ref{line:alg2_LkRk} of
Algorithm \ref{alg:stage2} obey
\begin{equation}
\big\|\bL_{k}\bR_{k}^{\top}-\Mkstar\big\|_{\Frm}\le2\max\Big\{\distU,\distV\Big\}\big\|\Mkstar\big\|_{\Frm}+2\big\|\betakhat-\bbeta_{k}\big\|_{2}\label{eq:init_error}
\end{equation}
for all $1\leq k\leq K$.
\end{prop}

\paragraph{Stage 3. }

The last result guarantees that Algorithm~\ref{alg:tsgd} --- when
suitably initialized --- converges linearly towards $\bm{M}_{k}^{\star}$
up to a certain error floor. Here $\bm{M}_{k}^{\star}$ is the closest
among $\{\bm{M}_{j}^{\star}\}_{1\leq j\leq K}$ to the point $\bm{L}^{0}(\bm{R}^{0})^{\top}$.
The proof can be found in Appendix \ref{subsec:proof_stage3}.
\begin{thm}
\label{thm:stage3} Consider the model (\ref{eq:mix_measurements_noisy})
under the assumptions in Section~\ref{subsec:models_assumptions}.
Suppose that the noise level obeys (\ref{eq:sigma_bound}). Choose
the step size $\eta$ and truncating fraction $\alpha$ such that
$0<\eta\le1.3/\pkhat$ and $0.6\pkhat\le\alpha\le0.8\pkhat$. Given
any $0<\delta<c_{0}/K$, if $\bL^{0}\in\R^{n_{1}\times r_{k}}$ and
$\bR^{0}\in\R^{n_{2}\times r_{k}}$ obey
\begin{equation}
\left\Vert \bL^{0}(\bR^{0})^{\top}-\Mkstar\right\Vert _{\Frm}\le c_{1}\left\Vert \Mkstar\right\Vert _{\Frm}\cdot\min\left\{ \frac{1}{\sqrt{r}\kappa},\frac{1}{K}\right\} ,\label{eq:local_region}
\end{equation}
then with probability at least $1-Ce^{-cn}$ the iterates of Algorithm
\ref{alg:tsgd} satisfy
\begin{align}
\left\Vert \bL^{t}(\bR^{t})^{\top}-\Mkstar\right\Vert _{\Frm} & \le\left(1-c_{2}\eta\pkhat\right)^{t}\left\Vert \bL^{0}(\bR^{0})^{\top}-\Mkstar\right\Vert _{\Frm}+C_{2}\max\left\{ K\sige\delta,\frac{K\sige^{2}}{\min_{j:j\neq k}\|\Mjstar-\Mkstar\|_{\Frm}}\right\} \label{eq:tsgd_error_floor}
\end{align}
for all $t\ge0$, provided that the sample size exceeds $N\ge C_{0}\frac{nrK}{\delta^{2}}\log N.$
Here, $0<c_{2}<1/4$ and $C,c,C_{2}>0$ are some universal constants,
and $c_{0},c_{1}>0$ (resp.~$C_{0}>0$) are some sufficiently small
(resp.~large) constants. 
\end{thm}

\paragraph{Putting pieces together: proof of Theorems \ref{thm:noiseless} and
\ref{thm:noisy}.}

With the above performance guarantees in place, we are ready to establish
the main theorems. Note that due to sample splitting in Algorithm~\ref{alg:full},
we shall apply Theorems \ref{thm:stage1} and \ref{thm:stage3} to
the dataset $\{\Ai,\yi\}_{1\le i\le N}$, and Theorem \ref{thm:stage2}
to the dataset $\{\Ai',\yi'\}_{1\le i\le\NMLR}$.~Set 
\begin{align*}
\delta & \le c_{3}\frac{1}{K\sqrt{r}\kappa\Gamma}\min\left\{ \frac{1}{\sqrt{r}\kappa},\frac{1}{K\Gamma^{2}}\right\} ,\qquad\text{and}\qquad\epsilon\le c_{4}\frac{1}{\Gamma}\min\left\{ \frac{1}{\sqrt{r}\kappa},\frac{1}{K}\right\} ,
\end{align*}
for some sufficiently small constants $c_{3},c_{4}>0$ in Theorems
\ref{thm:stage1} and \ref{thm:stage2}. These choices --- in conjunction
with our assumption on $\sigma$~in Theorem \ref{thm:noisy}, as
well as Proposition \ref{prop:init} --- guarantee that the initialization
$\bm{L}^{0}(\bm{R}^{0})^{\top}$ lies in the neighborhood of $\bm{M}_{k}^{\star}$
as required by (\ref{eq:local_region}). This allows us to invoke
Theorem~\ref{thm:stage3} to conclude the proof of Theorem \ref{thm:noisy}.
Finally, Theorem \ref{thm:noiseless} follows by simply setting the
noise level $\sigma=0$ in Theorem \ref{thm:noisy}.

\section{Discussion \label{sec:discussion}}

This paper develops a three-stage algorithm for the mixed low-rank
matrix sensing problem, which is provably efficient in terms of both
sample and computational complexities. Having said this, there are
numerous directions that are worthy of further investigations; we
single out a few in the following. 

To begin with, while our required sample complexity scales linearly
(and optimally) w.r.t.~the matrix dimension $\max\{n_{1},n_{2}\}$,
its dependency on other salient parameters --- e.g.~the number $K$
of components, the ranks $\{r_{k}\}$, {and the condition numbers $\{\kappa(\Mkstar)\}$} of the ground-truth matrices
$\{\bm{M}_{k}^{\star}\}$ --- is likely sub-optimal {(for example, in vanilla matrix sensing, the nonconvex method in \cite{tong2020accelerating} need only $\tilde{O}(n r^2 \kappa^2)$ samples)}. Improving the
sample efficiency in these aspects is certainly an interesting direction
to explore. In addition, in the presence of \emph{random} noise, the
performance of ScaledTGD saturates after the number of samples exceeds
a certain threshold. It would be helpful to investigate other algorithms
like expectation-maximization to see whether there is any performance
gain one can harvest. Furthermore, our current theory builds upon
the Gaussian designs $\{\bm{A}_{i}\}$, which often does not capture
the practical scenarios. It is of great practical importance to develop
efficient algorithms that can accommodate a wider range of design
matrices $\{\bm{A}_{i}\}$ --- for instance, the case of mixed low-rank
matrix completion. Last but not least, it would be of interest to
study more general meta-learning settings in the presence of both
light and heavy tasks (beyond the current single-sample setting) \cite{kong2020meta},
and see how sample complexities can be reduced (compared to meta-learning
for mixed regression) by exploiting such low-complexity structural
priors .

\section*{Acknowledgements}

Y.~Chen is supported in part by the grants AFOSR YIP award FA9550-19-1-0030,
ONR N00014-19-1-2120, ARO YIP award W911NF-20-1-0097, ARO W911NF-18-1-0303,
NSF CCF-1907661, IIS-1900140 and DMS-2014279, and the Princeton SEAS
Innovation Award. H.~V.~Poor is supported in part by NSF CCF-1908308,
and in part by a Princeton Schmidt Data-X Research Award. We would
like to thank Qixing Huang who taught us the symmetry synchronoziation
problem in computer vision that largely inspired this research.

\appendix

\section{The tensor method for mixed linear regression \label{sec:alg_MLR}}

This section reviews the tensor method proposed in \cite{yi2016solving}
for solving mixed linear regression. For simplicity of exposition,
we consider the noiseless case where we have access to the samples
$\{\ai,\yi\}_{1\le i\le N}$ obeying
\begin{equation}
y_{i}=\begin{cases}
\langle\ai,\betaonestar\rangle, & \text{if }i\in\Omegaonestar,\\
\dots & \dots\\
\langle\ai,\betaKstar\rangle, & \text{if }i\in\OmegaKstar.
\end{cases}\label{eq:mlr_meas}
\end{equation}
Our goal is to recover the ground truths $\betakstar\in\R^{d},1\le k\le K$,
without knowing the index sets $\{\Omegakstar\}$.

\paragraph{Notation for tensors. }

For two matrices $\bA$ and $\bB$, denote by $\boldsymbol{A}\otimes\boldsymbol{B}$
their Kronecker product, and let $\boldsymbol{A}^{\otimes3}$ represent
$\boldsymbol{A}\otimes\boldsymbol{A}\otimes\boldsymbol{A}$. For a
symmetric tensor $\bT\in\R^{d\times d\times d}$ and matrices $\bA\in\R^{d\times d_{1}},\bB\in\R^{d\times d_{2}},\bC\in\R^{d\times d_{3}}$,
let $\bT(\bA,\bB,\bC)\in\R^{d_{1}\times d_{2}\times d_{3}}$ denote
the multi-linear matrix multiplication such that
\[
\big[\bT(\bA,\bB,\bC)\big]_{m,n,p}=\sum_{1\le i,j,k\le d}T_{i,j,k}A_{i,m}B_{j,n}C_{k,p},\quad1\le m\le d_{1},\,1\le n\le d_{2},\,1\le p\le d_{3}.
\]
In addition, let $\|\bT\|$ stand for the operator norm of $\bT$,
namely, $\|\bT\|\coloneqq\sup_{\bx:\|\bx\|_{2}=1}\big|\bT(\bx,\bx,\bx)\big|$.

\begin{algorithm}[tbp]
\DontPrintSemicolon
\caption{The tensor method for mixed linear regression \cite[Algorithm~1]{yi2016solving}} \label{alg:tensor}
{\bf Input:} {$\{\ba_i, y_i\}_{1 \le i \le N}$}. \\
Randomly split the samples into two disjoint sets $\{\ai,\yi\}_{1 \le i \le N_1},\{\ai',\yi'\}_{1 \le i \le N_2}$ such that $N=N_1+N_2$, by assigning each sample to either dataset with probability $0.5$. \\
Compute $\mzero\gets\frac{1}{N_1}\sum_{i=1}^{N_1}y_{i}^{2},\boldsymbol{m}_{1}\gets\frac{1}{6N_2}\sum_{i=1}^{N_2}{y_{i}'}^{3}\ai'$. \\
Compute $\Mtwo\gets\frac{1}{2N_1}\sum_{i=1}^{N_1}y_{i}^{2}\ai\ai^{\top}-\frac{1}{2}\mzero\bI_{d},\Mthree\gets\frac{1}{6N_2}\sum_{i=1}^{N_2}{y_{i}'}^{3}{\ai'}^{\otimes3}-\Tcal(\mone)$, where $\Tcal$ is defined in~(\ref{eq:Tcal}). \\
Denote the rank-$K$ SVD of $\bM_2$ as $\bU_{2}\boldsymbol{\Sigma}_{2}\bV_{2}^{\top}$, and compute the whitening matrix $\bW\gets\bU_{2}\boldsymbol{\Sigma}_{2}^{-1/2}$. \\
Compute $\widetilde{\bM}_{3}\gets\bM_{3}(\bW,\bW,\bW)$. \\
Run the robust tensor power method \cite[Algorithm 2]{yi2016solving} on $\widetilde{\bM}_{3}$ to obtain $K$ eigenvalue/eigenvector pairs $\{\widetilde{\omega}_{k},\widetilde{\bbeta}_{k}\}_{1 \le k \le K}$. \\
Compute $\omega_{k}\gets1/\widetilde{\omega}_{k}^{2},\bbeta_{k}\gets\widetilde{\omega}_{k}\bW(\bW^{\top}\bW)^{-1}\widetilde{\bbeta}_{k}, 1\le k\le K$. \\
{\bf Output:} {$\{\omega_k,\bbeta_{k}\}_{1 \le k \le K}$}.
\end{algorithm}

\paragraph{The tensor method: algorithm and rationale. }

We summarize the tensor method in Algorithm \ref{alg:tensor}, which
is mostly the same as \cite[Algorithm 1]{yi2016solving} and included
here for completeness. 

In the following, we explain the intuitions behind its algorithmic
design. Given data $\{\ai,\yi\}_{1\le i\le N}$ generated according
to (\ref{eq:mlr_meas}), we compute the following empirical moments:
\begin{align*}
\mzero\coloneqq\frac{1}{N}\sum_{i=1}^{N}y_{i}^{2}\in\R, & \quad\boldsymbol{m}_{1}\coloneqq\frac{1}{6N}\sum_{i=1}^{N}y_{i}^{3}\ai\in\R^{d},\\
\Mtwo\coloneqq\frac{1}{2N}\sum_{i=1}^{N}y_{i}^{2}\ai\ai^{\top}-\frac{1}{2}\mzero\bI_{d}\in\R^{d\times d}, & \quad\Mthree\coloneqq\frac{1}{6N}\sum_{i=1}^{N}y_{i}^{3}\ai^{\otimes3}-\Tcal(\mone)\in\R^{d\times d\times d};
\end{align*}
here, letting $\{\be_{i}\}_{1\le i\le d}$ be the canonical basis
of $\R^{d}$, we define the operator $\Tcal(\cdot):\mathbb{R}^{d}\mapsto\mathbb{R}^{d\times d\times d}$
as
\begin{equation}
\Tcal(\boldsymbol{m})\coloneqq\sum_{i=1}^{d}\left(\boldsymbol{m}\otimes\be_{i}\otimes\be_{i}+\be_{i}\otimes\boldsymbol{m}\otimes\be_{i}+\be_{i}\otimes\be_{i}\otimes\boldsymbol{m}\right),\quad\text{where }\boldsymbol{m}\in\R^{d}.\label{eq:Tcal}
\end{equation}
The key observation is that: under the Gaussian design (i.e.~$\ai\iid\Ncal(\boldsymbol{0},\bI_{d})$),
$\bm{M}_{2}$ and $\bm{M}_{3}$ reveal crucial second-order and third-order
moments of $\{\bm{\beta}_{k}^{\star}\}$ since (cf.~\cite[Lemma 1]{yi2016solving})
\[
\E[\bM_{2}]=\sumK\pkhat\betakstar(\betakstar)^{\top}\quad\text{and}\quad\E[\bM_{3}]=\sumK\pkhat(\betakstar)^{\otimes3},
\]
where we recall $\pkhat=|\Omegakstar|/N$. This motivates one to apply
tensor decomposition \cite{anandkumar2014tensor} on $\bM_{2}$ and
$\bM_{3}$ in order to estimate $\{\betakstar\}$ and $\{\pkhat\}$.
Indeed, the estimates $\{\bbeta_{k}\}$ and $\{\omega_{k}\}$ returned
by Algorithm~\ref{alg:tensor} serve as our estimates of $\{\betakstar\}$
and $\{\pkhat\}$, respectively. 
\begin{rem}
[Sample splitting] Similar to \cite{yi2016solving}, we assume that
$m_{0}$ and $\bM_{2}$ are computed using one set of data, while
$\bM_{1}$ and $\boldsymbol{M}_{3}$ are obtained based on another
\emph{independent} set of samples. This sample splitting strategy
ensures that the whitening matrix $\bW$ is independent of $\bM_{3}$,
thus simplifying theoretical analysis. 
\end{rem}

\section{\label{sec:proof_for_analysis}Proofs for Section \ref{sec:analysis}}

For notational simplicity, we use $\disthat$ throughout to denote
the following subspace estimation error:
\begin{equation}
\disthat\coloneqq\max\Big\{\distU,\distV\Big\}.\label{eq:def_disthat}
\end{equation}

\subsection{\label{subsec:proof_stage1}Proof of Theorem \ref{thm:stage1}}

The proof is decomposed into two steps: we first develop an upper
bound $\|\bY-\E[\bY]\|$ (where $\bY$ is as defined in Algorithm~\ref{alg:spectral}),
and then combine this with Wedin's Theorem to control the subspace
distance $\disthat$. 

\paragraph{Step 1: controlling $\|\bY-\protect\E[\bY]\|$. }

We start by decomposing $\bY$ into $\bY=\bY_{A}+\bY_{\zeta}$, where
we define
\[
\bY_{A}\coloneqq\sumK\frac{\pkhat}{|\Omegakstar|}\sumomgk\langle\Ai,\Mkstar\rangle\Ai\qquad\text{and}\qquad\bY_{\zeta}\coloneqq\frac{1}{\Nall}\sumall\ei\Ai.
\]
Lemma \ref{lem:TRIP} asserts that: with probability at least $1-Ce^{-cn}$
for some universal constants $C,c>0$, we have 
\[
\bigg\|\frac{1}{|\Omegakstar|}\sumomgk\langle\Ai,\Mkstar\rangle\Ai-\Mkstar\bigg\|\le\delta\|\Mkstar\|_{{\rm F}},\quad1\le k\le K,
\]
as long as the sample size $N$ satisfies (\ref{eq:N_stage1}), which
together with the triangle inequality further implies 
\[
\|\bY_{A}-\E[\bY_{A}]\|\leq\sum_{k=1}^{K}p_{k}\bigg\|\frac{1}{|\Omegakstar|}\sumomgk\langle\Ai,\Mkstar\rangle\Ai-\Mkstar\bigg\|\leq\delta\sumK\pkhat\|\Mkstar\|_{{\rm F}}\le\delta\max_{1\le k\le K}\|\Mkstar\|_{\Frm}.
\]
In addition, \cite[Lemma 1.1]{candes2011tight} reveals that with
probability at least $1-Ce^{-cn}$ for some constants $C,c>0$, 
\[
\|\bY_{\zeta}\|\lesssim\sige\sqrt{\frac{n}{N}}\lesssim\sigma\frac{\delta}{\sqrt{Kr}}
\]
holds under the sample size condition (\ref{eq:N_stage1}). Given
that $\E[\bY_{A}]=\E[\bY]=\sum_{k}\pkhat\Mkstar$, we have established
the existence of some universal constant $C_{1}>0$ such that
\begin{equation}
\|\bY-\E[\bY]\|\le\|\bY_{A}-\E[\bY_{A}]\|+\|\bY_{\zeta}\|\le C_{1}\delta\left(\max_{1\le k\le K}\|\Mkstar\|_{\Frm}+\frac{\sigma}{\sqrt{Kr}}\right)\eqqcolon\Delta.\label{eq:Y_perturbation}
\end{equation}

\paragraph{Step 2: controlling $\protect\disthat$. }

Before embarking on controlling $\disthat$, we make the following
claim. 

\begin{itclaim} \label{claim:eigengap}

Under the assumptions of Theorem \ref{thm:stage1}, we have\begin{subequations}
\begin{equation}
\col\left\{ \sumK\pkhat\Mkstar\right\} =\col\Big\{\left[\Uonestar,\dots,\UKstar\right]\Big\},\quad\row\left\{ \sumK\pkhat\Mkstar\right\} =\col\Big\{\left[\Vonestar,\dots,\VKstar\right]\Big\},\quad R=\sumK r_{k},\label{eq:full_col_rank}
\end{equation}
\begin{equation}
\text{and}\qquad\sigma_{R}\left(\sumK\pkhat\Mkstar\right)\gtrsim\frac{1}{K}\min_{k}\sigma_{r_{k}}(\Mkstar).\label{eq:eigengap}
\end{equation}
\end{subequations}

\end{itclaim}

With this claim in place, we are ready to apply Wedin's Theorem \cite{wedin1972perturbation}
to obtain 
\begin{align}
\disthat & \le\frac{\|\bY-\E[\bY]\|}{\sigma_{R}(\E[\bY])-\|\bY-\E[\bY]\|}\le\frac{\Delta}{\sigma_{R}(\E[\bY])-\Delta}\le\frac{2\Delta}{\sigma_{R}(\E[\bY])}=2C_{1}\frac{\delta\left(\max_{k}\|\Mkstar\|_{\Frm}+\frac{\sigma}{\sqrt{Kr}}\right)}{\sigma_{R}\left(\sum_{k}\pkhat\Mkstar\right)},\label{eq:tmp_disthat}
\end{align}
with the proviso that $\Delta$ defined in (\ref{eq:Y_perturbation})
obeys $\Delta\leq\frac{1}{2}\sigma_{R}(\E[\bY])$. On the other hand,
if instead one has $\Delta>\frac{1}{2}\sigma_{R}(\E[\bY])$, then
we claim that (\ref{eq:tmp_disthat}) trivially holds; this can be
seen by observing that $\disthat\leq1$, while the right-hand side
of (\ref{eq:tmp_disthat}) is greater than $1$ if $\Delta>\frac{1}{2}\sigma_{R}(\E[\bY])$.
Finally, Claim~\ref{claim:eigengap} tells us that 
\[
\sigma_{R}\left(\sum_{k}\pkhat\Mkstar\right)\gtrsim\frac{1}{K}\min_{k}\sigma_{r_{k}}(\Mkstar)\gtrsim\frac{1}{K\sqrt{r}\kappa}\min_{k}\|\Mkstar\|_{\Frm}.
\]
Substituting this relation into (\ref{eq:tmp_disthat}) immediately
leads to the advertised bound (\ref{eq:disthat}) in Theorem~\ref{thm:stage1}.
\begin{proof}
[Proof of Claim \ref{claim:eigengap}] Recall that we can write $\sum_{k}\pkhat\Mkstar$
in terms of $\{\Ukstar,\Sigkstar,\Vkstar\}$, in the form of (\ref{eq:EY}).
Therefore, to prove (\ref{eq:full_col_rank}), it suffices to show
that $\min\{\sigma_{R'}([\Uonestar,\dots,\UKstar]),\sigma_{R'}([\Vonestar,\dots,\VKstar])\}\ge1/\sqrt{2}$,
where $R'\coloneqq\sum_{k}r_{k}$. We only prove this for $\sigma_{R'}([\Uonestar,\dots,\UKstar])$,
since the proof for $\sigma_{R'}([\Vonestar,\dots,\VKstar])$ is identical.
Denoting $\bW\coloneqq[\Uonestar,\dots,\UKstar]$ for notational convenience,
we have
\[
\bW^{\top}\bW=\begin{bmatrix}\Uonestar^{\top}\\
\vdots\\
\UKstar^{\top}
\end{bmatrix}\begin{bmatrix}\Uonestar & \dots & \UKstar\end{bmatrix}=\begin{bmatrix}\bI_{r_{1}} & \Uonestar^{\top}\Utwostar & \dots & \Uonestar^{\top}\UKstar\\
\Utwostar^{\top}\Uonestar & \bI_{r_{2}} & \ddots & \vdots\\
\vdots & \ddots & \ddots & \vdots\\
\UKstar^{\top}\Uonestar & \dots & \dots & \bI_{r_{K}}
\end{bmatrix}.
\]
This together with Assumption \ref{asp:weak_correlation} gives 
\[
\|\bW^{\top}\bW-\bI_{R'}\|_{\Frm}^{2}=\sum_{i\neq j}\|{\Uistar}^{\top}\Ujstar\|_{\Frm}^{2}\le K^{2}\left(\frac{1}{2K}\right)^{2}\le\frac{1}{4}.
\]
Apply Weyl's inequality to obtain
\[
\sigma_{R'}(\bW^{\top}\bW)\ge1-\|\bW^{\top}\bW-\bI_{R'}\|\ge1-\|\bW^{\top}\bW-\bI_{R'}\|_{\Frm}\ge\frac{1}{2},
\]
thus indicating that $\sigma_{R'}(\bW)=\sqrt{\sigma_{R'}(\bW^{\top}\bW)}\ge1/\sqrt{2}$.
This completes the proof of (\ref{eq:full_col_rank}).

Next, we turn attention to (\ref{eq:eigengap}). Denote the SVD of
$[\Uonestar,\dots,\UKstar]$ (resp.~$[\Vonestar,\dots,\VKstar]$)
as $\bU_{\mathsf{left}}\boldsymbol{\Sigma}_{\mathsf{left}}\bV_{\mathsf{left}}^{\top}$
(resp.~$\bU_{\mathsf{right}}\boldsymbol{\Sigma}_{\mathsf{right}}\bV_{\mathsf{right}}^{\top}$),
where $\bV_{\mathsf{left}}$ (resp.~$\bV_{\mathsf{right}}$) is a
$R\times R$ orthonormal matrix. Substitution into (\ref{eq:EY})
yields
\[
\sum_{k=1}^{K}\pkhat\Mkstar=\bU_{\mathsf{left}}\boldsymbol{\Sigma}_{\mathsf{left}}\bV_{\mathsf{left}}^{\top}\mathsf{diag}\big(\{\pkhat\Sigkstar\}_{1\leq k\leq K}\big)\bV_{\mathsf{right}}\boldsymbol{\Sigma}_{\mathsf{right}}\bU_{\mathsf{right}}^{\top},
\]
where $\mathsf{diag}\big(\{\pkhat\Sigkstar\}_{1\leq k\leq K}\big)$
is a $R\times R$ full-rank diagonal matrix, with blocks $\ponehat\Sigonestar,\dots,\pKhat\SigKstar$
on the diagonal. This implies that
\begin{align*}
\sigma_{R}\left(\sumK\pkhat\Mkstar\right) & =\sigma_{R}\Big(\boldsymbol{\Sigma}_{\mathsf{left}}\bV_{\mathsf{left}}^{\top}\mathsf{diag}\big(\{\pkhat\Sigkstar\}_{1\leq k\leq K}\big)\bV_{\mathsf{right}}\boldsymbol{\Sigma}_{\mathsf{right}}\Big)\ge\sigma_{R}(\boldsymbol{\Sigma}_{\mathsf{left}})\sigma_{R}(\boldsymbol{\Sigma}_{\mathsf{right}})\cdot\min_{k}\big\{\pkhat\sigma_{r_{k}}(\Mkstar)\big\}\\
 & \ge\left(\frac{1}{\sqrt{2}}\right)^{2}\min_{k}\big\{\pkhat\sigma_{r_{k}}(\Mkstar)\big\}\gtrsim\frac{1}{K}\min_{k}\sigma_{r_{k}}(\Mkstar),
\end{align*}
where the last inequality uses the assumption that $\pkhat\gtrsim1/K$.
This establishes (\ref{eq:eigengap}).
\end{proof}

\subsection{\label{subsec:proof_stage2}Proof of Theorem \ref{thm:stage2}}

\paragraph{Step 1: basic properties of the auxiliary mixed linear regression
problem.}

We begin by formally characterizing the intermediate mixed linear
regression problem in Stage 2. It is easily seen from Section~\ref{subsec:alg_stage_2}
that for $i\in\Omega_{k}^{\star}$, one has 
\begin{equation}
\yi=\langle\Ai,\Mkstar\rangle+\ei=\langle\ai,\bbeta_{k}\rangle+\underbrace{z_{i}+\ei}_{\eqqcolon\xi_{i}},\label{eq:yi-MLR-proof-1}
\end{equation}
where the additional term
\begin{equation}
z_{i}\coloneqq\langle\Ai,\Mkstar\rangle-\langle\ai,\bbeta_{k}\rangle=\langle\Ai,\Mkstar-\Uhat\Uhat^{\top}\Mkstar\Vhat\Vhat^{\top}\rangle\label{eq:zi-MLR-proof-1}
\end{equation}
accounts for the subspace estimation error. In words, the observations
$\{y_{i}\}$ can be equivalently written in the mixed linear regression
form, where $\{\bm{\beta}_{k}\}$ constitutes the underlying parameters,
$\{\bm{a}_{i}\}$ the measurement vectors and $\{\xi_{i}\}$ the measurement
noise. We then focus on characterizing the properties of $\bm{a}_{i}$
and $\xi_{i}$. 

Recall from Algorithm~\ref{alg:stage2} that $\bm{a}_{i}=\mathsf{vec}(\bm{U}^{\top}\bm{A}_{i}\bm{V})$.
In view of the independence between $\{\Ai\}$ and $\Uhat,\Vhat$,
one can deduce that 
\[
\ai\overset{{\rm i.i.d.}}{\sim}\Ncal(\boldsymbol{0},\boldsymbol{I}_{d}),\quad1\le i\le N,
\]
where $d\coloneqq R^{2}$. Again, leveraging the independence between
$\{\Ai,\ei\}$ and $\Uhat,\Vhat$, we have 
\begin{align*}
\xi_{i} & =\langle\Ai,\Mkstar-\Uhat\Uhat^{\top}\Mkstar\Vhat\Vhat^{\top}\rangle+\zeta_{i}\overset{\text{i.i.d.}}{\sim}\Ncal\big(0,\|\Mkstar-\Uhat\Uhat^{\top}\Mkstar\Vhat\Vhat^{\top}\|_{\Frm}^{2}+\sigma^{2}\big).
\end{align*}
For notational convenience, we shall denote the variance to be
\begin{equation}
\sigma_{k}^{2}\coloneqq\|\Mkstar-\Uhat\Uhat^{\top}\Mkstar\Vhat\Vhat^{\top}\|_{\Frm}^{2}+\sigma^{2},\quad1\le k\le K.\label{eq:def_sigk}
\end{equation}
More importantly, the measurement vectors $\{\bm{a}_{i}\}$ are independent
of the measurement noise $\{\xi_{i}\}$. To see this, one has 
\begin{align*}
\E[\xi_{i}\ai] & =\E[\zeta_{i}\ai]+\E[z_{i}\ai]=\bm{0}+\vc\left(\E[\langle\Ai,\Mkstar-\Uhat\Uhat^{\top}\Mkstar\Vhat\Vhat^{\top}\rangle\Uhat^{\top}\Ai\Vhat]\right)\\
 & =\vc\left(\Uhat^{\top}\left(\Mkstar-\Uhat\Uhat^{\top}\Mkstar\Vhat\Vhat^{\top}\right)\Vhat\right)=\boldsymbol{0}.
\end{align*}
Here the second equality follows from the independence between $\zeta_{i}$
and $\bm{A}_{i},\bm{U},\bm{V}$, whereas the last line utilizes the
independence between $\bm{A}_{i}$ and $\bm{U},\bm{V}$ and the isotropic
property of $\bm{A}_{i}$. 

In conclusion, in Line~\ref{line:alg2_mlr} of Algorithm \ref{alg:stage2},
we are equivalently faced with a $d$-dimensional mixed linear regression
problem with data $\{\ai,\yi\}_{1\le i\le N}$, which satisfies that
for $i\in\Omega_{k}^{\star}$, 
\begin{equation}
\yi=\big\langle\ai,\bbeta_{k}\big\rangle+\xi_{i},\qquad\xi_{i}\overset{{\rm i.i.d.}}{\sim}\Ncal\left(0,\sigma_{k}^{2}\right),\,\bm{a}_{i}\overset{{\rm i.i.d.}}{\sim}\Ncal(\boldsymbol{0},\boldsymbol{I}_{d})\label{eq:mlr_characterization}
\end{equation}
with $\xi_{i}$ being independent from $\bm{a}_{i}$. 

\paragraph{Step 2: performance of the tensor method. }

Next, we characterize the performance of the tensor method for solving
the above mixed linear regression problem. Our proof follows closely
that of \cite[Theorem 1]{yi2016solving}, with minor modifications
to accommodate the noise $\{\xi_{i}\}$. Therefore we only provide
a sketch here. 

Recall that in Algorithm \ref{alg:tensor}, we randomly split the
input data $\{\ai,\yi\}_{1\le i\le N}$ into two sets $\{\ai,\yi\}_{1\le i\le N_{1}}$
and $\{\ai',\yi'\}_{1\le i\le N_{2}}$ (with slight abuse of notation).
This sample splitting strategy is adopted merely to decouple statistical
dependence and facilitate analysis. The high-level idea of the proof
of \cite[Theorem~1]{yi2016solving} is simple to state: if the quantities
\begin{equation}
\Big\|\Mtwo-\sum_{k=1}^{K}\pkhat{\bbeta_{k}}\bbeta_{k}^{\top}\Big\|\quad\text{and}\quad\Big\|\Big(\Mthree-\sum_{k=1}^{K}\pkhat{\bbeta_{k}}^{\otimes3}\Big)(\bW,\bW,\bW)\Big\|\label{eq:moment_concentration}
\end{equation}
are sufficiently small, then the tensor method returns reliable estimates
of $\{\bbeta_{k}\}$; see \cite[Eq.~(24) in Section 5.4.1]{yi2016solving}.
Here, the empirical moments $\bM_{2},\bM_{3}$ and the whitening matrix
$\bW$ are defined in Algorithm~\ref{alg:tensor}. 

With this connection in place, it suffices to control the quantities
in (\ref{eq:moment_concentration}). While the analysis in~\cite[Section 5.4.2]{yi2016solving}
only applies to the noiseless mixed linear regression problem, we
can easily modify it to accommodate our noisy case~(\ref{eq:mlr_characterization}).
The trick is to \emph{augment} $\{\bbeta_{k}\}$ and $\{\ai\}$ as
follows:
\begin{equation}
\betaktil\coloneqq\begin{bmatrix}\bbeta_{k}\\
\sigma_{k}
\end{bmatrix}\in\R^{d+1},\quad1\le k\le K;\quad\aitil\coloneqq\begin{bmatrix}\ai\\
\xi_{i}/\sigma_{k}
\end{bmatrix}\in\R^{d+1},\quad i\in\Omegakstar.\label{eq:def_betatilde_atilde}
\end{equation}
The advantage is clear: the noisy mixed linear regression problem~(\ref{eq:mlr_characterization})
can be equivalently phrased as a noiseless one, that is for all $i\in\Omega_{k}^{\star}$,
\begin{equation}
\aitil\overset{{\rm i.i.d.}}{\sim}\Ncal(\boldsymbol{0},\bI_{d+1})\quad\quad\text{and}\quad\quad\yi=\big\langle\aitil,\betaktil\big\rangle.\label{eq:equivalent_noiseless_mlr}
\end{equation}
Similarly, we can define $\aitil'$ analogously, and introduce the
augmented versions of the empirical moments as follows:
\begin{subequations}
\label{eq:def_M_tilde}
\begin{equation}
m_{0}^{\mathsf{aug}}\coloneqq\frac{1}{N_{1}}\sum_{i=1}^{N_{1}}y_{i}^{2}\in\R,\quad\boldsymbol{m}_{1}^{\mathsf{aug}}\coloneqq\frac{1}{6N_{2}}\sum_{i=1}^{N_{2}}{\yi'}^{3}\aitil'\ \in\R^{d+1},\label{eq:m0_m1_tilde}
\end{equation}
\begin{equation}
\bM_{2}^{\mathsf{aug}}\coloneqq\frac{1}{2N_{1}}\sum_{i=1}^{N_{1}}y_{i}^{2}\aitil(\aitil)^{\top}-\frac{1}{2}m_{0}^{\mathsf{aug}}\,\bI_{d+1}\ \in\R^{(d+1)\times(d+1)},\label{eq:m2_tilde}
\end{equation}
\begin{equation}
\bM_{3}^{\mathsf{aug}}\coloneqq\frac{1}{6N_{2}}\sum_{i=1}^{N_{2}}{\yi'}^{3}(\aitil')^{\otimes3}-\Tcal^{\mathsf{aug}}(\boldsymbol{m}_{1}^{\mathsf{aug}})\ \in\R^{(d+1)\times(d+1)\times(d+1)},\label{eq:m3_tilde}
\end{equation}
\end{subequations}
 where $\Tcal^{\mathsf{aug}}(\cdot)$ is defined analogously as in
(\ref{eq:Tcal}).  By virtue of the augmentation procedure, $\bm{M}_{2}$
(resp.~$\bm{M}_{3}$) is a a sub-matrix (resp.~sub-tensor) of $\bM_{2}^{\mathsf{aug}}$
(resp.~$\bM_{3}^{\mathsf{aug}}$). Consequently, we have 
\begin{align*}
\Big\|\Mtwo-\sum_{k=1}^{K}\pkhat{\bbeta_{k}}\bbeta_{k}^{\top}\Big\| & \leq\left\Vert \bM_{2}^{\mathsf{aug}}-\sumK\pkhat\betaktil(\betaktil)^{\top}\right\Vert ;\\
\Big\|\Big(\Mthree-\sum_{k=1}^{K}\pkhat{\bbeta_{k}}^{\otimes3}\Big)(\bW,\bW,\bW)\Big\| & =\left\Vert \Big(\bM_{3}^{\mathsf{aug}}-\sumK\pkhat(\betaktil)^{\otimes3}\Big)(\bm{W}^{\mathsf{aug}},\bm{W}^{\mathsf{aug}},\bm{W}^{\mathsf{aug}})\right\Vert ,
\end{align*}
where $\bm{W}^{\mathsf{aug}}\coloneqq[\bW^{\top},\boldsymbol{0}]^{\top}$. 

With the above augmented vectors/matrices/tensors in place, one can
follow the analysis in \cite[Section 5.4.2]{yi2016solving} to upper
bound the quantities above. One subtle issue is that our sampling
scheme is slightly different from the one in \cite{yi2016solving},
where each sample has i.i.d.~labeling; nevertheless, it is easy to
check that this difference is minor, and does not affect the result
of the analysis. Indeed, repeating the analysis in \cite[Section 5.4]{yi2016solving}
yields the conclusion that: in order to achieve $\epsilon$ errors
(\ref{eq:beta_error}) with probability at least $1-\gamma$, it suffices
to require the sample complexities to exceed (analogous to \cite[Eq.~(13)]{yi2016solving})
\begin{subequations}
\label{eq:N_tmp}
\begin{align}
N_{1} & \ge C_{1}\left(\frac{d}{(\min_{k}\pkhat)\,\epsilon^{2}}\frac{\max_{k}\big\|\betaktil\big\|_{2}^{10}}{\sigma_{K}(\sum_{k}\pkhat\bbeta_{k}\bbeta_{k}^{\top})^{5}}\log\frac{12K}{\gamma}\log^{2}N_{1}+\frac{K}{(\min_{k}\pkhat)\,\gamma}\right)\label{eq:N1_tmp-1}\\
 & \overset{{\rm (i)}}{\asymp}\frac{dK^{6}}{\epsilon^{2}}\left(\Gamma^{10}+\frac{\sigma^{10}}{\min_{k}\|\Mkstar\|_{\Frm}^{10}}\right)\log(K\log n)\log^{2}N_{1}+K^{2}\log n,\label{eq:N1_tmp-2}\\
N_{2} & \ge C_{2}\left(\frac{(K^{2}+d)}{(\min_{k}\pkhat)\,\epsilon^{2}}\frac{\max_{k}\big\|\betaktil\big\|_{2}^{6}}{\sigma_{K}(\sum_{k}\pkhat\bbeta_{k}\bbeta_{k}^{\top})^{3}}\log\frac{12K}{\gamma}\log^{3}N_{2}+\frac{K}{(\min_{k}\pkhat)\,\gamma}\right)\label{eq:N2_tmp-1}\\
 & \overset{{\rm (ii)}}{\asymp}\frac{dK^{4}}{\epsilon^{2}}\left(\Gamma^{6}+\frac{\sigma^{6}}{\min_{k}\|\Mkstar\|_{\Frm}^{6}}\right)\log(K\log n)\log^{3}N_{2}+K^{2}\log n.\label{eq:N2_tmp-2}
\end{align}
\end{subequations}
Here $C_{1},C_{2}>0$ are some sufficiently large constants, and
the simplifications (i) (ii) hold due to the following facts: (i)
$d=R^{2}\ge K^{2}$, (ii) $\min_{k}\pkhat\asymp1/K$, (iii) we choose
$\gamma=O(1/\log n)$, (iv) $\|\bbeta_{k}^{\mathsf{aug}}\|_{2}^{2}=\|\bbeta_{k}\|_{2}^{2}+\sigma_{k}^{2}=\|\bbeta_{k}\|_{2}^{2}+\|\Mkstar-\Uhat\Uhat^{\top}\Mkstar\Vhat\Vhat^{\top}\|_{\Frm}^{2}+\sigma^{2}$,
and (v) the following claim (in particular, (\ref{eq:claim_beta})
and (\ref{eq:sigK}) therein). 

\begin{itclaim} \label{claim:conseq_of_weak_correlation}

Instate the assumptions of Theorem \ref{thm:stage2}. 
\begin{enumerate}
\item The ground-truth matrices $\{\Mkstar\}_{1\le k\le K}$ satisfy that
for all $1\le i,j\le K,i\neq j$,
\begin{equation}
\big|\langle\Mistar,\Mjstar\rangle\big|\le\frac{1}{4K\Gamma^{2}}\|\Mistar\|_{\Frm}\|\Mjstar\|_{\Frm},\quad\text{and}\quad\big\|\Mistar-\Mjstar\big\|_{\Frm}\gtrsim\|\Mistar\|_{\Frm}+\|\Mjstar\|_{\Frm}.\label{eq:claim_Mstar}
\end{equation}
\item In addition, the parameters $\{\bbeta_{k}\}_{1\le k\le K}$ obey that
for all $1\le k,i,j\le K,i\neq j$,
\begin{equation}
0.9\|\Mkstar\|_{\Frm}\le\|\bbeta_{k}\|_{2}\le\|\Mkstar\|_{\Frm},\quad\text{and}\quad\big|\big\langle\bbeta_{i},\bbeta_{j}\big\rangle\big|\le\frac{1}{2K\Gamma^{2}}\big\|\bbeta_{i}\big\|_{2}\big\|\bbeta_{j}\big\|_{2}.\label{eq:claim_beta}
\end{equation}
\item In the end, we have
\begin{equation}
\sigma_{K}\left(\sumK\pkhat\bbeta_{k}\bbeta_{k}^{\top}\right)\asymp\frac{1}{K}\min_{1\le k\le K}\|\Mkstar\|_{\Frm}^{2}.\label{eq:sigK}
\end{equation}
\end{enumerate}
\end{itclaim}

Armed with (\ref{eq:N1_tmp-2}) and (\ref{eq:N2_tmp-2}), we can plug
in the bounds $d=R^{2}\le K^{2}r^{2}$ and $\log(K\log n)\lesssim\log n$
to complete the proof of Theorem~\ref{thm:stage2}. 
\begin{proof}
[Proof of Claim~\ref{claim:conseq_of_weak_correlation}] With regards
to the first part of (\ref{eq:claim_Mstar}), it is seen that 
\begin{align*}
\left|\langle\Mistar,\Mjstar\rangle\right| & =\big|\langle\Uistar\Sigistar{\Vistar}^{\top},\Ujstar\Sigjstar{\Vjstar}^{\top}\rangle\big|=\big|\langle\Sigistar,\bm{U}_{i}^{\star\top}\bm{U}_{j}^{\star}\Sigjstar\bm{V}_{j}^{\star\top}\bm{V}_{i}^{\star}\rangle\big|\le\|{\Uistar}^{\top}\Ujstar\|_{\Frm}\|{\Vistar}^{\top}\Vjstar\|_{\Frm}\|\Sigistar\|_{\Frm}\|\Sigjstar\|_{\Frm}\\
 & \le\left(\frac{1}{2\sqrt{K}\Gamma}\right)^{2}\|\Sigistar\|_{\Frm}\|\Sigjstar\|_{\Frm}=\frac{1}{4K\Gamma^{2}}\|\Mistar\|_{\Frm}\|\Mjstar\|_{\Frm},
\end{align*}
where the second line utilizes Assumption \ref{asp:weak_correlation}.
The second part of (\ref{eq:claim_Mstar}) follows immediately from
the first part and some elementary calculations. 

Next, we turn to proving (\ref{eq:claim_beta}). Recall the definitions
$\bbeta_{k}=\vc(\bS_{k})=\vc(\Uhat^{\top}\Mkstar\Vhat)$ and $\disthat=\max\{\|\Uhat\Uhat^{\top}-\Ustar\Ustar^{\top}\|,\|\Vhat\Vhat^{\top}-\Vstar\Vstar^{\top}\|\}$.
We have the upper bound $\|\bbeta_{k}\|_{2}=\|\Uhat^{\top}\Mkstar\Vhat\|_{\mathrm{F}}\le\|\Mkstar\|_{\Frm}$
as well as the lower bound 
\begin{align*}
\big\|\bbeta_{k}\big\|_{2} & =\big\|\Uhat^{\top}\Mkstar\Vhat\big\|_{{\rm F}}=\|\Uhat\Uhat^{\top}\Mkstar\Vhat\Vhat^{\top}\|_{\Frm}\ge\|\Mkstar\|_{{\rm F}}-\|\Ustar\Ustar^{\top}\Mkstar\Vstar\Vstar^{\top}-\Uhat\Uhat^{\top}\Mkstar\Vhat\Vhat^{\top}\|_{{\rm F}}\\
 & \geq\|\Mkstar\|_{{\rm F}}-\|\big(\Ustar\Ustar^{\top}-\Uhat\Uhat^{\top}\big)\Mkstar\Vstar\Vstar^{\top}\|_{{\rm F}}-\|\Uhat\Uhat^{\top}\Mkstar\big(\Vstar\Vstar^{\top}-\Vhat\Vhat^{\top}\big)\|_{{\rm F}}\\
 & \ge(1-2\,\disthat)\|\Mkstar\|_{{\rm F}}\ge0.9\|\Mkstar\|_{\Frm},
\end{align*}
where the last inequality uses the assumption that $\disthat\le c_{1}/(K\Gamma^{2})\leq0.05$;
this justifies the first part of~(\ref{eq:claim_beta}). To prove
the second part of (\ref{eq:claim_beta}), we start with the decomposition
\begin{align*}
\big\langle\bbeta_{i},\bbeta_{j}\big\rangle & =\big\langle\Uhat^{\top}\Mistar\Vhat,\Uhat^{\top}\Mjstar\Vhat\big\rangle=\big\langle\Mistar,\Uhat\Uhat^{\top}\Mjstar\Vhat\Vhat^{\top}\big\rangle\\
 & =\big\langle\Mistar,\Mjstar\big\rangle+\big\langle\Mistar,\Uhat\Uhat^{\top}\Mjstar\Vhat\Vhat^{\top}-\Ustar\Ustar^{\top}\Mjstar\Vhat\Vhat^{\top}\big\rangle\\
 & \quad+\big\langle\Mistar,\Ustar\Ustar^{\top}\Mjstar\Vhat\Vhat^{\top}-\Ustar\Ustar^{\top}\Mjstar\Vstar\Vstar^{\top}\big\rangle,
\end{align*}
which together with the triangle inequality yields
\[
\big|\big\langle\bbeta_{i},\bbeta_{j}\big\rangle\big|\le|\langle\Mistar,\Mjstar\rangle|+2\,\disthat\|\Mistar\|_{\Frm}\|\Mjstar\|_{\Frm}.
\]
In light of the first part of (\ref{eq:claim_Mstar}), the first part
of (\ref{eq:claim_beta}), and our assumption on $\disthat$, this
establishes the second part of (\ref{eq:claim_beta}).

Finally, it remains to prove (\ref{eq:sigK}). In view of the assumption
that $p_{k}\asymp1/K$ ($1\leq k\leq K$), one has 
\begin{equation}
\sigma_{K}\bigg(\sum_{k=1}^{K}\pkhat\bbeta_{k}\bbeta_{k}^{\top}\bigg)\asymp\frac{1}{K}\sigma_{K}\bigg(\sum_{k=1}^{K}\bbeta_{k}\bbeta_{k}^{\top}\bigg).\label{eq:cons-of-balancedness}
\end{equation}
Therefore, it suffices to show that $\sigma_{K}(\sum_{k}\bbeta_{k}\bbeta_{k}^{\top})\asymp\min_{k}\|\bm{M}_{k}^{\star}\|_{\mathrm{F}}^{2}$.
Towards this, we find it helpful to define $\bB\coloneqq[\bbeta_{1},\dots,\bbeta_{K}]\in\R^{d\times K}$,
and decompose $\bB^{\top}\bB$ as $\bB^{\top}\bB=\bD+\bO$. Here,
$\bD$ stands for the diagonal part of $\bB^{\top}\bB$ with $D_{kk}=\|\bbeta_{k}\|_{2}^{2}$,
while $\bO$ is the off-diagonal part of $\bB^{\top}\bB$. Note that
for any $i\neq j$, $[\bO]_{ij}=[\bB^{\top}\bB]_{ij}=\langle\bbeta_{i},\bbeta_{j}\rangle$,
which combined with (\ref{eq:claim_beta}) gives
\[
\|\bO\|_{\Frm}^{2}=\sum_{i\neq j}\big\langle\bbeta_{i},\bbeta_{j}\big\rangle^{2}\le K^{2}\left(\frac{1}{2K\Gamma^{2}}\right)^{2}\max_{k}\big\|\bbeta_{k}\big\|_{2}^{4}=\frac{1}{4\Gamma^{4}}\max_{k}\big\|\bbeta_{k}\big\|_{2}^{4}\le\frac{1}{2}\min_{k}\big\|\bbeta_{k}\big\|_{2}^{4};
\]
the last inequality follows from the definition $\Gamma=\max_{k}\|\Mkstar\|_{\Frm}/\min_{k}\|\Mkstar\|_{\Frm}$,
and the first part of (\ref{eq:claim_beta}). This together with Weyl's
inequality implies that
\begin{equation}
\Big|\sigma_{K}\left(\bB^{\top}\bB\right)-\sigma_{K}(\bD)\Big|=\Big|\sigma_{K}\left(\bB^{\top}\bB\right)-\min_{k}\|\bbeta_{k}\|_{2}^{2}\Big|\le\|\bO\|_{{\rm F}}\le\frac{1}{\sqrt{2}}\min_{k}\big\|\bbeta_{k}\big\|_{2}^{2}.\label{eq:eq:sigK_BTB}
\end{equation}
As a result, we arrive at 
\[
\sigma_{K}\left(\sum_{k}\bbeta_{k}\bbeta_{k}^{\top}\right)=\sigma_{K}\left(\bB\bB^{\top}\right)=\sigma_{K}\left(\bB^{\top}\bB\right)\asymp\min_{k}\big\|\bbeta_{k}\big\|_{2}^{2}\asymp\min_{k}\|\Mkstar\|_{\Frm}^{2},
\]
which in conjunction with (\ref{eq:cons-of-balancedness}) completes
the proof of (\ref{eq:sigK}).
\end{proof}

\subsection{Proof of Proposition \ref{prop:init} \label{subsec:proof_prop_stage2}}

To begin with, the triangle inequality gives
\begin{equation}
\big\|\bL_{k}\bR_{k}^{\top}-\Mkstar\big\|_{{\rm F}}\le\big\|\Uhat\bS_{k}\Vhat^{\top}-\Mkstar\big\|_{{\rm F}}+\big\|\bL_{k}\bR_{k}^{\top}-\Uhat\bS_{k}\Vhat^{\top}\big\|_{{\rm F}}.\label{eq:stage2_error_decompose}
\end{equation}
Regarding the first term on the right-hand side of (\ref{eq:stage2_error_decompose}),
we plug in the definition (\ref{eq:def_Skhat}) of $\bS_{k}$ to obtain
\begin{align*}
\big\|\Uhat\bS_{k}\Vhat^{\top}-\Mkstar\big\|_{{\rm F}} & =\big\|\Uhat\Uhat^{\top}\Mkstar\Vhat\Vhat^{\top}-\Ustar\Ustar^{\top}\Mkstar\Vstar\Vstar^{\top}\big\|_{{\rm F}}\\
 & \leq\big\|\big(\Uhat\Uhat^{\top}-\Ustar\Ustar^{\top}\big)\Mkstar\Vhat\Vhat^{\top}\big\|_{{\rm F}}+\big\|\Ustar\Ustar^{\top}\Mkstar\big(\Vhat\Vhat^{\top}-\Vstar\Vstar^{\top}\big)\big\|_{{\rm F}}\\
 & \le2\max\Big\{\distU,\distV\Big\}\big\|\Mkstar\big\|_{{\rm F}}.
\end{align*}
With regards to the second term on the right-hand side of (\ref{eq:stage2_error_decompose}),
we observe that
\begin{align*}
\big\|\bL_{k}\bR_{k}^{\top}-\Uhat\bS_{k}\Vhat^{\top}\big\|_{{\rm F}} & \le\big\|\bL_{k}\bR_{k}^{\top}-\Uhat\Skhat\Vhat^{\top}\big\|_{\Frm}+\big\|\Uhat\Skhat\Vhat^{\top}-\Uhat\bS_{k}\Vhat^{\top}\big\|_{\Frm}\\
 & \overset{{\rm (i)}}{\le}2\big\|\Uhat(\Skhat-\bS_{k})\Vhat^{\top}\big\|_{\Frm}\overset{{\rm (ii)}}{\le}2\big\|\betakhat-\bbeta_{k}\big\|_{2}.
\end{align*}
Here, (i) follows since $\bL_{k}\bR_{k}^{\top}$ is the best rank-$r_{k}$
approximation of $\Uhat\Skhat\Vhat^{\top}$ and $\Uhat\bS_{k}\Vhat^{\top}$
is also rank-$r_{k}$; (ii) holds since $\betakhat=\vc(\Skhat)$ and
$\bbeta_{k}=\vc(\bS_{k})$. Substitution into (\ref{eq:stage2_error_decompose})
establishes (\ref{eq:init_error}).

\subsection{\label{subsec:proof_stage3}Proof of Theorem \ref{thm:stage3} }

We shall only prove the local convergence w.r.t.~the matrix $\Monestar$;
the proof for other components is identical and hence is omitted.
Our proof is decomposed into three steps.
\begin{enumerate}
\item Study the ScaledTGD dynamics (particularly the population-level dynamics),
and control the effects of mislabeling and finite-sample errors.
\item Show that if the estimation error is larger than the error floor (namely,
the last term in (\ref{eq:tsgd_error_floor})), then one step of the
ScaledTGD update contracts the error by a constant factor.
\item Show that, once the estimation error gets smaller than this error
floor, then the estimation errors remain small in subsequent iterations.
\end{enumerate}
Before continuing, we note that Condition (\ref{eq:local_region})
with $k=1$ implies the existence of some constant $c_{1}>0$ such
that 
\[
(\bL^{0},\bR^{0})\in\neighbor,
\]
where
\begin{equation}
\neighbor\coloneqq\Bigg\{(\bL,\bR)\in\R^{n_{1}\times r_{1}}\times\R^{n_{2}\times r_{1}}:\big\|\bL\bR^{\top}-\bm{M}_{1}^{\star}\big\|_{\Frm}\le c_{1}\min\left\{ \sigma_{r_{1}}(\bm{M}_{1}^{\star}),\,\frac{1}{K}\min_{j\neq1}\big\|\Mjstar-\bm{M}_{1}^{\star}\big\|_{\Frm}\right\} \Bigg\}.\label{eq:local_region_general}
\end{equation}
 This arises from the inequalities $\sigma_{r_{1}}(\bm{M}_{1}^{\star})\ge\|\bm{M}_{1}^{\star}\|_{\Frm}/(\sqrt{r}\kappa)$
and $\min_{j\neq1}\|\Mjstar-\bm{M}_{1}^{\star}\|_{\Frm}\gtrsim\|\bm{M}_{1}^{\star}\|_{\Frm}$
(due to Assumption \ref{asp:weak_correlation}). We isolate Condition
(\ref{eq:local_region_general}) since it is more convenient to work
with in the analysis. 

\paragraph{Notation.}

To simplify presentation, we shall often let $(\bm{L},\bm{R})$ denote
an iterate lying within $\mathcal{B}$ (cf.~\eqref{eq:local_region_general}),
and define the corresponding estimation errors as
\begin{equation}
\Delk\coloneqq\bL\bR^{\top}-\Mkstar,\quad1\le k\le K.\label{eq:estimation-error-Delta-k}
\end{equation}
The truncating level for a prescribed truncating fraction $\alpha$
is denoted by
\begin{equation}
\tau\coloneqq\quantalpha\left(\Big\{\big|\langle\Ai,\bL\bR^{\top}\rangle-\yi\big|\Big\}_{1\le i\le N}\right),\label{eq:def_tau_tsgd}
\end{equation}
where $Q_{\alpha}$ is the $\alpha$-quantile defined in Section \ref{subsec:Notation}.
We also define the following functions and quantities:
\begin{align}
\ind(a;b) & \coloneqq\ind(|a|\le b),\quad a,b\in\R,\label{eq:def_indicator}\\
w(x)\coloneqq\int_{-x}^{x}t^{2}\phi(t)\,\mathrm{d}t,\quad x\ge0, & \qquad\quad w_{k}\coloneqq w\left(\frac{\tau}{\sqrt{\|\Delk\|_{\Frm}^{2}+\sigma^{2}}}\right),\quad1\le k\le K,\label{eq:def_w}
\end{align}
where $\phi$ stands for the probability density function of a standard
Gaussian random variable. 

\paragraph{Step 1: characterizing the ScaledTGD dynamic.}

The above notation allows one to express the ScaledTGD update rule
(\ref{eq:tsgd_updates}) as
\begin{subequations}
\label{eq:tsgd_ind}
\begin{align}
\Lp & =\bL-\frac{\eta}{\Nall}\sumall\big(\langle\Ai,\bL\bR^{\top}\rangle-\yi\big)\ind\big(\langle\Ai,\bL\bR^{\top}\rangle-\yi;\tau\big)\Ai\bR\big(\bR^{\top}\bR\big)^{-1},\label{eq:tsgd_update_L_ind}\\
\Rp & =\bR-\frac{\eta}{\Nall}\sumall\big(\langle\Ai,\bL\bR^{\top}\rangle-\yi\big)\ind\big(\langle\Ai,\bL\bR^{\top}\rangle-\yi;\tau\big)\Ai^{\top}\bL\big(\bL^{\top}\bL\big)^{-1}.\label{eq:tsgd_update_R_ind}
\end{align}
\end{subequations}
Recall that for any $i\in\Omegakstar$, we have $\yi=\langle\Ai,\Mkstar\rangle+\ei$,
and thus 
\begin{equation}
\langle\Ai,\bL\bR^{\top}\rangle-\yi=\langle\Ai,\Delk\rangle-\ei,\qquad\text{for all }i\in\Omega_{k}^{\star}.\label{eq:product-Ai-L-R}
\end{equation}
The following result makes apparent a useful decomposition of the
ScaledTGD update rule.

\begin{itclaim} 

Recall the notation \eqref{eq:def_indicator} and \eqref{eq:def_w}.
The ScaledTGD update rule (\ref{eq:tsgd_ind}) can be written as 
\begin{equation}
\Lp=\Ltildep-\eta\EL,\qquad\Rp=\Rtildep-\eta\ER.\label{eq:Lplus_pop_E}
\end{equation}
Here, $(\Ltildep,\Rtildep)$ represents the population-level update
from $\Omegaonestar$  
\begin{equation}
\Ltildep\coloneqq\bL-\eta\ponehat w_{1}\Delone\Rinv,\qquad\Rtildep\coloneqq\bR-\eta\ponehat w_{1}\Delone^{\top}\Linv,\label{eq:tsgd_pop}
\end{equation}
and the residual components are given by
\[
\EL\coloneqq(\Delmis+\Deltrip)\Rinv,\qquad\ER\coloneqq(\Delmis+\Deltrip)^{\top}\Linv
\]
with
\begin{equation}
\Delmis\coloneqq\sum_{k\neq1}\pkhat w_{k}\Delk,\quad\Deltrip\coloneqq\sum_{k=1}^{K}\pkhat\left(\frac{1}{|\Omegakstar|}\sumomgk\big(\langle\Ai,\Delk\rangle-\ei\big)\ind\big(\langle\Ai,\Delk\rangle-\ei;\tau\big)\Ai-w_{k}\Delk\right).\label{eq:def_Delmis}
\end{equation}

\end{itclaim}

Before moving on, we note that it is crucial to control the sizes
of $\Delmis$ and $\Deltrip$, where ``$\mathsf{mis}$'' stands
for ``mislabeling'', and ``$\mathsf{fs}$'' stands for ``finite
sample''. Regarding $\Delmis$, Fact \ref{fact:w_quadratic} tells
us that for all $k\neq1$,
\begin{align*}
w_{k}\|\Delk\|_{{\rm F}} & =w_{1}\frac{w_{k}}{w_{1}}\|\Delk\|_{{\rm \Frm}}\le w_{1}\frac{\|\Delone\|_{\Frm}^{2}+\sigma^{2}}{\|\Delk\|_{\Frm}^{2}+\sigma^{2}}\|\Delk\|_{{\rm F}}\le\frac{\|\Delone\|_{\Frm}^{2}+\sigma^{2}}{\|\Delk\|_{{\rm F}}}\le2\frac{\|\Delone\|_{\Frm}^{2}+\sigma^{2}}{\|\Mkstar-\Monestar\|_{{\rm F}}}.
\end{align*}
Here, the last inequality holds since 
\[
\|\Delk\|_{{\rm F}}=\|\bm{L}\bm{R}^{\top}-\Mkstar\|_{{\rm F}}\ge\|\Monestar-\Mkstar\|_{{\rm F}}-\|\Delone\|_{{\rm F}}\geq0.5\|\Monestar-\Mkstar\|_{{\rm F}},
\]
where we have used $\|\Delone\|_{{\rm F}}\le c_{2}\sigma_{r_{1}}(\Monestar)\le0.5\|\Monestar-\Mkstar\|_{{\rm F}}$
due to the assumption that $(\bL,\bR)\in\neighbor$ defined in (\ref{eq:local_region_general}).
Consequently, we obtain
\begin{equation}
\|\Delmis\|_{{\rm F}}=\Big\|\sum_{k\neq1}\pkhat w_{k}\Delk\Big\|_{\Frm}\le\sum_{k\ne1}\pkhat w_{k}\|\Delk\|_{{\rm F}}\le2\frac{\|\Delone\|_{\Frm}^{2}+\sigma^{2}}{\min_{k\neq1}\|\Mkstar-\Monestar\|_{\Frm}}.\label{eq:Delmis_bound}
\end{equation}
Next, we turn to the term $\Deltrip$.  Note that $\rank(\Delk)\le2r$.
Therefore, Lemmas \ref{lem:TRIP} and \ref{lem:quantile} (see Remark~\ref{rem:noisy_TRIP_quantile})
imply that, with probability at least $1-Ce^{-cn}$ for some constants
$c,C>0$, the following holds simultaneously for all $(\bL,\bR)\in\mathcal{B}$
(cf.~(\ref{eq:local_region_general})): 
\begin{enumerate}
\item the truncating level $\tau$ obeys
\begin{equation}
0.54<\frac{\tau}{\sqrt{\|\Delone\|_{\Frm}^{2}+\sigma^{2}}}<1.35;\label{eq:tau_range_tsgd}
\end{equation}
\item for any real matrix $\bW$ with $n_{2}$ rows and of rank at most
$r$, we have
\begin{equation}
\|\Deltrip\bW\|_{\Frm}\le\sum_{k=1}^{K}\pkhat\delta\tau\|\bW\|=\delta\tau\|\bW\|\le1.35\delta\sqrt{\|\Delone\|_{\Frm}^{2}+\sigma^{2}}\|\bW\|.\label{eq:case1_fs}
\end{equation}
\end{enumerate}
The above-mentioned bounds will play a useful role in subsequent steps.

\paragraph{Step 2: per-iteration improvement above the error floor (\ref{eq:tsgd_error_floor}).}

Let us look at the Euclidean error
\begin{subequations}
\label{eq:tsgd_improve}
\begin{align}
\big\|\bL^{+}(\bR^{+})^{\top} & -\Monestar\big\|_{{\rm F}}=\big\|(\Ltildep-\eta\EL)(\Rtildep-\eta\ER)^{\top}-\Monestar\big\|_{{\rm F}}\label{eq:tsgd_Fro}\\
 & \le\big\|\Ltildep(\Rtildep)^{\top}-\Monestar\big\|_{{\rm F}}+\eta\Big(\big\|\EL(\Rtildep)^{\top})\big\|_{{\rm F}}+\big\|\Ltildep(\ER)^{\top}\big\|_{{\rm F}}+\eta\big\|\EL(\ER)^{\top}\big\|_{{\rm F}}\Big).\label{eq:tsgd_Fro_decompose}
\end{align}
\end{subequations}
Since $\Ltildep$ and $\Rtildep$ (\ref{eq:tsgd_pop}) are exactly
the same as the update rule of scaled gradient descent for low-rank
matrix factorization, \cite[Theorem 5]{tong2020accelerating} tells
us that if $0<\eta\ponehat w_{1}\le2/3$ (which holds true under our
choices of $\eta\le1.3/\ponehat$ and $\alpha\le0.8\ponehat$), then
\begin{equation}
\big\|\Ltildep(\Rtildep)^{\top}-\Monestar\big\|_{{\rm F}}\le\big(1-0.7\eta\ponehat w_{1}\big)\big\|\bL\bR^{\top}-\Monestar\big\|_{{\rm F}}.\label{eq:pop_improvement}
\end{equation}
It remains to control the perturbation terms in (\ref{eq:tsgd_Fro_decompose}),
accomplished as follows.

\begin{itclaim} \label{claim:perturb}

Denoting
\begin{equation}
B\coloneqq2\left(\delta\sqrt{\|\Delone\|_{\Frm}^{2}+\sigma^{2}}+\frac{\|\Delone\|_{\Frm}^{2}+\sigma^{2}}{\min_{k\neq1}\|\Mkstar-\Monestar\|_{{\rm F}}}\right),\label{eq:def_star}
\end{equation}
one has
\begin{equation}
\max\Big\{\big\|\EL(\Rtildep)^{\top}\big\|_{{\rm F}},\big\|\Ltildep(\ER)^{\top}\big\|_{{\rm F}}\Big\}\le2B,\quad\big\|\EL(\ER)^{\top}\big\|_{{\rm F}}\le\frac{2}{\sigma_{r_{1}}(\Monestar)}B^{2}.\label{eq:claim_perturb}
\end{equation}

\end{itclaim}

Putting (\ref{eq:pop_improvement}) and (\ref{eq:claim_perturb})
back to (\ref{eq:tsgd_improve}) and denoting $\Delone^{+}\coloneqq\bL^{+}(\bR^{+})^{\top}-\Monestar$,
we have
\begin{align}
\|\Delone^{+}\|_{{\rm F}} & \le\big(1-0.7\eta\ponehat w_{1}\big)\|\Delone\|_{{\rm F}}+\eta\left(4B+\frac{2\eta}{\sigma_{r_{1}}(\Monestar)}B^{2}\right).\label{eq:Delta_one_step}
\end{align}
It remains to control $B$. First, the relations $\delta\le c_{0}/K$
and $\|\Delone\|_{\Frm}\ge C_{2}K\delta\sigma$ (for some sufficiently
large constant $C_{2}>0$) imply that
\begin{equation}
\delta\sqrt{\|\Delone\|_{{\rm F}}+\sige^{2}}\le\delta\|\Delone\|_{\Frm}+\delta\sigma\le\frac{c_{3}}{K}\|\Delone\|_{{\rm F}}\label{eq:star1}
\end{equation}
for some sufficiently small constant $c_{3}>0$. Moreover, observing
that
\[
\frac{C_{2}K\sigma^{2}}{\min_{k\neq1}\|\Mkstar-\Monestar\|_{\Frm}}\le\|\Delone\|_{\Frm}\le\frac{c_{1}}{K}\min_{k\neq1}\|\Mkstar-\Monestar\|_{\Frm},
\]
we have
\begin{equation}
\frac{\|\Delone\|_{{\rm F}}^{2}+\sige^{2}}{\min_{k\neq1}\|\Mkstar-\Monestar\|_{{\rm F}}}=\frac{\|\Delone\|_{{\rm F}}^{2}}{\min_{k\neq1}\|\Mkstar-\Monestar\|_{{\rm F}}}+\frac{\sige^{2}}{\min_{k\neq1}\|\Mkstar-\Monestar\|_{{\rm F}}}\le\frac{c_{4}}{K}\|\Delone\|_{{\rm F}}\label{eq:star2}
\end{equation}
for some sufficiently small constant $c_{4}>0$. Putting (\ref{eq:star1})
and (\ref{eq:star2}) back into (\ref{eq:def_star}), we have
\begin{equation}
B\le\frac{2(c_{3}+c_{4})}{K}\|\Delone\|_{\Frm},\label{eq:B_Delone}
\end{equation}
which together with $\|\Delone\|_{\Frm}\le c_{1}\sigma_{r_{1}}(\Monestar)$
implies the existence of some small constant $c_{5}>0$ such that
\[
4B+\frac{2\eta}{\sigma_{r_{1}}(\Monestar)}B^{2}=4B\left(1+\frac{\eta}{2\sigma_{r_{1}}(\Monestar)}B\right)\le8B\le\frac{c_{5}}{K}\|\Delone\|_{\Frm}.
\]
Substituting this into (\ref{eq:Delta_one_step}), we arrive at the
desired bound
\[
\|\Delone^{+}\|_{\Frm}\le(1-c_{2}\eta\ponehat)\|\Delone\|_{\Frm}
\]
for some constant $c_{2}>0$; this is because in (\ref{eq:Delta_one_step}),
we have $p_{1}\asymp1/K$ by assumption, and $w_{1}\gtrsim1$ according
to (\ref{eq:tau_range_tsgd}).

\paragraph{Step 3: no blowing up below the error floor (\ref{eq:tsgd_error_floor}).}

Suppose that the estimation error satisfies
\begin{equation}
\|\Delone\|_{\Frm}\lesssim\max\left\{ K\sigma\delta,\frac{K\sigma^{2}}{\min_{k\neq1}\|\Mkstar-\Monestar\|_{\Frm}}\right\} .\label{eq:below_error_floor}
\end{equation}
We intend to show that, in this case, the estimation error of the
next iterate $\|\Delone^{+}\|_{\Frm}$ satisfies the same upper bound
(\ref{eq:below_error_floor}); if this claim were true, then combining
this with our results in Step 2 would complete the convergence analysis
of ScaledTGD. 

Note that (\ref{eq:Delta_one_step}) remains valid when $\|\Delone\|_{\Frm}$
is below the error floor, which implies that
\begin{equation}
\|\Delone^{+}\|_{\Frm}\lesssim\|\Delone\|_{\Frm}+KB\left(1+\frac{KB}{\sigma_{r_{1}}(\Monestar)}\right).\label{eq:one_step_tsgd}
\end{equation}
Recalling the definition of $B$ in (\ref{eq:def_star}), one has
\[
KB\lesssim K\big(\|\Delone\|_{\Frm}+\sigma\big)\left(\delta+\frac{\|\Delone\|_{\Frm}+\sigma}{\min_{k\neq1}\|\Mkstar-\Monestar\|_{\Frm}}\right).
\]
By the assumption that $\delta\lesssim1/K$ and $\sigma\lesssim\min_{k}\|\Mkstar\|_{\Frm}/K$,
we have $\|\Delone\|_{\Frm}\lesssim\sigma$ according to~(\ref{eq:below_error_floor}),
and thus $KB/\sigma_{r_{1}}(\Monestar)\lesssim\sigma/\sigma_{r_{1}}(\Monestar)\lesssim1$.
Consequently, on the right-hand side of (\ref{eq:one_step_tsgd})
we have
\[
KB\left(1+\frac{KB}{\sigma_{r_{1}}(\Monestar)}\right)\lesssim KB\lesssim K\sigma\left(\delta+\frac{\sigma}{\min_{k\neq1}\|\Mkstar-\Monestar\|_{\Frm}}\right),
\]
which has exactly the same form as the error floor in (\ref{eq:below_error_floor}).
This completes our proof for this step. 
\begin{proof}
[Proof of Claim \ref{claim:perturb}] We shall only prove the first
part of (\ref{eq:claim_perturb}) concerning $\|\EL(\Rtildep)^{\top}\|_{{\rm F}}$;
the analysis for $\|\Ltildep(\ER)^{\top}\|_{{\rm F}}$ is essentially
the same. By the triangle inequality, we have 
\begin{align}
\big\|\EL(\Rtildep)^{\top}\big\|_{{\rm F}} & \le\big\|\EL\bR^{\top}\big\|_{{\rm F}}+\eta\ponehat w_{1}\big\|\EL(\bL^{\top}\bL)^{-1}\bL^{\top}\Delone\big\|_{{\rm F}}.\label{eq:EL_Rpop}
\end{align}
We utilize (\ref{eq:Delmis_bound}) and (\ref{eq:case1_fs}) from
Step 1 to control the terms above. For the first term of \eqref{eq:EL_Rpop},
recognizing that $\|\Rinv\bR^{\top}\|\le1$ we have
\begin{align*}
\big\|\EL\bR^{\top}\big\|_{{\rm F}} & =\big\|(\Deltrip+\Delmis)\Rinv\bR^{\top}\big\|_{{\rm F}}\le\big\|\Deltrip\Rinv\bR^{\top}\big\|_{{\rm F}}+\big\|\Delmis\big\|_{{\rm F}}\\
 & \le2\left(\delta\sqrt{\|\Delone\|_{\Frm}^{2}+\sigma^{2}}+\frac{\|\Delone\|_{\Frm}^{2}+\sigma^{2}}{\min_{k\neq1}\|\Mkstar-\Monestar\|_{{\rm F}}}\right)=B.
\end{align*}
Regarding the second term of \eqref{eq:EL_Rpop}, we observe that
\begin{align*}
\big\|\EL(\bL^{\top}\bL)^{-1}\bL^{\top}\Delone\big\|_{{\rm F}} & =\big\|(\Deltrip+\Delmis)\Rinv(\bL^{\top}\bL)^{-1}\bL^{\top}\Delone\big\|_{{\rm F}}\\
 & \le\Big(\big\|\Deltrip\Rinv(\bL^{\top}\bL)^{-1}\bL^{\top}\big\|_{{\rm F}}+\big\|\Delmis\big\|_{{\rm F}}\big\|\Rinv(\bL^{\top}\bL)^{-1}\bL^{\top}\big\|\Big)\|\Delone\|_{{\rm F}}\\
 & \overset{{\rm (i)}}{\le}2\left(\delta\sqrt{\|\Delone\|_{\Frm}^{2}+\sigma^{2}}+\frac{\|\Delone\|_{\Frm}^{2}+\sigma^{2}}{\min_{k\neq1}\|\Mkstar-\Monestar\|_{{\rm F}}}\right)\frac{2}{\sigma_{r_{1}}(\Monestar)}\cdot c_{1}\sigma_{r_{1}}(\Monestar)\\
 & =4c_{1}\left(\delta\sqrt{\|\Delone\|_{\Frm}^{2}+\sigma^{2}}+\frac{\|\Delone\|_{\Frm}^{2}+\sigma^{2}}{\min_{k\neq1}\|\Mkstar-\Monestar\|_{{\rm F}}}\right)=2c_{1}B,
\end{align*}
where (i) follows from $\|\Delone\|_{{\rm F}}\le c_{1}\sigma_{r_{1}}(\Monestar)$
(see (\ref{eq:local_region_general})) as well as the following fact
(which will be proved at the end of this section): for any $\bL\in\R^{n_{1}\times r_{1}}$
and $\bR\in\R^{n_{2}\times r_{1}}$,
\begin{equation}
\text{if}\quad\big\|\bL\bR^{\top}-\Monestar\big\|_{{\rm F}}\le\frac{\sigma_{r_{1}}(\Monestar)}{2},\quad\text{then}\quad\big\|\Linv(\bR^{\top}\bR)^{-1}\bR^{\top}\big\|\le\frac{2}{\sigma_{r_{1}}(\Monestar)}.\label{eq:fact_Linv_Rinv}
\end{equation}
Combining these with (\ref{eq:EL_Rpop}) establishes that $\|\EL(\Rtildep)^{\top}\|_{{\rm F}}\le2B$,
which is the first part of (\ref{eq:claim_perturb}). 

Finally, for the second part of (\ref{eq:claim_perturb}), we can
apply similar techniques to reach
\begin{align*}
\big\|\EL(\ER)^{\top}\big\|_{{\rm F}} & =\big\|(\Deltrip+\Delmis)\Rinv(\bL^{\top}\bL)^{-1}\bL^{\top}(\Deltrip+\Delmis)\big\|_{{\rm F}}\\
 & \le\frac{2}{\sigma_{r_{1}}(\Monestar)}\cdot4\left(\delta\sqrt{\|\Delone\|_{\Frm}^{2}+\sigma^{2}}+\frac{\|\Delone\|_{\Frm}^{2}+\sigma^{2}}{\min_{k\neq1}\|\Mkstar-\Monestar\|_{{\rm F}}}\right)^{2}=\frac{2}{\sigma_{r_{1}}(\Monestar)}B^{2}.
\end{align*}
\end{proof}
\begin{proof}
[Proof of (\ref{eq:fact_Linv_Rinv})]Weyl's inequality tells us that
\begin{equation}
\sigma_{r_{1}}\big(\bL\bR^{\top}\big)\ge\sigma_{r_{1}}\big(\Monestar\big)-\big\|\bL\bR^{\top}-\Monestar\big\|_{{\rm F}}\ge\frac{\sigma_{r_{1}}\big(\Monestar\big)}{2},\label{eq:sigr1}
\end{equation}
which further implies that both $\bL$ and $\bR$ have full column
rank $r_{1}$. Consequently, we denote the SVD of $\bL$ and $\bR$
as $\bL=\UL\SigL\VL^{\top}$ and $\bR=\UR\SigR\VR^{\top}$, where
$\bV_{L},\bV_{R}$ are $r_{1}\times r_{1}$ orthonormal matrices.
With the SVD representations in place, it is easy to check that 
\[
\bL\bR^{\top}=\UL\SigL\VL^{\top}\VR\SigR\UR^{\top},\quad\text{and}\quad\Linv(\bR^{\top}\bR)^{-1}\bR^{\top}=\UL\SigL^{-1}\VL^{\top}\VR\SigR^{-1}\UR^{\top}.
\]
In addition, the orthonormality of $\bm{V}_{L}$ and $\bm{V}_{R}$
implies 
\[
(\SigL\VL^{\top}\VR\SigR)^{-1}=\SigR^{-1}(\VL^{\top}\VR)^{-1}\SigL^{-1}=\SigR^{-1}\VR^{\top}\VL\SigL^{-1}=(\SigL^{-1}\VL^{\top}\VR\SigR^{-1})^{\top},
\]
thus indicating that
\begin{align*}
\big\|\Linv(\bR^{\top}\bR)^{-1}\bR^{\top}\big\| & =\big\|\UL\SigL^{-1}\VL^{\top}\VR\SigR^{-1}\UR^{\top}\big\|=\big\|\SigL^{-1}\VL^{\top}\VR\SigR^{-1}\big\|\\
 & =\big\|(\SigL\VL^{\top}\VR\SigR)^{-1}\big\|=\frac{1}{\sigma_{r_{1}}\big(\SigL\VL^{\top}\VR\SigR\big)}=\frac{1}{\sigma_{r_{1}}\big(\bL\bR^{\top}\big)}.
\end{align*}
Combining this with (\ref{eq:sigr1}) completes the proof.
\end{proof}

\section{Technical lemmas \label{subsec:technical_lemmas}}

This section collects several technical lemmas that are helpful for
our analysis (particularly for the analysis of Stage 3). For notational
convenience, we define the set of low-rank matrices as 
\begin{equation}
\setlowrank_{r}\coloneqq\big\{\bX\in\R^{n_{1}\times n_{2}}:\rank(\bX)\le r\big\}.\label{eq:defn-rank-r-set}
\end{equation}
We remind the reader of the definitions $\ind(a;b)=\ind(|a|\le b)$
for $a,b\in\mathbb{R}$ and $w(x)=\int_{-x}^{x}t^{2}\phi(t)\,\mathrm{d}t$
for $x\geq0$.

\paragraph{Variants of matrix-RIP. }

We recall the standard notion of restricted isometry property (RIP)
from the literature of matrix sensing, and introduce a variant called
\emph{truncated} RIP (TRIP). 
\begin{defn}
Let $\{\Ai\}_{i=1}^{m}$ be a set of matrices in $\R^{n_{1}\times n_{2}}$.
Consider $1\le r\le\min\{n_{1},n_{2}\}$ and $0<\delta<1$. 
\begin{enumerate}
\item We say that $\{\Ai\}_{1\le i\le m}$ satisfy $(r,\delta)$-RIP if
\begin{equation}
\bigg|\frac{1}{m}\sum_{i=1}^{m}\langle\Ai,\bX\rangle\langle\Ai,\bZ\rangle-\langle\bX,\bZ\rangle\bigg|\le\delta\|\bX\|_{{\rm F}}\|\bZ\|_{{\rm F}}\label{eq:defn-RIP}
\end{equation}
holds simultaneously for all $\bX,\bm{Z}\in\setlowrank_{r}$.
\item We say that $\{\Ai\}_{1\le i\le m}$ satisfy $(r,\delta)$-TRIP if
\begin{equation}
\bigg|\frac{1}{m}\sum_{i=1}^{m}\langle\Ai,\bX\rangle\ind\big(\langle\Ai,\bX\rangle;\tau\|\bX\|_{\Frm}\big)\langle\Ai,\bZ\rangle-w(\tau)\langle\bX,\bZ\rangle\bigg|\le\delta\tau\|\bX\|_{\Frm}\|\bZ\|_{{\rm F}}\label{eq:def_TRIP}
\end{equation}
holds simultaneously for all $\bX,\bZ\in\setlowrank_{r}$ and for
all $0\le\tau\le1.35$.
\end{enumerate}
\end{defn}
As it turns out, the Gaussian design satisfies the above notion of
RIP and TRIP, as formalized below. 
\begin{lem}
\label{lem:TRIP} Let $\{\Ai\}_{1\le i\le m}$ be random matrices
in $\R^{n_{1}\times n_{2}}$ with i.i.d.~$\Ncal(0,1)$ entries, and
denote $n\coloneqq\max\{n_{1},n_{2}\}$. There exist some sufficiently
large constants $C_{1},C_{3}>0$ and some other constants $C_{2},c_{2},C_{4},c_{4}>0$
such that 
\begin{enumerate}
\item If $m\ge C_{1}nr\delta^{-2}\log(1/\delta)$, then with probability
at least $1-C_{2}e^{-c_{2}n}$, $\{\Ai\}_{1\le i\le m}$ satisfy $(r,\delta)$-RIP.
\item If $m\ge C_{3}nr\delta^{-2}\log m$, then with probability at least
$1-C_{4}e^{-c_{4}n}$, $\{\Ai\}_{1\le i\le m}$ satisfy $(r,\delta)$-TRIP. 
\end{enumerate}
\end{lem}

\paragraph{Empirical quantiles.}

Our next technical lemma is a uniform concentration result for empirical
quantiles. Given the design matrices $\{\Ai\}_{1\le i\le N}$, the
index sets $\{\Omegakstar\}_{1\le k\le K}$ and the low-rank matrices
$\{\bX_{k}\}_{1\le k\le K}$, we define several sets as follows:
\begin{equation}
\Dcal_{k}\coloneqq\Big\{\big|\langle\Ai,\bX_{k}\rangle\big|\Big\}_{i\in\Omegakstar},\quad1\le k\le K;\qquad\Dcal\coloneqq\Dcal_{1}\cup\dots\cup\Dcal_{K}.\label{eq:def_D}
\end{equation}
In addition, let us introduce the following set of low-rank matrices:
\begin{equation}
\Tcal_{1}\coloneqq\Big\{(\bX_{1},\dots,\bX_{K}):\bX_{k}\in\setlowrank_{r},1\le k\le K;\,0<\|\bX_{1}\|_{\Frm}\le\frac{c_{0}}{K}\,\min_{k\neq1}\|\bX_{k}\|_{\Frm}\Big\},\label{eq:set_Tquant}
\end{equation}
where $c_{0}>0$ is some sufficiently small constant. Recall that
$\quantalpha(\Dcal)$ denotes the $\alpha$-quantile of $\Dcal$,
as defined in (\ref{eq:def_quantile_distr}). 
\begin{lem}
\label{lem:quantile} Let $\{\Ai\}_{1\le i\le N}$ be random matrices
in $\R^{n_{1}\times n_{2}}$ with i.i.d.~$\Ncal(0,1)$ entries. Set
$n=\max\{n_{1},n_{2}\}$, and suppose the index sets $\{\Omegakstar\}_{1\le k\le K}$
are disjoint and satisfy the condition (\ref{eq:balanced_pk}). If
$0.6\ponehat\le\alpha\le0.8\ponehat$ and $\Nall\ge C_{0}nrK^{3}\log N$
for some sufficiently large constant $C_{0}>0$, then there exist
some universal constants $C,c>0$ such that: with probability at least
$1-Ce^{-cn}$,
\[
0.54<\frac{Q_{\alpha}(\Dcal)}{\|\bX_{1}\|_{\Frm}}<1.35
\]
holds simultaneously for all $(\bX_{1},\dots,\bX_{K})\in\Tcal_{1}$,
where $\Dcal$ is defined in (\ref{eq:def_D}).
\end{lem}
\begin{rem}
\label{rem:noisy_TRIP_quantile} We can further incorporate additional
Gaussian noise $\{\ei\}$ into Lemmas \ref{lem:TRIP} and \ref{lem:quantile},
where $\ei\iid\Ncal(0,\sigma^{2})$. For example, we claim that, with
the same sample complexity $m$ as in Lemma \ref{lem:TRIP}, we have
the following noisy version of $(r,\delta)$-TRIP (\ref{eq:def_TRIP}):
\begin{equation}
\bigg|\frac{1}{m}\sum_{i=1}^{m}\big(\langle\Ai,\bX\rangle-\ei\big)\ind\bigg(\langle\Ai,\bX\rangle-\ei;\tau\sqrt{\|\bX\|_{\Frm}^{2}+\sigma^{2}}\bigg)\langle\Ai,\bZ\rangle-w(\tau)\langle\bX,\bZ\rangle\bigg|\le\delta\tau\sqrt{\|\bX\|_{\Frm}^{2}+\sigma^{2}}\|\bZ\|_{{\rm F}}.\label{eq:noisy_trip}
\end{equation}
To see this, let us define the augmented matrices
\[
\Xaug\coloneqq\begin{bmatrix}\bX & \boldsymbol{0}\\
\boldsymbol{0} & -\sigma
\end{bmatrix},\quad\Zaug\coloneqq\begin{bmatrix}\bZ & \boldsymbol{0}\\
\boldsymbol{0} & 0
\end{bmatrix},\quad\Aiaug\coloneqq\begin{bmatrix}\Ai & *\\*
* & \zeta_{i}/\sigma
\end{bmatrix},\quad1\le i\le m,
\]
where $*$ stands for some auxiliary i.i.d.~$\Ncal(0,1)$ entries.
Observe that $\{\Aiaug\}_{1\le i\le N}$ are random matrices with
i.i.d.~$\Ncal(0,1)$ entries; in addition, $\rank(\Xaug)=\rank(\bX)+1$,
$\rank(\Zaug)=\rank(\bZ)$, and $\|\Xaug\|_{\Frm}^{2}=\|\bX\|_{\Frm}^{2}+\sigma^{2}$;
finally, $\langle\Ai,\bX\rangle-\ei=\langle\Aiaug,\Xaug\rangle$,
$\langle\Ai,\bZ\rangle=\langle\Aiaug,\Zaug\rangle$, and $\langle\bX,\bZ\rangle=\langle\Xaug,\Zaug\rangle$.
Therefore, the left-hand side of (\ref{eq:noisy_trip}) can be equivalently
written as in the noiseless form (\ref{eq:def_TRIP}), in terms of
these augmented matrices, thus allowing us to apply Lemma \ref{lem:TRIP}
to prove (\ref{eq:noisy_trip}). This trick of augmentation can be
applied to Lemma \ref{lem:quantile} as well, which we omit here for
brevity.
\end{rem}

\paragraph{One miscellaneous result. }

Further, we record below a basic property concerning the function
$w(\cdot)$. 
\begin{fact}
\label{fact:w_quadratic}The function $w(\cdot)$ defined in (\ref{eq:def_w})
satisfies
\begin{equation}
\frac{w(x)}{w(y)}\le\frac{x^{2}}{y^{2}},\quad\quad0<x\le y\le1.35.\label{eq:w_quadratic}
\end{equation}
\end{fact}
\begin{proof}
This result is equivalent to saying $w(x)/x^{2}\le w(y)/y^{2}$ for
any $0<x\le y\le1.35$.\texttt{ }Hence, it suffices to show that the
function $g(x)\coloneqq w(x)/x^{2}$ is nondecreasing over $(0,1.35]$,
or equivalently,
\begin{equation}
h(x)\coloneqq\sqrt{\frac{2}{\pi}}x^{3}e^{-\frac{x^{2}}{2}}-2w(x),\quad g'(x)=\frac{1}{x^{3}}h(x)\ge0,\quad0<x\le1.35.\label{eq:def_h}
\end{equation}
This can be verified numerically (see Figure \ref{fig:func_h}), which
completes the proof.
\end{proof}
\begin{figure}
\begin{centering}
\includegraphics[width=0.35\textwidth]{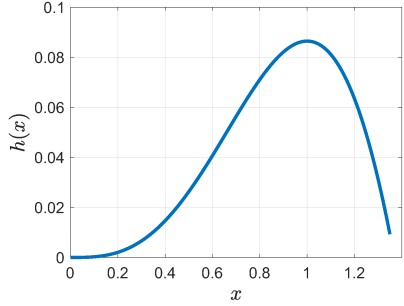}
\par\end{centering}
\caption{\label{fig:func_h} The function $h(\cdot)$ defined in (\ref{eq:def_h})
is nonnegative over the interval $(0,1.35]$.}
\end{figure}

The rest of this section is devoted to proving Lemmas \ref{lem:TRIP}
and \ref{lem:quantile}. We use the standard notions (e.g.~the subgaussian
norm $\|\cdot\|_{\psi_{2}}$) and properties related to subgaussian
random variables (cf.~\cite[Section 2]{vershynin2018high}). For
notational convenience, we define the normalized version of $\setlowrank_{r}$
defined in (\ref{eq:defn-rank-r-set}), as follows:
\begin{equation}
\setlowrank_{r}^{\mathsf{norm}}\coloneqq\big\{\bX\in\R^{n_{1}\times n_{2}}:\rank(\bX)\le r,\|\bX\|_{\Frm}=1\big\}.\label{eq:defn-rank-r-unit-norm}
\end{equation}
Before moving on, we record two results that will be useful throughout
the proof. 
\begin{lem}
\label{lemma:helper} Let $\{\Ai\}_{i=1}^{m}$ be a set of random
matrices in $\R^{n_{1}\times n_{2}}$ with i.i.d.~$\Ncal(0,1)$ entries.
Denote $n\coloneqq\max\{n_{1},n_{2}\}$, and let $Z$ be a random
variable having the same distribution as $|\mathcal{N}(0,1)|$. For
all $t>0$ and $0<\epsilon<1$, with probability at least $1-(9/\epsilon)^{3nr}\exp(-c_{1}mt^{2}/(\tau+t))-C_{2}e^{-c_{2}n}$,
the following
\[
\frac{1}{m}\sum_{i=1}^{m}\ind\Big({\left|\left\langle \bm{A}_{i},\bm{X}\right\rangle \right|\leq\tau}\Big)\leq\mathbb{P}\left(Z\leq1.01\tau\right)+t+\frac{200\epsilon}{\tau}
\]
holds simultaneously for all $\bm{X}\in\setlowrank_{r}^{\mathsf{norm}}$,
provided that $m\geq Cnr\log m$. Here, $c_{1},C_{2},c_{2}>0$ are
universal constants, and $C>0$ is some sufficiently large constant.
\end{lem}
\begin{prop}
\label{prop:Gaussian_max_norm} Consider $\ai\iid\Ncal(\boldsymbol{0},\bI_{d}),1\le i\le m$.
There exist some universal constants $C,c>0$ such that with probability
at least $1-Ce^{-cd}$, we have
\[
\max_{1\le i\le m}\|\ai\|_{2}\lesssim\sqrt{d}+\sqrt{\log m}.
\]
\end{prop}
\begin{proof}
This result follows from \cite[Corollary 7.3.3]{vershynin2018high}
and the union bound.
\end{proof}

\subsection{\label{app:pf_trip}Proof of Lemma \ref{lem:TRIP}}

The first result on RIP has been established in the literature (e.g.~\cite[Theorem 2.3]{candes2011tight}),
and hence we only need to prove the second result on TRIP. We first
restrict to the case 
\[
m^{-100}\le\tau\le1.35;
\]
at the end of this subsection, we will prove TRIP for the case $0\le\tau<m^{-100}$
separately. By homogeneity, it is sufficient to show that 
\begin{equation}
\left|\frac{1}{m}\sum_{i=1}^{m}\langle\Ai,\bX\rangle\ind\big(\langle\Ai,\bX\rangle;\tau\big)\langle\Ai,\bZ\rangle-w(\tau)\langle\bX,\bZ\rangle\right|\le\delta\tau\label{eq:TRIP_normalized}
\end{equation}
holds simultaneously for all $(\bX,\bZ,\tau)\in\TTRIP$, where
\[
\TTRIP\coloneqq\big\{(\bX,\bZ,\tau):\bX,\bZ\in\setlowrank_{r}^{\mathsf{norm}},m^{-100}\le\tau\le1.35\big\}.
\]
The proof consists of two steps: (1) we replace the discontinuous
function $\ind$ by a Lipschitz continuous surrogate $\indtilde$
and establish a uniform concentration result for $\chi$; (2) we show
that the discrepancy incurred by replacing $\ind$ with $\indtilde$
is uniformly small. Our proof argument is conditioned on the high-probability
event that $\{\Ai\}_{i=1}^{m}$ satisfy $(2r,\delta$)-RIP. 

\paragraph{Step 1: replacing $\protect\ind$ with $\protect\indtilde$.}

Define an auxiliary function $\indtilde$ as follows: for all $a\in\R$
and $\tau>0$,
\begin{equation}
\indtilde(a;\tau)\coloneqq\begin{cases}
1, & |a|\le(1-\cchi)\tau;\\
0, & |a|\ge\tau;\\
\frac{\tau-|a|}{\cchi\tau}, & (1-\cchi)\tau<|a|<\tau.
\end{cases}\label{eq:defn-chi-a-tau}
\end{equation}
Here we set the parameter 
\begin{equation}
\cchi\coloneqq c_{0}\delta^{2}m^{-100}\label{eq:cchi_choice}
\end{equation}
for some sufficiently small constant $c_{0}>0$, whose rationale will
be made apparent in Step 2. It is easily seen that $\chi$ enjoys
the following properties:
\begin{itemize}
\item (Continuity) For any $\tau>0$, $\indtilde(\cdot;\tau)$ is piecewise
linear and $1/(\cchi\tau)$-Lipschitz continuous.
\item (Closeness to $\ind$) For any $\tau>0$ and $a\in\R$, $\indtilde(a;\tau)\le\ind(a;\tau)\le\indtilde(a;\tau/(1-\cchi))$.
\item (Homogeneity) For any $\tau>0$, $a\in\R$ and $c_{0}>0$, $\indtilde(a;\tau)=\indtilde(a/c_{0};\tau/c_{0})$.
\item If $0\le\epstau\le\cchi\tau$ and $\tau-\epstau\le\tau_{0}\le\tau$,
then $\|\indtilde(\cdot;\tau)-\indtilde(\cdot;\tau_{0})\|_{\infty}=\indtilde(\tau_{0};\tau)=(\tau-\tau_{0})/(\cchi\tau)\le\epstau/(\cchi\tau)$.
\item The function $f(a)\coloneqq a\cdot\indtilde(a;\tau)$ is $1/\cchi$-Lipschitz
continuous.
\end{itemize}
For notational convenience, define 
\begin{subequations}
\label{eq:def_Echi}
\begin{align}
\Emchi\xztau & \coloneqq\frac{1}{m}\sum_{i=1}^{m}\langle\Ai,\bX\rangle\indtilde\big(\langle\Ai,\bX\rangle;\tau\big)\langle\Ai,\bZ\rangle,\label{eq:Emchi}\\
\Echi(\bX,\bZ,\tau) & \coloneqq\E\big[\langle\bA_{i},\bX\rangle\indtilde(\langle\bA_{i},\bX\rangle;\tau)\langle\bA_{i},\bZ\rangle\big],\label{eq:Echi}
\end{align}
\end{subequations}
where the expectation is taken w.r.t.~$\{\bm{A}_{i}\}$ while assuming
that $(\bm{X},\bm{Z},\tau)$ are fixed. With these preparations in
place, we set out to prove that: if $m\ge C_{0}nr\delta^{-2}\log m$,
then with probability at least $1-Ce^{-cn}$, 
\begin{equation}
\big|\Emchi\xztau-\Echi\xztau\big|\le\delta\tau/2\label{eq:claim-step1-lemma1}
\end{equation}
 holds simultaneously for all $(\bX,\bZ,\tau)\in\TTRIP$; here $C_{0}>0$
is some sufficiently large constant, and $C,c>0$ are some universal
constants. 

First, consider any fixed point $(\bX,\bZ,\tau)\in\TTRIP$. Note that
$|\langle\Ai,\bX\rangle\indtilde\big(\langle\Ai,\bX\rangle;\tau\big)|\le\tau$
is bounded, and that the subgaussian norm of $\langle\Ai,\bZ\rangle$
obeys $\|\langle\Ai,\bZ\rangle\|_{\psi_{2}}=\|\mathcal{N}(0,1)\|_{\psi_{2}}\lesssim1.$
As a result, 
\[
\big\|\langle\Ai,\bX\rangle\indtilde\big(\langle\Ai,\bX\rangle;\tau\big)\langle\Ai,\bZ\rangle-\Echi(\bX,\bZ,\tau)\big\|_{\psi_{2}}\lesssim\tau.
\]
Invoking \cite[Theorem 2.6.2]{vershynin2018high} tells us that for
all $t\ge0$,
\[
\Pr\Big(\big|\Emchi\xztau-\Echi(\bX,\bZ,\tau)\big|\ge t\tau\Big)\le2\exp\left(-c_{1}mt^{2}\right)
\]
holds for some constant $c_{1}>0$. Next, we construct an $\epsilon$-net
to cover $\TTRIP$. In view of \cite[Lemma 3.1]{candes2011tight},
the set $\setlowrank_{r}^{\mathsf{norm}}$ defined in (\ref{eq:defn-rank-r-unit-norm})
has an $\epsilon$-net (in terms of $\|\cdot\|_{\mathrm{F}}$ distance)
of cardinality at most $(9/\epsilon)^{3nr}$. In addition, we can
cover the interval $[m^{-100},1.35]$ with precision $\epsilon_{\tau}$
using no more than $2/\epsilon_{\tau}$ equidistant points. Putting
all this together, we can construct a set $\MTRIP\subseteq\setlowrank_{r}^{\mathsf{norm}}\times\setlowrank_{r}^{\mathsf{norm}}\times[0,1.35]$
of cardinality at most $(9/\epsilon)^{6nr}(2/\epstau)$ such that:
for any $(\bX,\bZ,\tau)\in\TTRIP$, there exists some point $(\bX_{0},\bZ_{0},\tau_{0})\in\MTRIP$
obeying
\begin{equation}
\|\bX-\bX_{0}\|_{\Frm}\le\epsilon,\quad\|\bZ-\bZ_{0}\|_{\Frm}\le\epsilon,\quad\text{and}\quad\tau-\epstau\le\tau_{0}\le\tau.\label{eq:cover_TTRIP}
\end{equation}
The union bound then implies that with probability at least $1-2\exp(-c_{1}mt^{2})(9/\epsilon)^{6nr}(2/\epstau)$,
one has 
\begin{equation}
\big|\Emchi-\Echi(\bX,\bZ,\tau)\big|\le t\tau,\qquad\text{for all }\xztau\in\MTRIP.\label{eq:error-EmX-1}
\end{equation}
In what follows, we shall choose
\begin{equation}
t=\frac{1}{4}\delta\quad\text{and}\quad m\ge C_{3}\frac{1}{\delta^{2}}\left(nr\log\frac{9}{\epsilon}+\log\frac{2}{\epstau}\right)\label{eq:trip_m_tmp}
\end{equation}
so as to achieve a uniformly small error $t\tau=\delta\tau/4$ in
\eqref{eq:error-EmX-1} with probability at least $1-2\exp(-c_{3}m\delta^{2})$
for some universal constant $c_{3}>0$. 

Now, for any $\xztau\in\TTRIP$, let $\xztauzero\in\MTRIP$ be the
point satisfying (\ref{eq:cover_TTRIP}). Then we have
\begin{subequations}
\label{eq:chi_cover_error}
\begin{align}
 & \big|\Emchi\xztau-\Echi(\bX,\bZ,\tau)\big|\le\underset{{\rm (A)}}{\underbrace{\big|\Emchi\xztauzero-\Echi\xztauzero\big|}}\label{eq:chi_cover_error_A_B}\\
 & \qquad+\underset{{\rm (B)}}{\underbrace{\big|\Emchi\xztau-\Emchi\xztauzero\big|}}+\underset{{\rm (C)}}{\underbrace{\big|\Echi(\bX,\bZ,\tau)-\Echi\xztauzero\big|}}.\label{eq:chi_cover_error_C}
\end{align}
\end{subequations}
Here, (A) is already bounded by $\delta\tau/4$ by construction. We
can control (B) via the following decomposition:
\begin{align*}
{\rm (B)} & \le\underset{{\rm (B.1)}}{\underbrace{\bigg|\frac{1}{m}\sum_{i=1}^{m}\langle\Ai,\bX\rangle\indtilde(\langle\Ai,\bX\rangle;\tau)\langle\Ai,\bZ-\bZ_{0}\rangle\bigg|}}\\
 & \quad+\underset{{\rm (B.2)}}{\underbrace{\bigg|\frac{1}{m}\sum_{i=1}^{m}\Big(\langle\Ai,\bX\rangle\indtilde(\langle\Ai,\bX\rangle;\tau)-\langle\Ai,\bX_{0}\rangle\indtilde(\langle\Ai,\bX_{0}\rangle;\tau)\Big)\langle\Ai,\bZ_{0}\rangle\bigg|}}\\
 & \quad+\underset{{\rm (B.3)}}{\underbrace{\bigg|\frac{1}{m}\sum_{i=1}^{m}\langle\Ai,\bX_{0}\rangle\Big(\indtilde(\langle\Ai,\bX_{0}\rangle;\tau)-\indtilde(\langle\Ai,\bX_{0}\rangle;\tau_{0})\Big)\langle\Ai,\bZ_{0}\rangle\bigg|}}.
\end{align*}
In light of the $(2r,\delta)$-RIP, the aforementioned properties
of $\chi$, and the Cauchy-Schwarz inequality, we have
\begin{align*}
{\rm (B.1)} & \overset{{\rm (i)}}{\le}\tau\frac{1}{m}\sum_{i=1}^{m}\big|\langle\Ai,\bZ-\bZ_{0}\rangle\big|\le\tau\sqrt{\frac{1}{m}\sum_{i=1}^{m}\langle\Ai,\bZ-\bZ_{0}\rangle^{2}}\lesssim\tau\epsilon,\\
{\rm (B.2)} & \le\frac{1}{m}\sum_{i=1}^{m}\big|\langle\Ai,\bX\rangle\indtilde(\langle\Ai,\bX\rangle;\tau)-\langle\Ai,\bX_{0}\rangle\indtilde(\langle\Ai,\bX_{0}\rangle;\tau)\big|\cdot\big|\langle\Ai,\bZ_{0}\rangle\big|\\
 & \overset{{\rm (ii)}}{\le}\frac{1}{\cchi}\frac{1}{m}\sum_{i=1}^{m}\big|\langle\Ai,\bX-\bX_{0}\rangle\big|\cdot\big|\langle\Ai,\bZ_{0}\rangle\big|\le\frac{1}{\cchi}\sqrt{\frac{1}{m}\sum_{i=1}^{m}\langle\Ai,\bX-\bX_{0}\rangle^{2}}\sqrt{\frac{1}{m}\sum_{i=1}^{m}\langle\Ai,\bZ_{0}\rangle^{2}}\lesssim\frac{\epsilon}{\cchi},\\
{\rm (B.3)} & \le\big\|\indtilde(\cdot;\tau)-\indtilde(\cdot;\tau_{0})\big\|_{\infty}\frac{1}{m}\sum_{i=1}^{m}\big|\langle\Ai,\bX_{0}\rangle\big|\cdot\big|\langle\Ai,\bZ_{0}\rangle\big|\\
 & \overset{{\rm (iii)}}{\lesssim}\frac{\epstau}{\cchi\tau}\sqrt{\frac{1}{m}\sum_{i=1}^{m}\langle\Ai,\bX_{0}\rangle^{2}}\sqrt{\frac{1}{m}\sum_{i=1}^{m}\langle\Ai,\bZ_{0}\rangle^{2}}\lesssim\frac{\epstau}{\cchi\tau}.
\end{align*}
Here, (i) uses $|\langle\Ai,\bX\rangle\indtilde(\langle\Ai,\bX\rangle;\tau)|\le\tau$,
(ii) follows from the property that the function $f(a)=a\cdot\indtilde(a;\tau)$
is $1/\cchi$-Lipschitz continuous, whereas (iii) is due to the property
$\|\indtilde(\cdot;\tau)-\indtilde(\cdot;\tau_{0})\|_{\infty}\le\epstau/(\cchi\tau)$.
The term (C) can be controlled by the same decomposition and thus
enjoys the same upper bound. Putting these back into (\ref{eq:chi_cover_error}),
we have for all $\xztau\in\TTRIP$,
\[
\left|\Emchi\xztau-\Echi(\bX,\bZ,\tau)\right|\le\frac{1}{4}\delta\tau+C_{3}\left(\tau\epsilon+\frac{\epsilon}{\cchi}+\frac{\epstau}{\cchi\tau}\right)
\]
for some universal constant $C_{3}>0$. Recalling that $\tau\ge m^{-100}$,
and choosing $\epsilon\le c_{4}\delta\cchi m^{-100}$ and $\epsilon_{\tau}\le c_{5}\delta\cchi m^{-200}$
for some sufficiently small constants $c_{4},c_{5}>0$, we have
\[
\big|\Emchi\xztau-\Echi(\bX,\bZ,\tau)\big|\le\delta\tau/2.
\]
Plugging our choice of $\epsilon$ and $\epstau$ into (\ref{eq:trip_m_tmp})
immediately establishes the claim \eqref{eq:claim-step1-lemma1} of
this step.

\paragraph{Step 2: controlling the errors incurred by using the surrogate $\protect\indtilde$.}

Similar to (\ref{eq:def_Echi}), we define 
\begin{align*}
\Emind\xztau & \coloneqq\frac{1}{m}\sum_{i=1}^{m}\langle\Ai,\bX\rangle\ind(\langle\Ai,\bX\rangle;\tau)\langle\Ai,\bZ\rangle,\\
\Eind(\bX,\bZ,\tau) & \coloneqq\E\big[\langle\bA_{i},\bX\rangle\ind(\langle\bA_{i},\bX\rangle;\tau)\langle\bA_{i},\bZ\rangle\big]=w(\tau)\langle\bX,\bZ\rangle,
\end{align*}
where the expectation is taken assuming independence between $\bA_{i}$
and $(\bX,\bZ,\tau)$. In this step, we aim to show that: if  $m\ge C_{0}nr\delta^{-2}\log m$,
then with probability at least $1-Ce^{-cn}$, 
\begin{equation}
\big|\Emind\xztau-\Eind\xztau\big|\le\big|\Emchi\xztau-\Echi\xztau\big|+\delta\tau/2\label{eq:claim-step2-lemma1}
\end{equation}
 holds simultaneously for all $(\bX,\bZ,\tau)\in\TTRIP$. If this
were true, then combining this with \eqref{eq:claim-step1-lemma1}
would immediately conclude the proof of Lemma~\ref{lem:TRIP}. 

Towards establishing \eqref{eq:claim-step2-lemma1}, we start with
the following decomposition:
\begin{align}
\left|\Emind\xztau-E\xztau\right| & \le\big|\Emchi\xztau-\Echi\xztau\big|+\underset{{\rm (A)}}{\underbrace{\big|\Eind\xztau-\Echi\xztau\big|}}\nonumber \\
 & \quad+\underset{{\rm (B)}}{\underbrace{\big|\Emind\xztau-\Emchi\xztau\big|}},\label{eq:Em-E-diff-AB-123}
\end{align}
where we abuse the notation (A) and (B). In the sequel, we shall control
(A) and (B) separately. 
\begin{itemize}
\item Regarding (A), the Cauchy-Schwarz inequality gives
\begin{align*}
{\rm (A)} & =\Big|\E\big[\langle\bA_{i},\bX\rangle\big(\ind(\langle\bA_{i},\bX\rangle;\tau)-\indtilde(\langle\bA_{i},\bX\rangle;\tau)\big)\langle\bA_{i},\bZ\rangle\big]\Big|\\
 & \le\sqrt{\E\Big[\big(\ind(\langle\bA_{i},\bX\rangle;\tau)-\indtilde(\langle\bA_{i},\bX\rangle;\tau)\big)^{2}\Big]}\sqrt{\E\Big[\big(\langle\bA_{i},\bX\rangle\langle\bA_{i},\bZ\rangle\big)^{2}\Big]}\lesssim\sqrt{\cchi\tau}.
\end{align*}
The last inequality holds since $|\ind(\langle\bA_{i},\bX\rangle;\tau)-\indtilde(\langle\bA_{i},\bX\rangle;\tau)|\in[0,1]$
is non-zero only for $|\langle\bA_{i},\bX\rangle|$ on an interval
of length $\cchi\tau$, over which the probability density function
of $\langle\bA_{i},\bX\rangle\sim\Ncal(0,1)$ is upper bounded by
some constant. By our choice of $\cchi$ in (\ref{eq:cchi_choice}),
we have ${\rm (A)}\le\delta\tau/4$. 
\item We then move on to (B). For notational convenience, given any $\tau>0$,
we let 
\begin{equation}
\tau'=\tau'(\tau)\coloneqq\frac{\tau}{1-\cchi},\label{eq:def_tauprime}
\end{equation}
which clearly satisfies $\indtilde(a;\tau)\le\ind(a;\tau)\le\indtilde(a;\tau')$.
In addition, defining 
\[
\ind_{-}(a)\coloneqq\ind(a<0),\quad\ind_{+}(a)\coloneqq\ind(a\ge0),\quad a\in\R,
\]
we can deduce that
\begin{align*}
\Emind\xztau & \le\Emchi\xztau+\frac{1}{m}\sum_{i=1}^{m}\big(\indtilde(\langle\Ai,\bX\rangle;\tau')-\indtilde(\langle\Ai,\bX\rangle;\tau)\big)\langle\Ai,\bX\rangle\langle\Ai,\bZ\rangle\ind_{+}(\langle\Ai,\bX\rangle\langle\Ai,\bZ\rangle),\\
\Emind\xztau & \ge\Emchi\xztau+\frac{1}{m}\sum_{i=1}^{m}\big(\indtilde(\langle\Ai,\bX\rangle;\tau')-\indtilde(\langle\Ai,\bX\rangle;\tau)\big)\langle\Ai,\bX\rangle\langle\Ai,\bZ\rangle\ind_{-}(\langle\Ai,\bX\rangle\langle\Ai,\bZ\rangle).
\end{align*}
As a consequence,
\begin{align*}
{\rm (B)}\le\max\bigg\{ & \underset{{\rm (C)}}{\underbrace{\Big|\frac{1}{m}\sum_{i=1}^{m}\big(\indtilde(\langle\Ai,\bX\rangle;\tau')-\indtilde(\langle\Ai,\bX\rangle;\tau)\big)\langle\Ai,\bX\rangle\langle\Ai,\bZ\rangle\ind_{+}(\langle\Ai,\bX\rangle\langle\Ai,\bZ\rangle)\Big|}},\\
 & \Big|\frac{1}{m}\sum_{i=1}^{m}\big(\indtilde(\langle\Ai,\bX\rangle;\tau')-\indtilde(\langle\Ai,\bX\rangle;\tau)\big)\langle\Ai,\bX\rangle\langle\Ai,\bZ\rangle\ind_{-}(\langle\Ai,\bX\rangle\langle\Ai,\bZ\rangle)\Big|\bigg\}.
\end{align*}
Next, we demonstrate how to analyze the first term (C) above; the
analysis for the other term is essentially the same. For notational
simplicity, define 
\begin{align*}
F_{m}^{+}\xztau & \coloneqq\frac{1}{m}\sum_{i=1}^{m}\indtilde(\langle\Ai,\bX\rangle;\tau)\langle\Ai,\bX\rangle\langle\Ai,\bZ\rangle\ind_{+}(\langle\Ai,\bX\rangle\langle\Ai,\bZ\rangle),\\
E^{+}\xztau & \coloneqq\E\Big[\big(\indtilde(\langle\bA_{i},\bX\rangle;\tau')-\indtilde(\langle\bA_{i},\bX\rangle;\tau)\big)\langle\bA_{i},\bX\rangle\langle\bA_{i},\bZ\rangle\ind_{+}(\langle\bA_{i},\bX\rangle\langle\bA_{i},\bZ\rangle)\Big],
\end{align*}
where the expectation is again taken assuming that $\bA_{i}$ is independent
of $\bm{X}$, $\bm{Z}$ and $\tau$. Then we have
\begin{align*}
{\rm (C)} & =\big|F_{m}^{+}(\bX,\bZ,\tau')-F_{m}^{+}\xztau\big|\le\big|E^{+}\xztau\big|+\big|F_{m}^{+}(\bX,\bZ,\tau')-F_{m}^{+}\xztau-E^{+}\xztau\big|.
\end{align*}
Regarding the first term on the right-hand side, we follow an argument
similar to our previous analysis for (A) to obtain
\begin{align*}
\big|E^{+}\xztau\big| & \le\sqrt{\E\Big[\big(\indtilde(\langle\bA_{i},\bX\rangle;\tau')-\indtilde(\langle\bA_{i},\bX\rangle;\tau)\big)^{2}\Big]}\sqrt{\E\Big[\big(\langle\bA_{i},\bX\rangle\langle\bA_{i},\bZ\rangle\big)^{2}\Big]}\lesssim\sqrt{\cchi\tau}\le c_{2}\delta\tau
\end{align*}
for some sufficiently small constant $0<c_{2}<1/8$. Thus, it remains
to show that 
\begin{equation}
\big|F_{m}^{+}(\bX,\bZ,\tau')-F_{m}^{+}\xztau-E^{+}\xztau\big|\le\frac{1}{8}\delta\tau\label{eq:concentration_step2}
\end{equation}
holds simultaneously for all $\xztau\in\TTRIP$. Note that by definition,
$F_{m}^{+}\xztau$ is the empirical average of some Lipschitz continuous
function (in particular, $\langle\Ai,\bX\rangle\langle\Ai,\bZ\rangle\ind_{+}(\langle\Ai,\bX\rangle\langle\Ai,\bZ\rangle)$
is $1$-Lipschitz continuous over $\langle\Ai,\bX\rangle\langle\Ai,\bZ\rangle$).
Therefore, we can prove (\ref{eq:concentration_step2}) by a standard
covering argument similar to that in Step 1; we omit the details for
brevity. Putting the above bounds together, we establish that ${\rm (B)}\le\delta\tau/4$.
\item Combining the above bounds ${\rm (A)}\le\delta\tau/4$ and ${\rm (B)}\le\delta\tau/4$
with \eqref{eq:Em-E-diff-AB-123}, we finish the proof of \eqref{eq:claim-step2-lemma1}.
\end{itemize}

\paragraph{Proof for the case $0\le\tau<m^{-100}$.}

It remains to prove that (\ref{eq:TRIP_normalized}) holds simultaneously
for all $\bX,\bZ\in\setlowrank_{r}^{\mathsf{norm}}$ (cf.~\eqref{eq:defn-rank-r-unit-norm})
and all $0\le\tau<m^{-100}$. We start with the following decomposition:
\begin{equation}
\bigg|\frac{1}{m}\sum_{i=1}^{m}\langle\Ai,\bX\rangle\ind\big(\langle\Ai,\bX\rangle;\tau\big)\langle\Ai,\bZ\rangle-w(\tau)\langle\bX,\bZ\rangle\bigg|\le\bigg|\frac{1}{m}\sum_{i=1}^{m}\langle\Ai,\bX\rangle\ind\big(\langle\Ai,\bX\rangle;\tau\big)\langle\Ai,\bZ\rangle\bigg|+\big|w(\tau)\langle\bX,\bZ\rangle\big|.\label{eq:tmp_trip_decomp}
\end{equation}
The second term on the right-hand side of (\ref{eq:tmp_trip_decomp})
can be bounded by
\[
\big|w(\tau)\langle\bX,\bZ\rangle\big|\le w(\tau)\overset{{\rm (i)}}{\le}\tau^{3}\le m^{-200}\tau\overset{{\rm (ii)}}{\le}0.1\delta\tau,
\]
where (i) can be seen from the definition of $w(\cdot)$ in (\ref{eq:def_w}),
and (ii) relies on the observation that our assumption $m\ge C_{0}nr\delta^{-2}\log m$
implies $\delta\gtrsim m^{-1/2}$.

It thus remains to show that the first term on the right-hand side
of (\ref{eq:tmp_trip_decomp}) is bounded by $0.9\delta\tau$. In
view of $(2r,\delta)$-RIP, the Cauchy-Schwarz inequality, and the
observation that $|\langle\Ai,\bX\rangle\ind\big(\langle\Ai,\bX\rangle;\tau\big)|\le\tau$,
we have
\begin{align}
\bigg|\frac{1}{m}\sum_{i=1}^{m}\langle\Ai,\bX\rangle\ind\big(\langle\Ai,\bX\rangle;\tau\big)\langle\Ai,\bZ\rangle\bigg| & \le\sqrt{\frac{1}{m}\sum_{i=1}^{m}\langle\Ai,\bX\rangle^{2}\ind\big(\langle\Ai,\bX\rangle;\tau\big)}\cdot\sqrt{\frac{1}{m}\sum_{i=1}^{m}\langle\Ai,\bZ\rangle^{2}}\nonumber \\
 & \le2\sqrt{\frac{1}{m}\sum_{i=1}^{m}\tau^{2}\ind\big(\langle\Ai,\bX\rangle;\tau\big)}\le2\tau\sqrt{\frac{1}{m}\sum_{i=1}^{m}\ind\big(\langle\Ai,\bX\rangle;m^{-100}\big)},\label{eq:turn_into_bernoulli}
\end{align}
 where the last inequality uses the assumption that $\tau<m^{-100}$.
We can invoke Lemma \ref{lemma:helper} with $t=0.01\delta^{2}$ and
$\epsilon=m^{-200}$ to obtain that with probability at least $1-Ce^{-cn}$
(for some constants $c,C>0$),
\[
\frac{1}{m}\sum_{i=1}^{m}\ind\big(\langle\Ai,\bX\rangle;m^{-100}\big)\le\Pr(Z_{0}\le1.01m^{-100})+t+\frac{200\epsilon}{m^{-100}}\le2t=0.02\delta^{2}
\]
holds simultaneously for all $\bX\in\setlowrank_{r}^{\mathsf{norm}}$,
provided that $m\ge C_{0}nr\delta^{-2}\log m$; here, $Z_{0}$ denotes
a random variable having the same distribution as $|\Ncal(0,1)|$.
Plugging this into (\ref{eq:turn_into_bernoulli}) confirms that the
first term on the right-hand side of (\ref{eq:tmp_trip_decomp}) is
bounded by $0.9\delta\tau$, thus concluding the proof for the case
with $0\le\tau<m^{-100}$.

\subsection{Proof of Lemma \ref{lem:quantile} }

It is easy to check that $\quantalpha(\Dcal)/\|\bX_{1}\|_{\Frm}$
is invariant under a global scaling of $\left\{ \bX_{1},\dots,\bX_{K}\right\} $.
Therefore, it suffices to consider a normalized version of $\Tquant$
(\ref{eq:set_Tquant}) defined as follows:
\begin{equation}
\Tquantnorm\coloneqq\Tquant\cap\big\{(\bX_{1},\dots,\bX_{k}):\|\bX_{1}\|_{\Frm}=1\big\}.\label{eq:def_Tnorm}
\end{equation}
In what follows, we shall treat the upper bound and the lower bound
separately and invoke a standard covering argument to prove Lemma
\ref{lem:quantile} with $\Tquant$ replaced by $\Tquantnorm$. Throughout
this proof, we denote by $Z$ a random variable following the distribution
of $|\Ncal(0,1)|$. 

\paragraph{Step 1: upper bounding $Q_{\alpha}(\protect\Dcal)$.}

Since $\alpha\leq0.8p_{1}$, we have
\[
\quantalpha(\Dcal)\le Q_{\alpha/\ponehat}(\Dcal_{1})\le Q_{0.8}(\Dcal_{1}).
\]
Now it suffices to upper bound $Q_{0.8}(\Dcal_{1})$, which is only
related to $\bX_{1}\in\setlowrank_{r}^{\mathsf{norm}}$. Consider
any fixed point $\bX_{1}\in\setlowrank_{r}^{\mathsf{norm}}$. Note
that the set $\Dcal_{1}$ defined in~(\ref{eq:def_D}) contains i.i.d.~samples
having the same distribution as $Z$. This combined with the concentration
of empirical quantiles \cite[Section 2.3.2]{serfling2009approximation}
gives
\begin{equation}
\Pr\big(Q_{0.8}(\Dcal_{1})\ge Q_{0.8}(Z)+0.01\big)\le\exp\left(-c_{2}N_{1}\right)\label{eq:quant_upper_prob}
\end{equation}
for some universal constant $c_{2}>0$. Here, $N_{1}\coloneqq|\Omegaonestar|\asymp N/K$
by the assumption of the well-balancedness property~(\ref{eq:balanced_pk}).
Next, we construct an $\epsilon$-net of $\setlowrank_{r}^{\mathsf{norm}}$
--- denoted by $\mathcal{M}$ --- whose cardinality is at most $(9/\epsilon)^{3nr}$
(according to~\cite[Lemma~3.1]{candes2011tight}). Taking the union
bound over $\mathcal{M}$ and assuming that
\[
N_{1}\ge C_{0}nr\log\frac{9}{\epsilon}
\]
for some sufficiently large constant $C_{0}>0$, we have with probability
at least $1-Ce^{-cn}$, for all $\bX_{1}\in\mathcal{M}$, the dataset
$\Dcal_{1}$ defined in~(\ref{eq:def_D}) satisfies $Q_{0.8}(\Dcal_{1})\le Q_{0.8}(Z)+0.01$.
Finally, consider an arbitrary $\bX_{1}\in\setlowrank_{r}^{\mathsf{norm}}$,
and let $\bX_{1}^{0}$ be the point in $\mathcal{M}$ such that $\|\bm{X}_{1}^{0}-\bm{X}_{1}\|_{\mathrm{F}}\leq\epsilon$.
Denote by $\Dcal_{1}^{0}$ the dataset generated by $\bX_{1}^{0}$
analogous to (\ref{eq:def_D}). Then we have 
\[
\big|Q_{0.8}(\Dcal_{1})-Q_{0.8}(\Dcal_{1}^{0})\big|\le\max_{i\in\Omegaonestar}\big|\langle\Ai,\bX_{1}\rangle-\langle\Ai,\bX_{1}^{0}\rangle\big|\le\epsilon\max_{i\in\Omegaonestar}\|\Ai\|_{\Frm}\lesssim\epsilon\left(n+\sqrt{\log N_{1}}\right),
\]
where the last inequality holds with probability at least $1-Ce^{-cn}$,
according to Proposition~\ref{prop:Gaussian_max_norm}. Setting $\epsilon=N_{1}^{-10}$,
we further have $|Q_{0.8}(\Dcal_{1})-Q_{0.8}(\Dcal_{1}^{0})|\lesssim N_{1}^{-9}\le0.01$,
as long as $N_{1}$ is sufficiently large. In addition, it can be
verified numerically that $Q_{0.8}(Z)<1.30$. These together imply
that for any $(\bX_{1},\dots,\bX_{K})\in\Tquantnorm$, we have
\[
\quantalpha(\Dcal)\le Q_{0.8}(\Dcal_{1})\le Q_{0.8}(Z)+0.02\le1.35,
\]
which gives rise to the upper bound in Lemma \ref{lem:quantile}.

\paragraph{Step 2: lower bounding $Q_{\alpha}(\protect\Dcal)$.}

For notational convenience, we denote 
\begin{equation}
q\coloneqq\frac{0.7\alpha}{\ponehat}\in[0.42,0.56],\quad\text{and}\quad B_{N}\coloneqq\frac{1}{\Nall}\sum_{k=1}^{K}\sum_{i\in\Omegakstar}\ind\Big(\big|\langle\Ai,\bX_{k}\rangle\big|\le\frac{Q_{q}(Z)}{1.01}\Big).\label{eq:quant_lower_bound}
\end{equation}
Clearly, by the definition of $B_{N}$, one has 
\[
\Pr\bigg(\quantalpha(\Dcal)<\frac{Q_{q}(Z)}{1.01}\bigg)\le\Pr\big(B_{N}>\alpha\big),
\]
where it can be verified numerically that $Q_{q}(Z)/1.01\ge0.54$.
Therefore, it suffices to upper bound the probability $\Pr\big(B_{N}>\alpha\big)$.
To accomplish this, we first upper bound $B_{N}$ as follows:
\begin{align}
B_{N} & =\frac{1}{\Nall}\sum_{k=1}^{K}\sum_{i\in\Omegakstar}\ind\Bigg(\bigg|\Big\langle\Ai,\frac{\bX_{k}}{\|\bX_{k}\|_{\Frm}}\Big\rangle\bigg|\le\frac{Q_{q}(Z)}{1.01\|\bX_{k}\|_{\Frm}}\Bigg)\nonumber \\
 & \le\frac{1}{\Nall}\sum_{i\in\Omegaonestar}\ind\Big(\big|\langle\Ai,\bX_{1}\rangle\big|\le\frac{Q_{q}(Z)}{1.01}\Big)+\frac{1}{\Nall}\sum_{k\neq1}\sum_{i\in\Omegakstar}\ind\Bigg(\bigg|\Big\langle\Ai,\frac{\bX_{k}}{\|\bX_{k}\|_{\Frm}}\Big\rangle\bigg|\le\frac{c_{0}Q_{q}(Z)}{1.01K}\Bigg).\label{eq:BN_upper_bound}
\end{align}
Here, the last line follows from the assumption that $1=\|\bX_{1}\|_{\Frm}\le(c_{0}/K)\min_{k\neq1}\|\bX_{k}\|_{\Frm}$;
see the definition of $\Tquantnorm$ in (\ref{eq:def_Tnorm}). Note
that $\bX_{1}\in\setlowrank_{r}^{\mathsf{norm}}$, and for all $k\neq1$,
we also have $\bX_{k}/\|\bX_{k}\|_{\Frm}\in\setlowrank_{r}^{\mathsf{norm}}$.
Therefore, we can invoke Lemma \ref{lemma:helper} with $m=N_{1}=|\Omegaonestar|,\tau=Q_{q}(Z)/1.01,t=0.15\alpha$
and $\epsilon=N_{1}^{-10}$ to obtain that: with probability at least
$1-Ce^{-cn}$ (provided that $m\ge C_{0}nrK^{2}\log m$), the following
holds simultaneously for all $\bX_{1}\in\setlowrank_{r}^{\mathsf{norm}}$:
\begin{align*}
\frac{1}{\Nall_{1}}\sum_{i\in\Omegaonestar}\ind\Big(\big|\langle\Ai,\bX_{1}\rangle\big|\le\frac{Q_{q}(Z)}{1.01}\Big) & \leq\mathbb{P}\big(Z\leq Q_{q}(Z)\big)+t+\frac{200\epsilon}{\tau}=q+0.15\alpha+\frac{202N_{1}^{-10}}{Q_{q}(Z)}.
\end{align*}
Similarly, for all $k\neq1$, one can apply Lemma \ref{lemma:helper}
with $m=N_{k}\coloneqq|\Omegakstar|,\tau=c_{0}Q_{q}(Z)/(1.01K),t=0.15\alpha$
and $\epsilon=N_{k}^{-10}$ to show that: with probability at least
$1-Ce^{-cn}$ (provided $m\ge C_{0}nrK^{2}\log m$), the following
holds simultaneously for all $\bX_{k}/\|\bX_{k}\|_{\Frm}\in\setlowrank_{r}^{\mathsf{norm}}$:
\begin{align*}
\frac{1}{N_{k}}\sum_{i\in\Omegakstar}\ind\Bigg(\bigg|\Big\langle\Ai,\frac{\bX_{k}}{\|\bX_{k}\|_{\Frm}}\Big\rangle\bigg|\le\frac{c_{0}Q_{q}(Z)}{1.01K}\Bigg) & \leq\mathbb{P}\left(Z\leq\frac{c_{0}Q_{q}(Z)}{K}\right)+t+\frac{200\epsilon}{\tau}\le\frac{c_{0}Q_{q}(Z)}{K}+0.15\alpha+\frac{202KN_{k}^{-10}}{c_{0}Q_{q}(Z)},
\end{align*}
where the last inequality relies on the property of $Z$. Combine
the above two bounds with (\ref{eq:BN_upper_bound}) to reach 
\begin{align*}
B_{N} & \leq p_{1}\bigg(q+0.15\alpha+\frac{202N_{1}^{-10}}{Q_{q}(Z)}\bigg)+\sum_{k\neq1}p_{k}\bigg(\frac{c_{0}Q_{q}(Z)}{K}+0.15\alpha+\frac{202KN_{k}^{-10}}{c_{0}Q_{q}(Z)}\bigg)\\
 & \le p_{1}q+\frac{c_{0}Q_{q}(Z)}{K}+0.15\alpha+p_{1}\frac{202N_{1}^{-10}}{Q_{q}(Z)}+\sum_{k\neq1}p_{k}\frac{202KN_{k}^{-10}}{c_{0}Q_{q}(Z)}.
\end{align*}
Recall that $p_{1}q=0.7\alpha$, $\alpha\asymp p_{1}\asymp1/K$, and
observe that $p_{1}\frac{202N_{1}^{-10}}{Q_{q}(Z)}+\sum_{k\neq1}p_{k}\frac{202KN_{k}^{-10}}{c_{0}Q_{q}(Z)}\le0.05\alpha$
as long as $N_{k}\gtrsim K$ for all $k$. Putting these together
guarantees that $B_{N}\le\alpha$ as desired, which further implies
\[
Q_{\alpha}(\mathcal{D})\geq Q_{q}(Z)/1.01\geq0.54.
\]
Combining this lower bound with the upper bound in Step 1 completes
our proof of Lemma~\ref{lem:quantile}.

\paragraph{}

\subsection{Proof of Lemma \ref{lemma:helper}}

Throughout the proof, we assume that the ensemble $\{\bm{A}_{i}\}$
obeys $(2r,1/4)$-RIP. In view of Lemma~\ref{lem:TRIP}, this happens
with probability at least $1-C_{2}e^{-c_{2}n}$ for some constants
$c_{2},C_{2}>0$, as long as $m\ge Cnr$. Recall the definition of
$\chi$ from Appendix \ref{app:pf_trip}, and set the parameter as
$\cchi=0.01/1.01$. One then has
\[
\frac{1}{m}\sum_{i=1}^{m}\ind\left({\left|\left\langle \bm{A}_{i},\bm{X}\right\rangle \right|\leq\tau}\right)\leq\frac{1}{m}\sum_{i=1}^{m}\chi\left(\left\langle \bm{A}_{i},\bm{X}\right\rangle ;1.01\tau\right)
\]
In the sequel, we invoke the standard covering argument to upper bound
$\frac{1}{m}\sum_{i=1}^{m}\chi\left(\left\langle \bm{A}_{i},\bm{X}\right\rangle ;1.01\tau\right)$. 

First, consider a fixed $\bm{X}\in\setlowrank_{r}^{\mathsf{norm}}$
independent of $\{\bm{A}_{i}\}$. In this case we can bound the expectation
as 
\begin{align*}
\mathbb{E}\left[\frac{1}{m}\sum_{i=1}^{m}\chi\left(\left\langle \bm{A}_{i},\bm{X}\right\rangle ;1.01\tau\right)\right] & \leq\mathbb{E}\left[\frac{1}{m}\sum_{i=1}^{m}\ind\left({\left|\left\langle \bm{A}_{i},\bm{X}\right\rangle \right|\leq1.01\tau}\right)\right]=\mathbb{P}\left(Z\leq1.01\tau\right),
\end{align*}
where we recall that $Z$ follows the same distribution as $|\mathcal{N}(0,1)|$.
In addition, note that $\frac{1}{m}\sum_{i=1}^{m}\chi\left(\left\langle \bm{A}_{i},\bm{X}\right\rangle ;1.01\tau\right)$
is the empirical average of $m$ independent random variables, each
lying within $[0,1]$ and having variance bounded by $2\tau$. Therefore,
for all $t\ge0$, one sees from Bernstein's inequality \cite[Theorem 2.8.4]{vershynin2018high}
that 
\[
\mathbb{P}\left(\frac{1}{m}\sum_{i=1}^{m}\chi\left(\left\langle \bm{A}_{i},\bm{X}\right\rangle ;1.01\tau\right)\geq\mathbb{E}\left[\frac{1}{m}\sum_{i=1}^{m}\chi\left(\left\langle \bm{A}_{i},\bm{X}\right\rangle ;1.01\tau\right)\right]+t\right)\leq\exp\left(-\frac{c_{1}mt^{2}}{\tau+t}\right),
\]
where $c_{0},c_{1}>0$ are some universal constants. Let $\mathcal{M}\subseteq\setlowrank_{r}^{\mathsf{norm}}$
be an $\epsilon$-net of $\setlowrank_{r}^{\mathsf{norm}}$, whose
cardinality is at most $(9/\epsilon)^{3nr}$. The union bound reveals
that: with probability at least $1-(9/\epsilon)^{3nr}\exp(-c_{1}mt^{2}/(\tau+t))$,
one has
\[
\sup_{\bm{X}\in\mathcal{M}}\quad\frac{1}{m}\sum_{i=1}^{m}\chi\left(\left\langle \bm{A}_{i},\bm{X}\right\rangle ;1.01\tau\right)\leq\mathbb{P}\left(Z\leq1.01\tau\right)+t.
\]

Next, we move on to account for an arbitrary $\bm{X}\in\setlowrank_{r}^{\mathsf{norm}}$
(which is not necessarily independent of $\{\bm{A}_{i}\}$). Let $\bm{X}_{0}$
be a point in $\mathcal{M}$ obeying $\|\bm{X}-\bm{X}_{0}\|_{\mathrm{F}}\leq\epsilon$.
As a result, one has 
\begin{align*}
\frac{1}{m}\sum_{i=1}^{m}\chi\left(\left\langle \bm{A}_{i},\bm{X}\right\rangle ;1.01\tau\right)-\frac{1}{m}\sum_{i=1}^{m}\chi\left(\left\langle \bm{A}_{i},\bm{X}_{0}\right\rangle ;1.01\tau\right) & \leq\frac{1}{m}\sum_{i=1}^{m}\big|\chi\left(\left\langle \bm{A}_{i},\bm{X}\right\rangle ;1.01\tau\right)-\chi\left(\left\langle \bm{A}_{i},\bm{X}_{0}\right\rangle ;1.01\tau\right)\big|\\
 & \overset{(\text{i})}{\leq}\frac{100}{\tau}\cdot\frac{1}{m}\sum_{i=1}^{m}\left|\left\langle \bm{A}_{i},\bm{X}-\bm{X}_{0}\right\rangle \right|\\
 & \overset{(\text{ii})}{\leq}\frac{100}{\tau}\cdot\sqrt{\frac{1}{m}\sum_{i=1}^{m}\left\langle \bm{A}_{i},\bm{X}-\bm{X}_{0}\right\rangle ^{2}}\\
 & \overset{(\text{iii})}{\leq}\frac{200}{\tau}\|\bm{X}-\bm{X}_{0}\|_{\mathrm{F}}\leq\frac{200}{\tau}\epsilon.
\end{align*}
Here the inequality (i) holds since $\chi(\cdot;1.01\tau)$ is Lipschitz
with the Lipschitz constant $1/(1.01\cchi\tau)=100/\tau$, the relation
(ii) results from the Cauchy-Schwarz inequality, and (iii) follows
since $\{\bm{A}_{i}\}$ obeys $(2r,1/4)$-RIP. 

Combine the above two inequalities to finish the proof.

\section{Estimating unknown parameters in Algorithm~\ref{alg:full} \label{sec:estimate_parameters}}

Throughout the paper, we have assumed the knowledge of several problem-specific
parameters, e.g.~the proportion $p_{k}$ of the $k$-th component,
the rank $r_{k}$ of the low-rank matrix $\bm{M}_{k}^{\star}$ and
the rank $R=\rank(\E[\bY])$. In the sequel, we specify where we need
them and discuss how to estimate them in practice.
\begin{itemize}
\item In Line~\ref{line:stage1} of Algorithm \ref{alg:full}, when running
Algorithm \ref{alg:spectral}, we need to know $R=\rank(\E[\bY])$,
which can be estimated faithfully by examining the singular values
of the data matrix $\bY$.
\item In Line~\ref{line:stage2} of Algorithm \ref{alg:full}, when running
Algorithm \ref{alg:stage2}, we need to know $\{r_{k}\}_{1\le k\le K}$,
where $r_{k}=\rank(\Mkstar)$. Recall from~(\ref{eq:Sk_Mpik}) that
$\Uhat\Skhat\Vhat^{\top}\approx\Mkstar$; therefore, $r_{k}$ can
be estimated accurately by examining the singular values of $\Skhat$. 
\item In Line~\ref{line:stage3} of Algorithm \ref{alg:full}, when running
Algorithm \ref{alg:tsgd}, we need to know $p_{k}$ to set $\eta_{k}$
and $\alpha_{k}$ appropriately. It turns out that the outputs $\{\omega_{k}\}$
of the tensor method (see Algorithm \ref{alg:tensor}) satisfy $\omega_{k}\approx\pkhat,1\le k\le K$.
 
\end{itemize}

\bibliographystyle{alphaabbr}
\bibliography{refs_sense}

\newcommand{\etalchar}[1]{$^{#1}$}
\begin{thebibliography}{SBVDG10}

\bibitem[ACHL19]{arora2019implicit}
S.~Arora, N.~Cohen, W.~Hu, and Y.~Luo.
\newblock Implicit regularization in deep matrix factorization.
\newblock In {\em Advances in Neural Information Processing Systems}, pages
  7413--7424, 2019.

\bibitem[AEP07]{argyriou2007multi}
A.~Argyriou, T.~Evgeniou, and M.~Pontil.
\newblock Multi-task feature learning.
\newblock In {\em Advances in Neural Information Processing Systems}, pages
  41--48, 2007.

\bibitem[AGH{\etalchar{+}}14]{anandkumar2014tensor}
A.~Anandkumar, R.~Ge, D.~Hsu, S.~M. Kakade, and M.~Telgarsky.
\newblock Tensor decompositions for learning latent variable models.
\newblock {\em Journal of Machine Learning Research}, 15:2773--2832, 2014.

\bibitem[AZ05]{ando2005framework}
R.~K. Ando and T.~Zhang.
\newblock A framework for learning predictive structures from multiple tasks
  and unlabeled data.
\newblock {\em Journal of Machine Learning Research}, 6(Nov):1817--1853, 2005.

\bibitem[Bax00]{baxter2000model}
J.~Baxter.
\newblock A model of inductive bias learning.
\newblock {\em Journal of Artificial Intelligence Research}, 12:149--198, 2000.

\bibitem[BDS03]{ben2003exploiting}
S.~Ben-David and R.~Schuller.
\newblock Exploiting task relatedness for multiple task learning.
\newblock In {\em Learning Theory and Kernel Machines}, pages 567--580.
  Springer, 2003.

\bibitem[BJK15]{bhatia2015robust}
K.~Bhatia, P.~Jain, and P.~Kar.
\newblock Robust regression via hard thresholding.
\newblock In {\em Advances in Neural Information Processing Systems}, pages
  721--729, 2015.

\bibitem[BNS16]{bhojanapalli2016global}
S.~Bhojanapalli, B.~Neyshabur, and N.~Srebro.
\newblock Global optimality of local search for low rank matrix recovery.
\newblock In {\em Advances in Neural Information Processing Systems}, pages
  3873--3881, 2016.

\bibitem[Car97]{caruana1997multitask}
R.~Caruana.
\newblock Multitask learning.
\newblock {\em Machine Learning}, 28(1):41--75, 1997.

\bibitem[CC17]{chen2017solving}
Y.~Chen and E.~J. Cand{\`e}s.
\newblock Solving random quadratic systems of equations is nearly as easy as
  solving linear systems.
\newblock {\em Communications on Pure and Applied Mathematics}, 70(5):822--883,
  2017.

\bibitem[CC18a]{chen2018projected}
Y.~Chen and E.~J. Cand{\`e}s.
\newblock The projected power method: An efficient algorithm for joint
  alignment from pairwise differences.
\newblock {\em Communications on Pure and Applied Mathematics},
  71(8):1648--1714, 2018.

\bibitem[CC18b]{chen2018harnessing}
Y.~Chen and Y.~Chi.
\newblock Harnessing structures in big data via guaranteed low-rank matrix
  estimation: Recent theory and fast algorithms via convex and nonconvex
  optimization.
\newblock {\em IEEE Signal Processing Magazine}, 35(4):14--31, 2018.

\bibitem[CCD{\etalchar{+}}21]{charisopoulos2019low}
V.~Charisopoulos, Y.~Chen, D.~Davis, M.~D{\'\i}az, L.~Ding, and
  D.~Drusvyatskiy.
\newblock Low-rank matrix recovery with composite optimization: good
  conditioning and rapid convergence.
\newblock {\em Foundations of Computational Mathematics}, pages 1--89, 2021.

\bibitem[CCFM19]{chen2019gradient}
Y.~Chen, Y.~Chi, J.~Fan, and C.~Ma.
\newblock Gradient descent with random initialization: Fast global convergence
  for nonconvex phase retrieval.
\newblock {\em Mathematical Programming}, 176(1-2):5--37, 2019.

\bibitem[CCFM20]{chen2020spectral}
Y.~Chen, Y.~Chi, J.~Fan, and C.~Ma.
\newblock Spectral methods for data science: A statistical perspective.
\newblock {\em arXiv preprint arXiv:2012.08496}, 2020.

\bibitem[CCF{\etalchar{+}}ar]{chen2019noisy}
Y.~Chen, Y.~Chi, J.~Fan, C.~Ma, and Y.~Yan.
\newblock Noisy matrix completion: Understanding statistical guarantees for
  convex relaxation via nonconvex optimization.
\newblock {\em SIAM Journal on Optimization}, to appear.

\bibitem[CCG15]{chen2015exact}
Y.~Chen, Y.~Chi, and A.~J. Goldsmith.
\newblock Exact and stable covariance estimation from quadratic sampling via
  convex programming.
\newblock {\em IEEE Transactions on Information Theory}, 61(7):4034--4059,
  2015.

\bibitem[CFMY19]{chen2019inference}
Y.~Chen, J.~Fan, C.~Ma, and Y.~Yan.
\newblock Inference and uncertainty quantification for noisy matrix completion.
\newblock {\em Proceedings of the National Academy of Sciences},
  116(46):22931--22937, 2019.

\bibitem[CFMY20]{chen2020bridging}
Y.~Chen, J.~Fan, C.~Ma, and Y.~Yan.
\newblock Bridging convex and nonconvex optimization in robust {PCA}: Noise,
  outliers, and missing data.
\newblock {\em accepted to the Annals of Statistics}, 2020.

\bibitem[CGH14]{pmlr-v32-chend14}
Y.~Chen, L.~Guibas, and Q.~Huang.
\newblock Near-optimal joint object matching via convex relaxation.
\newblock In {\em Proceedings of the International Conference on Machine
  Learning}, pages 1269--1277, 2014.

\bibitem[CL13]{chaganty2013spectral}
A.~T. Chaganty and P.~Liang.
\newblock Spectral experts for estimating mixtures of linear regressions.
\newblock In {\em Proceedings of the International Conference on Machine
  Learning}, pages 1040--1048, 2013.

\bibitem[CLC19]{chi2019nonconvex}
Y.~Chi, Y.~M. Lu, and Y.~Chen.
\newblock Nonconvex optimization meets low-rank matrix factorization: An
  overview.
\newblock {\em IEEE Transactions on Signal Processing}, 67(20):5239--5269,
  2019.

\bibitem[CLL20]{chen2020nonconvex}
J.~Chen, D.~Liu, and X.~Li.
\newblock Nonconvex rectangular matrix completion via gradient descent without
  $\ell_{2,\infty}$ regularization.
\newblock {\em IEEE Transactions on Information Theory}, 2020.

\bibitem[CLS15]{candes2015phase}
E.~J. Candes, X.~Li, and M.~Soltanolkotabi.
\newblock Phase retrieval via wirtinger flow: Theory and algorithms.
\newblock {\em IEEE Transactions on Information Theory}, 61(4):1985--2007,
  2015.

\bibitem[CLS20]{chen2020learning}
S.~Chen, J.~Li, and Z.~Song.
\newblock Learning mixtures of linear regressions in subexponential time via
  fourier moments.
\newblock In {\em Proceedings of the 52nd Annual ACM SIGACT Symposium on Theory
  of Computing}, pages 587--600, 2020.

\bibitem[CP11]{candes2011tight}
E.~J. Candes and Y.~Plan.
\newblock Tight oracle inequalities for low-rank matrix recovery from a minimal
  number of noisy random measurements.
\newblock {\em IEEE Transactions on Information Theory}, 57(4):2342--2359,
  2011.

\bibitem[CR09]{candes2009exact}
E.~J. Cand{\`e}s and B.~Recht.
\newblock Exact matrix completion via convex optimization.
\newblock {\em Foundations of Computational Mathematics}, 9(6):717, 2009.

\bibitem[CYC17]{chen2014convex}
Y.~Chen, X.~Yi, and C.~Caramanis.
\newblock Convex and nonconvex formulations for mixed regression with two
  components: Minimax optimal rates.
\newblock {\em IEEE Transactions on Information Theory}, 64(3):1738--1766,
  2017.

\bibitem[DC20]{ding2020leave}
L.~Ding and Y.~Chen.
\newblock Leave-one-out approach for matrix completion: Primal and dual
  analysis.
\newblock {\em IEEE Transactions on Information Theory}, 2020.

\bibitem[DH00]{deb2000estimates}
P.~Deb and A.~M. Holmes.
\newblock Estimates of use and costs of behavioural health care: a comparison
  of standard and finite mixture models.
\newblock {\em Health Economics}, 9(6):475--489, 2000.

\bibitem[DHK{\etalchar{+}}21]{du2020few}
S.~S. Du, W.~Hu, S.~M. Kakade, J.~D. Lee, and Q.~Lei.
\newblock Few-shot learning via learning the representation, provably.
\newblock In {\em International Conference on Learning Representations}, 2021.

\bibitem[DK20]{diakonikolas2020small}
I.~Diakonikolas and D.~M. Kane.
\newblock Small covers for near-zero sets of polynomials and learning latent
  variable models.
\newblock {\em arXiv preprint arXiv:2012.07774}, 2020.

\bibitem[DV89]{de1989mixtures}
R.~D. De~Veaux.
\newblock Mixtures of linear regressions.
\newblock {\em Computational Statistics \& Data Analysis}, 8(3):227--245, 1989.

\bibitem[EMP05]{evgeniou2005learning}
T.~Evgeniou, C.~A. Micchelli, and M.~Pontil.
\newblock Learning multiple tasks with kernel methods.
\newblock {\em Journal of Machine Learning Research}, 6(Apr):615--637, 2005.

\bibitem[FAL17]{finn2017model}
C.~Finn, P.~Abbeel, and S.~Levine.
\newblock Model-agnostic meta-learning for fast adaptation of deep networks.
\newblock In {\em Proceedings of the 34th International Conference on Machine
  Learning - Volume 70}, pages 1126--1135. JMLR.org, 2017.

\bibitem[GJZ17]{ge2017no}
R.~Ge, C.~Jin, and Y.~Zheng.
\newblock No spurious local minima in nonconvex low rank problems: {A} unified
  geometric analysis.
\newblock In {\em Proceedings of the 34th International Conference on Machine
  Learning, {ICML} 2017, Sydney, NSW, Australia, 6-11 August 2017}, volume~70
  of {\em Proceedings of Machine Learning Research}, pages 1233--1242. {PMLR},
  2017.

\bibitem[GS99]{gaffney1999trajectory}
S.~Gaffney and P.~Smyth.
\newblock Trajectory clustering with mixtures of regression models.
\newblock In {\em Proceedings of the Fifth ACM SIGKDD International Conference
  on Knowledge Discovery and Data Mining}, pages 63--72, 1999.

\bibitem[HJ18]{hand2018convex}
P.~Hand and B.~Joshi.
\newblock A convex program for mixed linear regression with a recovery
  guarantee for well-separated data.
\newblock {\em Information and Inference: A Journal of the IMA}, 7(3):563--579,
  2018.

\bibitem[JNS13]{jain2013low}
P.~Jain, P.~Netrapalli, and S.~Sanghavi.
\newblock Low-rank matrix completion using alternating minimization.
\newblock In {\em Proceedings of the Forty-fifth Annual ACM Symposium on Theory
  of Computing}, pages 665--674, 2013.

\bibitem[JSRR10]{jalali2010dirty}
A.~Jalali, S.~Sanghavi, C.~Ruan, and P.~K. Ravikumar.
\newblock A dirty model for multi-task learning.
\newblock In {\em Advances in neural information processing systems}, pages
  964--972, 2010.

\bibitem[KC07]{khalili2007variable}
A.~Khalili and J.~Chen.
\newblock Variable selection in finite mixture of regression models.
\newblock {\em Journal of the American Statistical Association},
  102(479):1025--1038, 2007.

\bibitem[KC20]{kwon2020converges}
J.~Kwon and C.~Caramanis.
\newblock Em converges for a mixture of many linear regressions.
\newblock In {\em International Conference on Artificial Intelligence and
  Statistics}, pages 1727--1736, 2020.

\bibitem[KHC20]{kwon2020minimax}
J.~Kwon, N.~Ho, and C.~Caramanis.
\newblock On the minimax optimality of the em algorithm for learning
  two-component mixed linear regression.
\newblock {\em arXiv preprint arXiv:2006.02601}, 2020.

\bibitem[KMMP19]{krishnamurthy2019sample}
A.~Krishnamurthy, A.~Mazumdar, A.~McGregor, and S.~Pal.
\newblock Sample complexity of learning mixture of sparse linear regressions.
\newblock In {\em Advances in Neural Information Processing Systems}, pages
  10532--10541, 2019.

\bibitem[KMO10]{keshavan2010matrix}
R.~H. Keshavan, A.~Montanari, and S.~Oh.
\newblock Matrix completion from a few entries.
\newblock {\em IEEE Transactions on Information Theory}, 56(6):2980--2998,
  2010.

\bibitem[KQC{\etalchar{+}}19]{kwon2019global}
J.~Kwon, W.~Qian, C.~Caramanis, Y.~Chen, and D.~Davis.
\newblock Global convergence of the em algorithm for mixtures of two component
  linear regression.
\newblock In {\em Proceedings of the Conference on Learning Theory}, pages
  2055--2110, 2019.

\bibitem[KSKO20]{kong2020robust}
W.~Kong, R.~Somani, S.~Kakade, and S.~Oh.
\newblock Robust meta-learning for mixed linear regression with small batches.
\newblock In {\em Advances in Neural Information Processing Systems}, pages
  4683--4696, 2020.

\bibitem[KSS{\etalchar{+}}20]{kong2020meta}
W.~Kong, R.~Somani, Z.~Song, S.~Kakade, and S.~Oh.
\newblock Meta-learning for mixed linear regression.
\newblock In {\em Proceedings of the International Conference on Machine
  Learning}, pages 5394--5404. PMLR, 2020.

\bibitem[KYB19]{klusowski2019estimating}
J.~M. Klusowski, D.~Yang, and W.~Brinda.
\newblock Estimating the coefficients of a mixture of two linear regressions by
  expectation maximization.
\newblock {\em IEEE Transactions on Information Theory}, 65(6):3515--3524,
  2019.

\bibitem[LCZL20]{li2020non}
Y.~Li, Y.~Chi, H.~Zhang, and Y.~Liang.
\newblock Non-convex low-rank matrix recovery with arbitrary outliers via
  median-truncated gradient descent.
\newblock {\em Information and Inference: A Journal of the IMA}, 9(2):289--325,
  2020.

\bibitem[LL18]{li2018learning}
Y.~Li and Y.~Liang.
\newblock Learning mixtures of linear regressions with nearly optimal
  complexity.
\newblock In {\em Proceedings of the Conference On Learning Theory}, pages
  1125--1144, 2018.

\bibitem[LMZ18]{li2018algorithmic}
Y.~Li, T.~Ma, and H.~Zhang.
\newblock Algorithmic regularization in over-parameterized matrix sensing and
  neural networks with quadratic activations.
\newblock In {\em Proceedings of the Conference On Learning Theory}, pages
  2--47, 2018.

\bibitem[LZT19]{li2019non}
Q.~Li, Z.~Zhu, and G.~Tang.
\newblock The non-convex geometry of low-rank matrix optimization.
\newblock {\em Information and Inference: A Journal of the IMA}, 8(1):51--96,
  2019.

\bibitem[MPRP16]{maurer2016benefit}
A.~Maurer, M.~Pontil, and B.~Romera-Paredes.
\newblock The benefit of multitask representation learning.
\newblock {\em The Journal of Machine Learning Research}, 17(1):2853--2884,
  2016.

\bibitem[MWCC20]{ma2017implicit}
C.~Ma, K.~Wang, Y.~Chi, and Y.~Chen.
\newblock Implicit regularization in nonconvex statistical estimation: Gradient
  descent converges linearly for phase retrieval, matrix completion, and blind
  deconvolution.
\newblock {\em Foundations of Computational Mathematics}, 20(3):451--632, 2020.

\bibitem[NNS{\etalchar{+}}14]{netrapalli2014non}
P.~Netrapalli, U.~Niranjan, S.~Sanghavi, A.~Anandkumar, and P.~Jain.
\newblock Non-convex robust pca.
\newblock In {\em Advances in Neural Information Processing Systems}, pages
  1107--1115, 2014.

\bibitem[PA18]{pimentel2018mixture}
D.~Pimentel-Alarc{\'o}n.
\newblock Mixture matrix completion.
\newblock In {\em Advances in Neural Information Processing Systems}, pages
  2193--2203, 2018.

\bibitem[PKCS17]{park2017non}
D.~Park, A.~Kyrillidis, C.~Carmanis, and S.~Sanghavi.
\newblock Non-square matrix sensing without spurious local minima via the
  burer-monteiro approach.
\newblock In {\em Artificial Intelligence and Statistics}, pages 65--74, 2017.

\bibitem[PL14]{pentina2014pac}
A.~Pentina and C.~Lampert.
\newblock A pac-bayesian bound for lifelong learning.
\newblock In {\em Proceedings of the International Conference on Machine
  Learning}, pages 991--999, 2014.

\bibitem[PLW{\etalchar{+}}21]{peng2020cfit}
M.~Peng, Y.~Li, B.~Wamsley, Y.~Wei, and K.~Roeder.
\newblock Integration and transfer learning of single-cell transcriptomes via
  cfit.
\newblock {\em Proceedings of the National Academy of Sciences}, 118(10), 2021.

\bibitem[PM20]{mazumdar2020recovery}
S.~Pal and A.~Mazumdar.
\newblock Recovery of sparse signals from a mixture of linear samples.
\newblock In {\em Proceedings of the International Conference on Machine
  Learning}, pages 7466--7475. PMLR, 2020.

\bibitem[PY09]{pan2009survey}
S.~J. Pan and Q.~Yang.
\newblock A survey on transfer learning.
\newblock {\em IEEE Transactions on Knowledge and Data Engineering},
  22(10):1345--1359, 2009.

\bibitem[QR78]{quandt1978estimating}
R.~E. Quandt and J.~B. Ramsey.
\newblock Estimating mixtures of normal distributions and switching
  regressions.
\newblock {\em Journal of the American Statistical Association},
  73(364):730--738, 1978.

\bibitem[RFP10]{recht2010guaranteed}
B.~Recht, M.~Fazel, and P.~A. Parrilo.
\newblock Guaranteed minimum-rank solutions of linear matrix equations via
  nuclear norm minimization.
\newblock {\em SIAM Review}, 52(3):471--501, 2010.

\bibitem[Rou84]{rousseeuw1984least}
P.~J. Rousseeuw.
\newblock Least median of squares regression.
\newblock {\em Journal of the American Statistical Association},
  79(388):871--880, 1984.

\bibitem[SBVDG10]{stadler2010L}
N.~St{\"a}dler, P.~B{\"u}hlmann, and S.~Van De~Geer.
\newblock L1-penalization for mixture regression models.
\newblock {\em Test}, 19(2):209--256, 2010.

\bibitem[Ser09]{serfling2009approximation}
R.~J. Serfling.
\newblock {\em Approximation Theorems of Mathematical Statistics}, volume 162.
\newblock John Wiley \& Sons, 2009.

\bibitem[SJA16]{sedghi2016provable}
H.~Sedghi, M.~Janzamin, and A.~Anandkumar.
\newblock Provable tensor methods for learning mixtures of generalized linear
  models.
\newblock In {\em Artificial Intelligence and Statistics}, pages 1223--1231,
  2016.

\bibitem[SL16]{sun2016guaranteed}
R.~Sun and Z.-Q. Luo.
\newblock Guaranteed matrix completion via non-convex factorization.
\newblock {\em IEEE Transactions on Information Theory}, 62(11):6535--6579,
  2016.

\bibitem[SLHH18]{sun2018joint}
Y.~Sun, Z.~Liang, X.~Huang, and Q.~Huang.
\newblock Joint map and symmetry synchronization.
\newblock In {\em Proceedings of the European Conference on Computer Vision
  (ECCV)}, pages 251--264, 2018.

\bibitem[SQW18]{sun2018geometric}
J.~Sun, Q.~Qu, and J.~Wright.
\newblock A geometric analysis of phase retrieval.
\newblock {\em Foundations of Computational Mathematics}, 18(5):1131--1198,
  2018.

\bibitem[SS19a]{shen2019iterative}
Y.~Shen and S.~Sanghavi.
\newblock Iterative least trimmed squares for mixed linear regression.
\newblock In {\em Advances in Neural Information Processing Systems}, pages
  6078--6088, 2019.

\bibitem[SS19b]{shen2019learning}
Y.~Shen and S.~Sanghavi.
\newblock Learning with bad training data via iterative trimmed loss
  minimization.
\newblock In {\em Proceedings of the International Conference on Machine
  Learning}, pages 5739--5748, 2019.

\bibitem[SSZ17]{snell2017prototypical}
J.~Snell, K.~Swersky, and R.~Zemel.
\newblock Prototypical networks for few-shot learning.
\newblock In {\em Advances in Neural Information Processing Systems}, pages
  4077--4087, 2017.

\bibitem[SWS20]{shah2020choosing}
V.~Shah, X.~Wu, and S.~Sanghavi.
\newblock Choosing the sample with lowest loss makes sgd robust.
\newblock In {\em International Conference on Artificial Intelligence and
  Statistics}, pages 2120--2130. PMLR, 2020.

\bibitem[TBS{\etalchar{+}}16]{tu2016low}
S.~Tu, R.~Boczar, M.~Simchowitz, M.~Soltanolkotabi, and B.~Recht.
\newblock Low-rank solutions of linear matrix equations via procrustes flow.
\newblock In {\em Proceedings of the International Conference on Machine
  Learning}, pages 964--973, 2016.

\bibitem[TJJ20]{tripuraneni2020provable}
N.~Tripuraneni, C.~Jin, and M.~I. Jordan.
\newblock Provable meta-learning of linear representations.
\newblock {\em arXiv preprint arXiv:2002.11684}, 2020.

\bibitem[TMC20a]{tong2020accelerating}
T.~Tong, C.~Ma, and Y.~Chi.
\newblock Accelerating ill-conditioned low-rank matrix estimation via scaled
  gradient descent.
\newblock {\em arXiv preprint arXiv:2005.08898}, 2020.

\bibitem[TMC20b]{tong2020low}
T.~Tong, C.~Ma, and Y.~Chi.
\newblock Low-rank matrix recovery with scaled subgradient methods: Fast and
  robust convergence without the condition number.
\newblock {\em arXiv preprint arXiv:2010.13364}, 2020.

\bibitem[Tur00]{turner2000estimating}
T.~R. Turner.
\newblock Estimating the propagation rate of a viral infection of potato plants
  via mixtures of regressions.
\newblock {\em Journal of the Royal Statistical Society: Series C (Applied
  Statistics)}, 49(3):371--384, 2000.

\bibitem[Ver18]{vershynin2018high}
R.~Vershynin.
\newblock {\em High-dimensional Probability: An Introduction with Applications
  in Data Science}, volume~47.
\newblock Cambridge university press, 2018.

\bibitem[VT02]{viele2002modeling}
K.~Viele and B.~Tong.
\newblock Modeling with mixtures of linear regressions.
\newblock {\em Statistics and Computing}, 12(4):315--330, 2002.

\bibitem[Wed72]{wedin1972perturbation}
P.-{\AA}. Wedin.
\newblock Perturbation bounds in connection with singular value decomposition.
\newblock {\em BIT Numerical Mathematics}, 12(1):99--111, 1972.

\bibitem[YC15]{yi2015regularized}
X.~Yi and C.~Caramanis.
\newblock Regularized em algorithms: A unified framework and statistical
  guarantees.
\newblock In {\em Advances in Neural Information Processing Systems}, pages
  1567--1575, 2015.

\bibitem[YCS14]{yi2014alternating}
X.~Yi, C.~Caramanis, and S.~Sanghavi.
\newblock Alternating minimization for mixed linear regression.
\newblock In {\em Proceedings of the International Conference on Machine
  Learning}, pages 613--621, 2014.

\bibitem[YCS16]{yi2016solving}
X.~Yi, C.~Caramanis, and S.~Sanghavi.
\newblock Solving a mixture of many random linear equations by tensor
  decomposition and alternating minimization.
\newblock {\em arXiv preprint arXiv:1608.05749}, 2016.

\bibitem[YL20]{yuan2020learning}
H.~Yuan and Y.~Liang.
\newblock Learning entangled single-sample distributions via iterative
  trimming.
\newblock In {\em International Conference on Artificial Intelligence and
  Statistics}, pages 2666--2676. PMLR, 2020.

\bibitem[YPCR18]{yin2018learning}
D.~Yin, R.~Pedarsani, Y.~Chen, and K.~Ramchandran.
\newblock Learning mixtures of sparse linear regressions using sparse graph
  codes.
\newblock {\em IEEE Transactions on Information Theory}, 65(3):1430--1451,
  2018.

\bibitem[ZCL18]{zhang2018median}
H.~Zhang, Y.~Chi, and Y.~Liang.
\newblock Median-truncated nonconvex approach for phase retrieval with
  outliers.
\newblock {\em IEEE Transactions on Information Theory}, 64(11):7287--7310,
  2018.

\bibitem[ZJD16]{zhong2016mixed}
K.~Zhong, P.~Jain, and I.~S. Dhillon.
\newblock Mixed linear regression with multiple components.
\newblock In {\em Advances in Neural Information Processing Systems}, pages
  2190--2198, 2016.

\bibitem[ZL15]{zheng2015convergent}
Q.~Zheng and J.~Lafferty.
\newblock A convergent gradient descent algorithm for rank minimization and
  semidefinite programming from random linear measurements.
\newblock In {\em Advances in Neural Information Processing Systems}, pages
  109--117, 2015.

\bibitem[ZLTW18]{zhu2018global}
Z.~Zhu, Q.~Li, G.~Tang, and M.~B. Wakin.
\newblock Global optimality in low-rank matrix optimization.
\newblock {\em IEEE Transactions on Signal Processing}, 66(13):3614--3628,
  2018.

\bibitem[ZQW20]{zhang2020symmetry}
Y.~Zhang, Q.~Qu, and J.~Wright.
\newblock From symmetry to geometry: Tractable nonconvex problems.
\newblock {\em arXiv preprint arXiv:2007.06753}, 2020.

\bibitem[ZWYG18]{zhang2018primal}
X.~Zhang, L.~Wang, Y.~Yu, and Q.~Gu.
\newblock A primal-dual analysis of global optimality in nonconvex low-rank
  matrix recovery.
\newblock In {\em International conference on machine learning}, pages
  5862--5871, 2018.

\end{thebibliography}

\end{document}